\documentclass[aos,preprint]{imsart}
\usepackage{amsmath,amssymb,amsthm,graphicx,xcolor,mathrsfs}
\usepackage{bbm, bm}
\usepackage{url}
\usepackage[colorlinks,linkcolor=black, citecolor=blue,urlcolor=blue]{hyperref}
\usepackage[numbers]{natbib}
\usepackage{varwidth}
\usepackage[ruled, linesnumbered, lined, commentsnumbered, resetcount,algosection]{algorithm2e}
\usepackage{multirow}
\usepackage{subfig, float}
\usepackage{rotating}
\usepackage{array}
\allowdisplaybreaks
\setattribute{copyright}{text}{}

\newcommand{\PreserveBackslash}[1]{\let\temp=\\#1\let\\=\temp}
\newcolumntype{C}[1]{>{\PreserveBackslash\centering}p{#1}}

\renewcommand{\thetable}{E\arabic{table}}

\usepackage[latin1]{inputenc}

\numberwithin{equation}{section}
\theoremstyle{plain}
\newtheorem{Theorem}{Theorem}
\numberwithin{Theorem}{section}

\newtheorem{Lemma}{Lemma}
\numberwithin{Lemma}{section}
{
	\theoremstyle{definition}
	\newtheorem{Definition}{Definition}
	\newtheorem{example}{Example}
	\numberwithin{example}{section}
	\newtheorem{fact}{Fact}
	\numberwithin{fact}{section}
	
}

\newcommand{\A}{\mathcal{A}}
\newcommand{\Pro}{\mathbb{P}}
\newcommand{\1}{\mathbbm{1}}

\def\bmu{{\boldsymbol{\mu}}}

\def\bbeta{{\boldsymbol{\beta}}}
\def\bth{{\boldsymbol{\theta}}}

\def\Pr{\mathbb P}

\DeclareMathOperator*{\argmax}{arg\,max}

\def\supp{{\rm supp}}

\def\R{{\mathbb R}}

\def\sign{{\rm sign}}
\def\E{\mathbb{E}}

\def\L{{\mathcal L}}

\begin{document}

\begin{frontmatter}
\title{The Cost of Privacy: Rates of Convergence for Parameter Estimation with Differential Privacy}
\runtitle{The Cost of Privacy}

\begin{aug}
	\author{\fnms{T. Tony} \snm{Cai}\ead[label=e1]{tcai@wharton.upenn.edu}
		\ead[label=u1,url]{http://www-stat.wharton.upenn.edu/\string~tcai/}},
	\author{\fnms{Yichen} \snm{Wang}\ead[label=e2]{wangyc@wharton.upenn.edu}},
	\and
	\author{\fnms{Linjun} \snm{Zhang}\ead[label=e3]{linjun.zhang@rutgers.edu}
		\ead[label=u2,url]{https://linjunz.github.io/}}
	\runauthor{T. T. Cai, Y. Wang and L. Zhang}
	\affiliation{University of Pennsylvania and Rutgers University}
	\address{DEPARTMENT OF STATISTICS\\
		THE WHARTON SCHOOL\\
		UNIVERSITY OF PENNSYLVANIA\\
		PHILADELPHIA, PENNSYLVANIA 19104\\
		USA\\
		\textup{E-mail:} \href{mailto:tcai@wharton.upenn.edu}{\nolinkurl{tcai@wharton.upenn.edu}}\\
		\phantom{E-mail:\ }\href{mailto:wangyc@wharton.upenn.edu}{\nolinkurl{wangyc@wharton.upenn.edu}}\\
		\phantom{E-mail:\ }\href{mailto:linjun.zhang@rutgers.edu}{\nolinkurl{linjun.zhang@rutgers.edu}}\\
		\textup{URL:} \url{http://www-stat.wharton.upenn.edu/~tcai/}\\
		\phantom{URL:\ }\url{https://linjunz.github.io/}
	}

\end{aug}

\begin{abstract}
{We study the minimax cost of $(\varepsilon,\delta)$-differential privacy for mean estimation and Gaussian linear regression in low and high dimensions. For low-dimensional mean estimation, a resampling reduction to fingerprinting yields the privacy contribution $d^2\log(1/\delta)/(n^2\varepsilon^2)$ in the stated polynomial-$\delta$ regime. For low-dimensional regression, a tracing argument gives the contribution $d^2/(n^2\varepsilon^2)$ under an explicit approximate-DP remainder condition. For sparse mean estimation and sparse regression, a constant-weight packing and a private Fano lemma produce an effective privacy entropy of order $\min\{s\log(ed/s),[\log((e^\varepsilon-1)/\delta)]_+\}$ for $\delta>0$, up to universal constants and a fixed threshold; for pure DP it is $s\log(ed/s)$. Thus, when $\delta$ is polynomially smaller than $\varepsilon$, the pure-DP dependence is retained up to polylogarithmic factors whenever the effective dimension is polylogarithmic in $n$, including regimes with $\varepsilon=o(1)$. Coordinatewise-clipping estimators for means and split-sample noisy-gradient estimators for regression attain the lower bounds up to explicit logarithmic factors. Simulations and data examples illustrate related implementations.}
\end{abstract}



\end{frontmatter}

\section{Introduction}\label{sec: introduction}
{The growing availability of data containing sensitive personal information makes it essential to quantify how privacy protection affects statistical accuracy. This paper studies that tradeoff in two fundamental problems---mean estimation and linear regression---under a statistical minimax criterion.}

{Differential privacy, introduced by \cite{dwork2006calibrating}, is one of the most widely used formal privacy guarantees in data analysis. It protects each individual even against an adversary who observes the algorithm's output and may know the remaining records. Differential privacy has been studied extensively \cite{dwork2014algorithmic, abadi2016deep, dwork2017exposed, dwork2018privacy} and deployed by Apple \cite{apple2018privacy}, Google \cite{erlingsson2014rappor}, Microsoft \cite{ding2017collecting}, and the U.S. Census Bureau \cite{abowd2016challenge}.}

{Many private algorithms add random noise calibrated to the sensitivity of a non-private statistic \cite{dwork2006calibrating, mcsherry2007mechanism, dwork2014algorithmic}. The resulting loss of accuracy is well documented \cite{wasserman2010statistical, smith2011privacy, lei2011differentially,bassily2014private,dwork2014analyze}; our goal is to characterize it quantitatively. We work with $(\varepsilon,\delta)$-differential privacy, defined next, and compare upper and lower bounds for the corresponding minimax risk.}
	\begin{Definition}[Differential Privacy \citep{dwork2006calibrating}]\label{def: dp}
		A randomized algorithm $M: \mathcal X^n \to \mathcal R$ is $(\varepsilon, \delta)$-differentially private if for every pair of adjacent data sets $\bm X, \bm X' \in \mathcal X^n$ that differ by one individual datum and every (measurable) $S \subseteq \mathcal R$,
		\begin{align*}
			\Pro\left(M(\bm X) \in S\right) \leq e^\varepsilon \cdot \Pro\left(M(\bm X') \in S\right) + \delta,
		\end{align*}
		where the probability measure $\Pro$ is induced by the randomness of $M$ only.
	\end{Definition}
	\noindent According to the definition, the two parameters $\varepsilon$ and $\delta$ control the level of privacy against an adversary who attempts to detect the presence of a certain individual in the sample. The privacy constraint becomes more stringent as $\varepsilon,\delta$ tend to $0$.

	\subsection*{Our contributions and related literature}
	\quad

	{\textbf{Lower bounds under privacy constraints}. We use two complementary lower-bound strategies. For low-dimensional mean estimation, a resampling reduction transfers a fingerprinting lower bound from one-way marginals to i.i.d. population means. For low-dimensional regression, a tracing score relates estimator accuracy to the distinguishability of one-record replacements. These ideas originate in the theoretical computer science literature \cite{bun2014fingerprinting, steinke2017between} and were later adapted to population quantities \cite{dwork2015robust, kamath2018privately}. Sparse mean estimation and sparse regression require a different argument: a constant-weight packing and a private Fano lemma handle the unknown support directly. The main conclusions are summarized below and stated precisely in Sections \ref{sec: mean} and \ref{sec: regression}.}
	\begin{enumerate}
		\item [(1)] Improved lower bound for low-dimensional mean estimation. In Section \ref{sec: low-dim mean lower bound}, a resampling reduction to the fingerprinting theorem of \cite{steinke2017between} gives a minimax squared $\ell_2$ risk contribution $\Omega\left(\frac{d^2\log(1/\delta)}{n^2\varepsilon^2}\right)$ for coordinatewise sub-Gaussian means. A coordinatewise-truncated Gaussian mechanism matches this rate for bounded observations and up to the logarithmic truncation cost for unbounded observations.
		{\item [(2)] Sparse lower bounds and low-dimensional regression. For sparse mean estimation and sparse linear regression, a private packing argument gives the privacy contribution $\Omega\left(q_{d,s}(\varepsilon,\delta)^2/(n^2\varepsilon^2)\right)$. For $\delta>0$, $q_{d,s}(\varepsilon,\delta)$ is, up to universal constants and an inessential constant shift, $\min\{s\log(ed/s),[\log((e^\varepsilon-1)/\delta)]_+\}$; for pure DP it equals $s\log(ed/s)$. In particular, if $\varepsilon=n^{-b}$ and $\delta=n^{-\kappa}$ with $\kappa>b$, polynomially small $\delta$ preserves the pure-DP rate up to a polylogarithmic factor whenever $s\log(ed/s)$ is polylogarithmic in $n$. For low-dimensional regression, the privacy contribution is $\Omega(d^2/(n^2\varepsilon^2))$, up to the diameter cap and under the stated approximate-DP remainder condition.}

		{These results concern population parameter estimation and should be distinguished from lower bounds for related tasks. For example, \cite{steinke2017tight} studied private selection of top coordinates, while \cite{bassily2014private} established lower bounds for private empirical risk minimization using worst-case strongly convex Lipschitz objectives.}
	\end{enumerate}

	{\textbf{Differentially private algorithms}. For mean estimation, we analyze coordinatewise clipping followed by Gaussian or Peeling noise. For regression, we use disjoint data blocks across iterations. The sample split makes each iterate independent of the block used for its next gradient, which is essential for the conditional concentration argument. In the regimes stated in the theorems, the resulting bounds match the lower bounds up to explicit polylogarithmic factors.}

	{In low-dimensional problems, Algorithms \ref{algo: low-dim mean estimation} and \ref{algo: low-dim regression} are related to the noisy Gaussian sample mean \cite{karwa2017finite} and noisy gradient descent \cite{bassily2014private}. Our regression result is a finite-sample, high-probability bound for squared parameter error rather than an expectation bound for excess empirical risk. A deterministic Euclidean clipping step gives global sensitivity on arbitrary input data, so the guarantee does not require an almost-sure bound on the design. Under the sub-Gaussian design condition, that clipping is inactive with high probability. Up to logarithmic factors and fixed design constants, the privacy term is $d^2\log(1/\delta)/(n^2\varepsilon^2)$; see Theorem \ref{thm: low-dim regression upper bound}. This target differs from the excess-risk criteria studied in \cite{bassily2014private,bassily2019private}.}

	{For high-dimensional sparse estimation, Algorithm \ref{algo: high-dim mean estimation} applies the Peeling procedure of \cite{dwork2018differentially} to a clipped sample mean. Algorithm \ref{algo: high-dim regression} is a private, noisy analogue of iterative hard thresholding \cite{blumensath2009iterative,jain2014iterative}; coordinatewise design clipping controls its $\ell_\infty$ sensitivity without a bounded-design assumption. Theorems \ref{thm: high-dim mean upper bound} and \ref{thm: high-dim regression upper bound} give high-probability squared parameter-error bounds whose privacy terms are, up to logarithmic factors and fixed design constants, $(s\log d)^2\log(1/\delta)/(n^2\varepsilon^2)$. These parameter-error guarantees are not directly comparable to the excess-risk bounds in \cite{kifer2012private,talwar2015nearly}.}

	\subsection*{Other related literature}
	{The broader literature studies several accuracy criteria and privacy models.  Under sufficiently stringent asymptotic conditions, \cite{smith2011privacy} showed that privacy need not change the first-order statistical rate for certain estimators.  Other work developed private algorithms and excess-risk guarantees for convex empirical risk minimization, principal-component analysis, and sparse regression \cite{bassily2014private,dwork2014analyze,talwar2015nearly}.  These results motivate our focus on squared population-parameter error, for which the privacy contribution can be stated directly in the original statistical units.}

	{Local differential privacy imposes a stronger constraint because each record is randomized before aggregation.  Early statistical analyses include \cite{wasserman2010statistical} and \cite{kasiviswanathan2011can}; \cite{duchi2018minimax} developed a general minimax framework, and \cite{rohde2018geometrizing} obtained optimal rates and mechanisms for several locally private problems.  Our results concern central $(\varepsilon,\delta)$-differential privacy, so their dimensional dependence and attainable rates are different from those under local privacy.}

	\subsection*{Organization of the paper}
	{Section \ref{sec: formulation} defines the private minimax risk and reviews the tools used throughout the paper. Sections \ref{sec: mean} and \ref{sec: regression} present the results for mean estimation and linear regression, respectively. Sections \ref{sec: simulations} and \ref{sec: data analysis} report simulations and data illustrations, and Section \ref{sec: discussions} discusses limitations and open problems. Proofs are divided between Section \ref{sec: proofs} and the supplement.}

\subsection*{Notation}
For real-valued sequences $\{a_n\}$ and $\{b_n\}$, we write $a_n \lesssim b_n$ if $a_n \leq cb_n$ for some universal constant $c \in (0, \infty)$, and $a_n \gtrsim b_n$ if $a_n \geq c'b_n$ for some universal constant $c' \in (0, \infty)$. We say $a_n \asymp b_n$ if $a_n \lesssim b_n$ and $a_n \gtrsim b_n$. {Unless a dependence is stated explicitly, $c,C,c_0,c_1,c_2,\ldots$ denote positive universal constants whose values may change from line to line.}

{For a positive integer $k$, write $[k]=\{1,\ldots,k\}$. For $S\subseteq[d]$, $\bm v_S$ denotes the vector equal to $\bm v$ on $S$ and zero on $S^c$, and $\supp(\bm v)=\{j:v_j\neq0\}$. We use $\|\bm v\|_p$ for the usual $\ell_p$ norm, $1\leq p\leq\infty$, and $\|\bm v\|_0$ for the number of nonzero coordinates. For a positive definite matrix $\Sigma$, let $\|\bm v\|_\Sigma=(\bm v^\top\Sigma\bm v)^{1/2}$. For a scalar $z$, $\Pi_R(z)=\sign(z)\min\{|z|,R\}$ denotes clipping to $[-R,R]$; for a vector $\bm v$, $\Pi_R(\bm v)$ denotes Euclidean projection onto $\{\bm u:\|\bm u\|_2\leq R\}$. The intended meaning is determined by the argument.}

\section{Problem Formulation}\label{sec: formulation}
{This section defines the private minimax risk and fixes the basic tools used later. Section \ref{sec: cost of privacy} introduces the decision-theoretic criterion, Section \ref{sec: upper bound general} reviews privacy mechanisms, and Section \ref{sec: lower bound general} explains the lower-bound reductions.}

\subsection{The Cost of Privacy}\label{sec: cost of privacy}
We quantify the cost of differential privacy in statistical estimation via the minimax risk with differential privacy constraint, defined as follows.

Let $\mathcal P$ denote a family of distributions supported on a set $\mathcal X$, and let $\bth: \mathcal P\to \Theta \subseteq \R^d$ denote a population quantity of interest. The statistician has access to a data set of $n$ i.i.d. samples, $\bm X=(\bm x_1,...,\bm x_n)\in\mathcal X^n$, drawn from some distribution $P\in\mathcal P$.

{Let $\mathcal A\subseteq\R^d$ be an ambient action space containing $\Theta$; ordinarily, and throughout this paper, $\mathcal A=\R^d$. Define
\begin{align*}
	\mathcal M_{\varepsilon,\delta}(\mathcal A)
	=\left\{M:\mathcal X^n\to\mathcal A:
	M\text{ is }(\varepsilon,\delta)\text{-differentially private}\right\}.
\end{align*}
Thus the true parameter remains in $\Theta$, whereas an admissible estimator may take values in the larger decision space $\mathcal A$. We abbreviate
$\mathcal M_{\varepsilon,\delta}:=\mathcal M_{\varepsilon,\delta}(\R^d)$ below. Let $\rho:\mathcal A\times\Theta\to\R^+$ be induced by a norm, so that $\rho(a,\bm\theta)=\|a-\bm\theta\|$, and let $l:\R^+\to\R^+$ be increasing. The private minimax risk is}
\begin{align}\label{eq: private minimax risk}
{\inf_{M\in \mathcal{M}_{\varepsilon, \delta}(\mathcal A)}\sup_{P\in\mathcal P} \E_P\left[l(\rho(M(\bm X),\bth(P)))\right].}
\end{align}
{The quantity characterizes the worst-case performance over $\mathcal P$ of the best $(\varepsilon, \delta)$-differentially private estimator with values in $\mathcal A$. We compare it with the usual unconstrained minimax risk on the same action space}
\begin{align}\label{eq: non-private minimax risk}
	{\inf_{M:\mathcal X^n\to\mathcal A}\sup_{P\in\mathcal P} \E_P\left[l(\rho(M(\bm X),\bth(P)))\right].}
\end{align}
{The increase from \eqref{eq: non-private minimax risk} to \eqref{eq: private minimax risk} is what we call the cost of privacy. Because the unconstrained rates for mean estimation and linear regression are classical, the remainder of the paper focuses on the constrained risk and isolates its privacy-dependent contribution.}

\subsection{Construction of Differentially Private Algorithms}\label{sec: upper bound general}
It is frequently the case that differentially private algorithms are constructed by perturbing the output of a non-private algorithm with random noise. Among the most prominent examples are the Laplace and Gaussian mechanisms.
\subsubsection*{The Laplace and Gaussian mechanisms}
As the name suggests, the Laplace and Gaussian mechanisms achieve differential privacy by perturbing an algorithm with Laplace and Gaussian noise, respectively. The noise scale is determined by the sensitivity of the algorithm:
\begin{Definition}\label{def: sensitivity}
	For any algorithm $f$ mapping a data set $\bm X$ to $\R^d$, the $\ell_p$-sensitivity of $f$ is
	$$\Delta_{p}(f) = \sup_{\bm X, \bm X' \text{ adjacent}}\|f(\bm X) - f(\bm X')\|_p.$$
\end{Definition}
The sensitivity of an algorithm $f$ characterizes the magnitude of the change in its output resulting from replacement of one element in the input data set. A perturbation of comparable scale then makes the private output stable to the presence or absence of any individual datum.

For algorithms with finite $\ell_1$-sensitivity, differential privacy can be attained by adding Laplace noises.
\begin{example}[The Laplace mechanism] \label{ex: Laplace mechanism}
	For any algorithm $f$ mapping a data set $\bm X$ to $\R^d$ such that $\Delta_1(f) < \infty$, $M_1(\bm X) := f(\bm X) + (\xi_1, \xi_2, \cdots, \xi_d)$, where $\xi_1, \xi_2, \cdots, \xi_d$ are i.i.d. $\text{Laplace}(\Delta_1(f)/\varepsilon)$ variables, achieves $(\varepsilon, 0)$-differential privacy.
\end{example}
Adding Gaussian noises to algorithms with finite $\ell_2$-sensitivity guarantees differential privacy.
\begin{example}[The Gaussian mechanism] \label{ex: Gaussian mechanism} For any algorithm $f$ mapping a data set $\bm X$ to $\R^d$ such that $\Delta_2(f) < \infty$, $M_2(\bm X) := f(\bm X) +  (\xi_1, \xi_2, \cdots, \xi_d)$, where $\xi_1, \xi_2, \cdots, \xi_d$ are i.i.d. $N(0, 2(\Delta_2(f)/\varepsilon)^2\log(1.25/\delta))$ variables, achieves $(\varepsilon, \delta)$-differential privacy.
\end{example}

Although conceptually simple, these mechanisms can lead to complex differentially private algorithms through post-processing and composition.

\subsubsection*{Post-processing and Composition}
Conveniently, post-processing a differentially private algorithm preserves privacy.
\begin{fact}[Post-processing \citep{dwork2006calibrating, wasserman2010statistical}]\label{lemma: postprocess}
	{Let $f$ be an $(\varepsilon, \delta)$-differentially private algorithm. If $g$ is any data-independent, possibly randomized mapping, then $g(f(\bm X))$ is $(\varepsilon, \delta)$-differentially private.}
\end{fact}

Further, the privacy parameters are additive with respect to compositions of differentially private algorithms.
\begin{fact}[Composition \citep{dwork2006calibrating}]\label{lemma: additive composition}
	If two mechanisms, possibly chosen adaptively, are $(\varepsilon_1,\delta_1)$- and $(\varepsilon_2,\delta_2)$-differentially private on the same data set, then their joint release is $(\varepsilon_1+\varepsilon_2,\delta_1+\delta_2)$-differentially private.
\end{fact}
The mechanisms and composition theorem reviewed in this section shall later enable us to construct differentially private algorithms for mean estimation and linear regression.

\subsection{Minimax Risk Lower Bounds with Differential Privacy Constraint}\label{sec: lower bound general}

We use two techniques to prove lower bounds for the minimax risk \eqref{eq: private minimax risk}. For the low-dimensional problems, we refine the ``tracing adversary" argument originally proposed by \cite{bun2014fingerprinting}. It has proven to be a powerful tool for obtaining lower bounds in the context of releasing sample quantities \cite{steinke2017between, steinke2017tight} and for Gaussian mean estimation \cite{dwork2015robust, kamath2018privately}. For the sparse mean and regression problems, we instead use constant-weight packings together with a private Fano inequality; this handles uncertainty in the unknown support without differentiating a parameter-dependent support-selection rule.

Informally, a tracing adversary attempts to detect the presence of a candidate datum from an estimator computed on the target data set. An accurate estimator can make this attack succeed, contradicting differential privacy. We use such a score attack for low-dimensional regression. For low-dimensional mean estimation we instead reduce to the fingerprinting lower bound for one-way marginals, and for the sparse problems we use private packings. These choices keep the support and resampling steps discrete and avoid unjustified differentiation through data-dependent selection rules.

\subsubsection*{Example: resampling a hard marginal-release instance}
For low-dimensional mean estimation we use the fingerprinting lower bound for one-way marginals from \cite{steinke2017between}. A database $D\in\{-\sigma,\sigma\}^{m\times d}$ defines a distribution $Q_D$ that is uniform on its rows. This distribution is coordinatewise sub-Gaussian and has mean equal to the vector of one-way marginals of $D$. Drawing observations i.i.d. from $Q_D$ therefore converts a population mean estimator into a marginal-release mechanism.

The privacy parameters of the converted mechanism require care because a row of $D$ can be sampled more than once. If $K$ is the number of times the modified row is resampled, then $K\sim\operatorname{Binomial}(n,1/m)$. On the event $K\leq r$, group privacy gives privacy parameters	$(r\varepsilon, \delta\frac{e^{r\varepsilon}-1}{e^\varepsilon-1})$. Combining this calculation with the fingerprinting completeness--soundness argument of \cite{steinke2017between} gives Lemma \ref{lm: low-dim mean group privacy}. A Bernoulli-thinned version of the same resampling mechanism gives Theorem \ref{thm: low-dim mean lower bound}. The full calculation is given in Section \ref{sec: proof of lm: low-dim mean group privacy}. This route avoids differentiating through the output of a mean estimator and makes the dependence on $\delta$ explicit.

\section{The Cost of Privacy in Mean Estimation}\label{sec: mean}

{We first study mean estimation. The low-dimensional results isolate the cost of coordinatewise clipping and Gaussian noise; the sparse results add private support selection. Throughout this section, an $\R^d$-valued random variable $\bm x$ is coordinatewise sub-Gaussian$(\sigma)$ if, for $\bmu=\E\bm x$ and every $j\in[d]$,}
\begin{align*}
	\E \exp\left(\lambda(x_j-\mu_j)\right) \leq \exp(\lambda^2\sigma^2), \qquad \forall \lambda \in \R.
\end{align*}
{No independence among coordinates is assumed. This condition is exactly what is needed for the coordinatewise clipping analysis, and it contains the bounded-coordinate hard instances used in the fingerprinting lower bound. We begin with the low-dimensional lower and upper bounds and then turn to sparse means.}

{For sparse mean estimation, Section \ref{sec: high-dim mean lower bound} gives an $(\varepsilon,\delta)$-adaptive private-packing lower bound, while Section \ref{sec: high-dim mean upper bound} analyzes a computationally efficient Peeling estimator.}

\subsection{{Lower bounds for low-dimensional mean estimation}}\label{sec: low-dim mean lower bound}
In this section, we prove a lower bound for estimating a $d$-dimensional coordinatewise sub-Gaussian mean. We consider the class of coordinatewise sub-Gaussian$(\sigma)$ distributions with mean vector in $\Theta = \{\bmu \in \R^d: \|\bmu\|_\infty < \sigma\}$ and denote the class by $\mathcal P(\sigma, d, \Theta)$.

{The key step is a resampling lemma at an effective sample size $n_0\asymp\sqrt{d\log(1/\delta)}/\varepsilon$. The subsequent Bernoulli-thinning reduction transfers this hard instance to every sufficiently large sample size $n$.}
{\begin{Lemma}\label{lm: low-dim mean group privacy}
	Fix $\omega>0$, put $\ell=\log(1/\delta)$, and choose sufficiently small constants $c,c'>0$, depending only on $\omega$, and integers
	\begin{align*}
		m=\left\lfloor c'\sqrt{d/\ell}\right\rfloor,\qquad
		k=\left\lceil c\ell/\varepsilon\right\rceil,\qquad n_0=km.
	\end{align*}
	Suppose $0<\varepsilon\leq1$, $0<\delta<1/2$, $d/\ell$ is larger than a sufficiently large constant, and $\delta\leq n_0^{-(1+\omega)}$. If $\bm Y=(\bm y_1,\ldots,\bm y_{n_0})$ is sampled with replacement from a deterministic set $\bm Z=(\bm z_1,\ldots,\bm z_m)$, then for every $(\varepsilon,\delta)$-differentially private estimator $M$ based on $n_0$ observations there exists a choice of $\bm Z\in\{-\sigma,\sigma\}^{m\times d}$ such that
	\begin{align*}
		\E\|M(\bm Y) - \E\bm y_1\|_2 \gtrsim \sigma\sqrt{d}
	\end{align*}
	where constants may depend on $\omega$.
\end{Lemma}}
Lemma \ref{lm: low-dim mean group privacy} is proved in Section \ref{sec: proof of lm: low-dim mean group privacy}. The discrete uniform distribution described in the lemma is coordinatewise sub-Gaussian$(\sigma)$ because each coordinate is bounded in $[-\sigma,\sigma]$.

On the basis of Lemma \ref{lm: low-dim mean group privacy}, we are able to translate the lower bound result to the more interesting large $n$ regime, as described by the following theorem.
{\begin{Theorem}\label{thm: low-dim mean lower bound}
		Let $\bm X = \{\bm x_1, \bm x_2, \cdots, \bm x_n\}$ be an i.i.d. sample drawn from some distribution in $\mathcal P(\sigma,d,\Theta)$ with mean $\E \bm x_1 = \bmu$. Fix $\omega>0$. If $0 < \varepsilon \leq 1$, $0<\delta\leq n^{-(1+\omega)}$, $d/\log(1/\delta)$ is larger than a sufficiently large constant, and $n \geq C_\omega\sqrt{d\log(1/\delta)}/\varepsilon$, then
	\begin{align}\label{eq: low-dim mean lower bound}
			\inf_{M \in \mathcal M_{\varepsilon, \delta}}\sup_{P\in\mathcal P(\sigma,d,\Theta)}\E_P\|M(\bm X) - \bmu(P)\|_2^2 \gtrsim \sigma^2\left(\frac{d}{n} + \frac{d^2\log(1/\delta)}{n^2\varepsilon^2}\right).
	\end{align}
\end{Theorem}}
The theorem is proved in Section \ref{sec: proof of thm: low-dim mean lower bound}. The minimax lower bound characterizes the cost of privacy in the mean estimation problem: the cost of privacy dominates the statistical risk when ${d\log(1/\delta)}/{n\varepsilon^2} \gtrsim 1$. This minimax lower bound matches the sample complexity lower bound in \cite{steinke2017between}, which considered the deterministic worst case instead of the i.i.d. statistical setting. \cite{kamath2018privately} studied the Gaussian mean estimation problem but did not obtain a tight bound with respect to $\delta$; Theorem \ref{thm: low-dim mean lower bound} improves the lower bound in \cite{kamath2018privately} by $\log(1/\delta)$. {In Section \ref{sec: low-dim mean upper bound}, we give a coordinatewise-truncation estimator that attains this rate for bounded-coordinate models and incurs only the displayed logarithmic truncation cost for general coordinatewise sub-Gaussian observations.}

\subsection{{Private low-dimensional mean estimation}}\label{sec: low-dim mean upper bound}
{In this section, we give a differentially private estimator that attains \eqref{eq: low-dim mean lower bound} for bounded-coordinate models and is within a logarithmic truncation factor for general coordinatewise sub-Gaussian observations.}

{Let $\bm x_1,\ldots,\bm x_n$ be i.i.d. from a coordinatewise sub-Gaussian$(\sigma)$ distribution on $\R^d$, with mean $\bmu$. The following estimator clips each coordinate of each observation, averages, and then applies the Gaussian mechanism.}

\vspace{2mm}
\begin{algorithm}[H]
	\SetAlgoLined
	\SetKwInOut{Input}{Input}
	\SetKwInOut{Output}{Output}
	\Input{Data set $\bm X = \{\bm x_i\}_{i \in [n]}$, privacy parameters $\varepsilon, \delta$, truncation level $R$.}
	Compute $\overline{ \bm X}_R$: for $j \in [d]$, $\overline{\bm X}_{R,j} = n^{-1}\sum_{i \in [n]} \Pi_R(x_{ij})$ \;
		Compute $\hat \bmu = \overline{\bm X}_R + \bm w$, where $\bm w \sim N_{d}\left(\bm 0, \frac{8R^2d\log(1.25/\delta)}{n^2\varepsilon^2} \cdot \bm I_d \right)$ \;
	\Output{$\hat \bmu$.}
	\caption{Differentially Private Mean Estimation}\label{algo: low-dim mean estimation}
\end{algorithm}
\vspace{2mm}
{If $\bm X$ and $\bm X'$ differ in one record, then $\|\overline{\bm X}_R-\overline{\bm X'}_R\|_2\leq2R\sqrt d/n$. Thus the displayed Gaussian noise is calibrated to the exact sensitivity order. The choice of $R$ balances clipping bias against privacy noise.}
\begin{Theorem}\label{thm: low-dim mean upper bound}
	Let $0<\varepsilon\leq1$, $0<\delta<1/2$, and $\mu_\infty=\|\bmu\|_\infty$. If
	\begin{align*}
		R\geq \mu_\infty+K\sigma\sqrt{\log(en)}
	\end{align*}
	for a sufficiently large universal constant $K$, Algorithm \ref{algo: low-dim mean estimation} is $(\varepsilon,\delta)$-differentially private and
	\begin{align*}
		\E\|\hat\bmu-\bmu\|_2^2
		\lesssim \frac{\sigma^2d}{n}
		+\frac{R^2d^2\log(1/\delta)}{n^2\varepsilon^2}.
	\end{align*}
	If $\|\bm x\|_\infty\leq R_0$ almost surely, the same conclusion holds with $R=R_0$ and without the lower bound on $R$.
\end{Theorem}
The theorem is proved in Section \ref{sec: proof of thm: low-dim mean upper bound} of the supplement. In particular, when $\mu_\infty\lesssim\sigma$, taking $R\asymp\sigma\sqrt{\log n}$ gives a factor $\log n$ gap in the privacy term. For bounded-coordinate models with $R_0\lesssim\sigma$, this logarithmic loss disappears.

It should be noted that Algorithm \ref{algo: low-dim mean estimation} lacks some practicality: the truncation level $R$ is a tuning parameter that needs to be set at the correct level for the convergence rate to hold; we included this somewhat simplistic algorithm here for the theoretical analysis of privacy cost. In Section \ref{sec: simulations}, we consider data-driven and differentially private proxies of the theoretical choice of $R$ and demonstrate their numerical performance. As the focus of this paper is theoretical properties of private estimators, we refer interested readers to \cite{karwa2017finite} and \cite{biswas2020coinpress} for more practical methods of differentially private mean estimation.

\subsection{{Lower bounds for sparse mean estimation}}\label{sec: high-dim mean lower bound}
We consider the minimax risk of estimating the mean vector of a sub-Gaussian$(\sigma)$ distribution when the mean is $s^*$-sparse. Concretely, let $\Theta = \{\bmu \in \R^d: \|\bmu\|_0 \leq s^*, \|\bmu\|_\infty < 1\}$, and denote the corresponding distribution class by $\mathcal P(\sigma,d,s^*,\Theta)$. Write
\begin{align*}
	m_{d,s^*}=s^*\log(ed/s^*).
\end{align*}

{The sparse lower bound uses a constant-weight packing together with the private Fano lemma proved in Section \ref{sec: private packing lemma} of the supplement. This argument handles the discrete support uncertainty directly and makes the dependence on approximate privacy explicit.}
{For a real number $x$, write $[x]_+=\max\{x,0\}$. Fix a sufficiently small universal constant $c_\star\in(0,1)$ and, for $0<\varepsilon\leq1$ and $0\leq\delta<1/2$, define the effective private packing entropy
\begin{align}\label{eq: effective private entropy}
	q_{d,s^*}(\varepsilon,\delta)=
	\begin{cases}
		m_{d,s^*},&\delta=0,\\
		\displaystyle\min\left\{m_{d,s^*},
		c_\star\left[\log\left(\frac{c_\star(e^\varepsilon-1)}{\delta}\right)\right]_+\right\},&0<\delta<1/2.
	\end{cases}
\end{align}
The constant inside the logarithm only removes the regime in which approximate DP is too weak for a nontrivial private-packing conclusion.
\begin{Theorem}\label{thm: high-dim mean lower bound}
		Suppose $1\leq s^*\leq d/2$, $0<\varepsilon\leq1$, and $0\leq\delta<1/2$. Then
		 \begin{align}
		 	&\inf_{M \in \mathcal M_{\varepsilon, \delta}}
		 	\sup_{P\in\mathcal P(\sigma, d, s^*,\Theta)}
		 	\E_P\|M(\bm X)-\bmu(P)\|_2^2 \notag\\
		 	&\qquad\gtrsim \min\left\{s^*,\ \sigma^2\left(
		 	\frac{m_{d,s^*}}{n}+
		 	\frac{q_{d,s^*}(\varepsilon,\delta)^2}{n^2\varepsilon^2}\right)\right\}.
		 	\label{eq: high-dim mean lower bound}
		 \end{align}
\end{Theorem}
The proof is given in Section \ref{sec: proof of thm: high-dim mean lower bound} of the supplement. The minimum with $s^*$ is necessary because the parameter space has squared $\ell_2$-diameter of order $s^*$. Since $e^\varepsilon-1\asymp\varepsilon$ for $0<\varepsilon\leq1$, the effective entropy for $\delta>0$ is, up to universal constants and the fixed threshold in \eqref{eq: effective private entropy},
\begin{align*}
	\min\left\{m_{d,s^*},\left[\log\left(\frac{\varepsilon}{\delta}\right)\right]_+\right\}.
\end{align*}
In particular, if $\varepsilon=n^{-b}$ and $\delta=n^{-\kappa}$ with fixed $\kappa>b\geq0$, then
\begin{align*}
	q_{d,s^*}(\varepsilon,\delta)\asymp
	\min\{m_{d,s^*},(\kappa-b)\log n\}.
\end{align*}
Consequently, if $m_{d,s^*}\leq(\log n)^A$ for a fixed $A$, the pure-DP privacy rate is retained up to a factor no larger than a fixed polylogarithm of $n$. If $\delta$ is not asymptotically smaller than $\varepsilon$, the effective entropy may vanish; in that regime the theorem makes only the always-valid non-private assertion.}

\subsection{{Private sparse mean estimation}}\label{sec: high-dim mean upper bound}
{Let $\bm x_1,\ldots,\bm x_n$ be i.i.d. from a coordinatewise sub-Gaussian$(\sigma)$ distribution on $\R^d$, with an $s^*$-sparse mean $\bmu$.}

In this section, we propose a differentially private algorithm for estimating the sparse mean vector $\bmu$. At a high level, the algorithm selects the large coordinates of the (truncated) sample mean vector in a differentially private manner, and sets the remaining coordinates to zero. We start with describing and analyzing the differentially private selection step.

The following ``peeling" algorithm, developed by \cite{dwork2018differentially}, is an efficient and differentially private method for selecting the top-$s$ largest coordinates in terms of absolute value. In each of the $s$ iterations, one coordinate is ``peeled'' from the original vector and added to the output set.

\vspace{2mm}

\begin{algorithm}[H]
	\SetAlgoLined
	\SetKwInOut{Input}{Input}
	\SetKwInOut{Output}{Output}
	\Input{vector-valued function $\bm v = \bm v(\bm X) \in \R^d$, data $\bm X$, sparsity $s$, privacy parameters $\varepsilon, \delta$, noise scale $\lambda$.}
	Initialize $S = \emptyset$\;
	\For{$i$ in $1$ \KwTo $s$}{
		{Generate $\bm w_i \in \R^d$ with $w_{i1}, w_{i2}, \cdots, w_{id} \stackrel{\text{i.i.d.}}{\sim} \text{Laplace}\left(\lambda \cdot \frac{8\sqrt{2s\log(1/\delta)}}{\varepsilon}\right)$\;}
		Append $j^* = \argmax_{j \in [d] \setminus S} |v_j| + w_{ij}$ to $S$\;
	}
	Set $\tilde P_s(\bm v) = \bm v_S$\;
	{Generate $\tilde {\bm w}$ with $\tilde w_{1}, \cdots, \tilde w_{d} \stackrel{\text{i.i.d.}}{\sim} \text{Laplace}\left(\lambda \cdot \frac{8\sqrt{2s\log(1/\delta)}}{\varepsilon}\right)$\;}
	\Output{$\tilde P_s(\bm v) + \tilde {\bm w}_S$.}
	\caption{``Peeling" \cite{dwork2018differentially}}\label{algo: peeling}
\end{algorithm}

\vspace{2mm}

{The algorithm is guaranteed to be differentially private when the vector $\bm v = v(\bm X)$ has bounded change under replace-one adjacency. The exact black-box statement, including the adjacency convention, the $\ell_\infty$-sensitivity hypothesis, the displayed Laplace scale, and the resulting $(\varepsilon,\delta)$ guarantee, is Lemma \ref{lm: peeling privacy} in the supplement.}

Another important property of the Peeling algorithm is its (approximate) accuracy, proved in Section \ref{sec: proof of lm: peeling accuracy} of the supplement.
\begin{Lemma}\label{lm: peeling accuracy}
	Let $S$ and $\{\bm w\}_{i \in [s]}$ be defined as in Algorithm \ref{algo: peeling}. For every $R_1 \subseteq S$ and $R_2 \subseteq S^c$ such that $|R_1| = |R_2|$ and every $c > 0$, we have
	\begin{align*}
		\|\bm v_{R_2}\|_2^2 \leq (1 + c)\|\bm v_{R_1}\|_2^2 + 4(1 + 1/c)\sum_{i \in [s]} \|\bm w_i\|^2_\infty.
	\end{align*}
\end{Lemma}

Returning to sparse mean estimation, we construct a differentially private estimator by applying the ``peeling" algorithm to a truncated sample mean, as follows.

\vspace{2mm}
\begin{algorithm}[H]
	\SetAlgoLined
	\SetKwInOut{Input}{Input}
	\SetKwInOut{Output}{Output}
	\SetKwFunction{Peeling}{Peeling}
	\Input{Data set $\bm X = \{\bm x_i\}_{i \in [n]}$, privacy parameters $\varepsilon, \delta$, truncation level $R$, sparsity $s$.}
		Compute $\overline{ \bm X}_R$: for $j \in [d]$, $\overline{\bm X}_{R,j} = n^{-1}\sum_{i \in [n]} \Pi_R(x_{ij})$ \;
		Compute $\hat \bmu = \Peeling(\overline{\bm X}_R, \bm X, s, \varepsilon, \delta, 2R/n)$\;
		\Output{$\hat \bmu$.}
		\caption{Differentially Private Sparse Mean Estimation}\label{algo: high-dim mean estimation}
\end{algorithm}
\vspace{2mm}

{Under replace-one adjacency,
$\|\overline{\bm X}_R-\overline{\bm X'}_R\|_\infty\leq2R/n$.
Substituting $\lambda=2R/n$ into Lemma \ref{lm: peeling privacy} shows that every scalar selection and release noise in Algorithm \ref{algo: high-dim mean estimation} has scale $	\frac{16R\sqrt{2s\log(1/\delta)}}{n\varepsilon}$,
and proves $(\varepsilon,\delta)$-privacy. The accuracy analysis combines coordinatewise concentration of the clipped mean with the approximate support-selection guarantee in Lemma \ref{lm: peeling accuracy}.}

\begin{Theorem}\label{thm: high-dim mean upper bound}
		Let $0<\varepsilon\leq1$, $0<\delta<1/2$, $0<\eta<1/4$, {$s^*\leq s\leq\min\{Cs^*,d\}$,} and $R\geq\|\bmu\|_\infty+K\sigma\sqrt{\log(en)}$. Then, with probability at least $1-\eta$, Algorithm \ref{algo: high-dim mean estimation} satisfies
		\begin{align*}
			\|\hat \bmu - \bmu\| _2^2 \lesssim
			\frac{\sigma^2s^*\log(d/\eta)}{n}
			+\frac{R^2(s^*\log(d/\eta))^2\log(1/\delta)}{n^2\varepsilon^2}.
		\end{align*}
\end{Theorem}
{Theorem \ref{thm: high-dim mean upper bound} is proved in Section \ref{sec: proof of thm: high-dim mean upper bound} of the supplement. When $\|\bmu\|_\infty\lesssim\sigma$, taking $R\asymp\sigma\sqrt{\log n}$ gives the displayed rate. In the polynomial-$\delta$ regime described after Theorem \ref{thm: high-dim mean lower bound}, the lower and upper privacy terms differ only by explicit polylogarithmic factors.}

\section{The Cost of Privacy in Linear Regression}\label{sec: regression}
{We next study Gaussian linear regression. The model is
\begin{align}\label{eq: linear model}
	y=\bm x^\top\bbeta+\xi,\qquad
	\xi\sim N(0,\sigma^2),\qquad \xi\ \text{independent of }\bm x\sim f_{\bm x}.
\end{align}}
{Let $f_{\bbeta}$ denote the joint law induced by \eqref{eq: linear model}. Given i.i.d. observations $D=(\bm y,\bm X)=\{(y_i,\bm x_i)\}_{i=1}^n$, we estimate $\bbeta\in\R^d$ under $(\varepsilon,\delta)$-differential privacy. We first treat the low-dimensional case and then impose sparsity for high-dimensional regression.}

\subsection{{Lower bounds for low-dimensional linear regression}}\label{sec: low-dim regression lower bound}
{Let $\mathcal P(\sigma,d,\Theta)$ be the class of distributions satisfying \eqref{eq: linear model} with $\bbeta\in\Theta=\{\bbeta\in\R^d:\|\bbeta\|_2\leq1\}$. We lower-bound
\begin{align*}
	\inf_{M\in\mathcal M_{\varepsilon,\delta}}
	\sup_{P\in\mathcal P(\sigma,d,\Theta)}
	\E_P\|M(D)-\bbeta(P)\|_2^2
\end{align*}
using a tracing attack.}

Consider the attack given by
\begin{align}\label{eq: low-dim regression attack}
	\A_{\bbeta} ((y, {\bm x}), M(\bm y, \bm X)) = \big\langle M(\bm y, \bm X) - \bbeta, ( y - {\bm x}^\top\bbeta) {\bm x} \big\rangle.
\end{align}
{The score is centered when $(y,\bm x)$ is an independent replacement, but its sum over the observed records is large for an accurate estimator. Differential privacy prevents these two behaviors from being too far apart.}

{\begin{Lemma}\label{lm: low-dim regression attack}
	Let $(\bm y, \bm X)$ be an i.i.d. sample drawn from some distribution in $\mathcal P(\sigma, d, \Theta)$ such that $\|\bm x\|_\infty \leq 1$ with probability $1$, and $\Sigma_{\bm x} = \E \bm x\bm x^\top$ is diagonal and satisfies $L^{-1}\leq \lambda_{\min}(\Sigma_{\bm x}) \leq \lambda_{\max}(\Sigma_{\bm x})\leq L$ for some constant $L\geq1$. Let $M$ be any estimator taking values in $\Theta$. Then the following are true.
	\begin{enumerate}
		\item For each $i \in [n]$, let $(\bm y'_i, \bm X'_i)$ denote the data set obtained by replacing $(y_i, \bm x_i)$ in $(\bm y, \bm X)$ with an independent copy, then $\E \A_\bbeta ((y_i, \bm x_i), M(\bm y'_i, \bm X'_i)) = 0$ and
		\begin{align*}
				\E |\A_\bbeta ((y_i, \bm x_i), M(\bm y'_i, \bm X'_i))| \leq \sigma \sqrt{\E\|M(\bm y, \bm X) - \bbeta\|_{\Sigma_{\bm x}}^2}.
		\end{align*}
	\item There exists a product prior $\bm\pi$ supported on $\Theta$ such that
		\begin{align}
			&\sum_{i \in [n]}\E_{\bm \pi}\E_{\bm y, \bm X|\bbeta}
			\A_{\bbeta}((y_i, \bm x_i), M(\bm y, \bm X)) \notag\\
			&\qquad\geq \sigma^2d\left[1-\pi
			\sqrt{\E_{\bm\pi}\E_{\bm y,\bm X|\bbeta}
			\|M(\bm y,\bm X)-\bbeta\|_2^2}\right].
			\label{eq: regression attack quantitative completeness}
		\end{align}
	\end{enumerate}
\end{Lemma}}
The lemma is proved in Section \ref{sec: proof of lm: low-dim regression attack} of the supplement. These properties of tracing attack imply a minimax lower bound for $(\varepsilon, \delta)$-differentially private estimation of $\bbeta$.

{\begin{Theorem}\label{thm: low-dim regression lower bound}
	Suppose $0 < \varepsilon \leq 1$, $0\leq\delta<\delta_0$ for a sufficiently small universal constant $\delta_0$, and, with $\delta\sqrt{\log(1/\delta)}$ interpreted as zero at $\delta=0$,
	\begin{align*}
		n\delta\sqrt{\log(1/\delta)}\leq c_0\sigma\sqrt d.
	\end{align*}
	Then
	\begin{align}\label{eq: low-dim regression lower bound}
		\inf_{M \in \mathcal M_{\varepsilon, \delta}}
		\sup_{P\in\mathcal P(\sigma, d,  \Theta)}
		\E_P\|M(\bm y, \bm X) - \bbeta(P)\|_{\Sigma_{\bm x}}^2
		\gtrsim \min\left\{1,\ \sigma^2\left(\frac{d}{n} + \frac{d^2}{n^2\varepsilon^2}\right)\right\}.
	\end{align}
\end{Theorem}
The lower bound is proved in Section \ref{sec: proof of thm: low-dim regression lower bound} of the supplement. The displayed condition is exactly what is needed to absorb the approximate-DP tail remainder in the tracing argument; for fixed $\sigma>0$, it follows from $\delta<n^{-(1+\omega)}$ for all sufficiently large $n$. The cap by one reflects the diameter of $\Theta$.}

\subsection{{Private low-dimensional linear regression}}\label{sec: low-dim regression upper bound}
{For the upper bound, we approximately minimize the least-squares objective}
\begin{align*}
	\L_n(\bbeta) = \frac{1}{2n}\sum_{i=1}^n (y_i - \bm x_i^\top\bbeta)^2.
\end{align*}
{using a split-sample noisy-gradient algorithm. Each iteration uses a fresh block, so the current iterate is independent of the data used for its next gradient. To obtain finite global sensitivity without assuming bounded covariates, the algorithm first projects every design vector onto a deterministic Euclidean ball. This costs only a high-probability logarithmic factor under the sub-Gaussian design condition.}

{For $B_2>0$, define the Euclidean design-clipping map
\begin{align}\label{eq: l2 design clipping}
	\mathsf C_{2,B_2}(\bm x)
	=\bm x\min\left\{1,\frac{B_2}{\|\bm x\|_2}\right\},
\end{align}
with $\mathsf C_{2,B_2}(\bm0)=\bm0$. Thus $\|\mathsf C_{2,B_2}(\bm x)\|_2\leq B_2$ for every possible input.}
\vspace{2mm}

\begin{algorithm}[H]
	\SetAlgoLined
	\SetKwInOut{Input}{Input}
	\SetKwInOut{Output}{Output}
		\Input{{Data set $\{(y_i, \bm x_i)\}_{i \in [n]}$, step size $\eta^0$, privacy parameters $\varepsilon, \delta$, number of iterations $T$, residual-clipping level $R$, design-clipping level $B_2$, feasibility radius $C$, initial value $\bbeta^0$.}}
		{Set $\overline{\bm x}_i=\mathsf C_{2,B_2}(\bm x_i)$ for every $i\in[n]$\;}
		{Independently of the data, randomly partition the observations into disjoint blocks $I_0,\ldots,I_{T-1}$ of size $b=\lfloor n/T\rfloor$ (discarding at most $T-1$ observations)\;}
		\For{$t$ in $0$ \KwTo $T-1$}{
			{Generate $\bm w_t\sim N_d\left(\bm0,\,32(\eta^0)^2R^2B_2^2\,\frac{\log(1.25/\delta)}{b^2\varepsilon^2}\bm I_d\right)$\;}
			{Compute $\bbeta^{t + 1} = \Pi_C\left(\bbeta^t - (\eta^0/b)\sum_{i\in I_t}  \Pi_{R}(\overline{\bm x}_i^\top \bbeta^t- y_i)\overline{\bm x}_i + \bm w_t\right)$\;}
	}

	\Output{$\bbeta^T$.}
	\caption{Differentially Private Linear Regression}\label{algo: low-dim regression}
\end{algorithm}

\vspace{2mm}

{The accuracy analysis of Algorithm \ref{algo: low-dim regression} uses the following assumptions; its privacy guarantee is distribution free.}
\begin{itemize}
	\item[{(D)}] {Sub-Gaussian well-conditioned design: $\E\bm x=\bm0$, $\sup_{\|u\|_2=1}\|u^\top\bm x\|_{\psi_2}\leq K_{\bm x}$, and $L^{-1}\bm I_d\preceq\Sigma_{\bm x}\preceq L\bm I_d$.}
	\item[{(P)}] {The true parameter satisfies $\|\bbeta\|_2\leq c_0$, and the projection radius obeys $C\geq c_0$.}
\end{itemize}
{The deterministic map \eqref{eq: l2 design clipping} controls global sensitivity on every data set. Assumption (D) is used only to show that this additional clipping is inactive with high probability and to obtain concentration and contraction of the empirical gradient; (P) ensures that projection does not exclude the true parameter.}

{The calibrated Gaussian perturbation guarantees privacy for each block; disjointness of the blocks then gives privacy for the complete algorithm.}
\begin{Lemma}\label{lm: low-dim regression privacy}
		{For $0<\varepsilon\leq1$, $0<\delta<1/2$, every $B_2,R,C>0$, and every input distribution, Algorithm \ref{algo: low-dim regression} is $(\varepsilon, \delta)$-differentially private.}
\end{Lemma}

{The lemma is proved in Section \ref{sec: proof of lm: low-dim regression privacy} of the supplement. Under (D), the following theorem describes the convergence rate of the clipped-design algorithm.}

\begin{Theorem}\label{thm: low-dim regression upper bound}
		Let $0<\varepsilon\leq1$, $0<\delta<1/2$, $0<\eta<1/4$, and {suppose (D) and (P) hold. There are constants $C_i$, and an implicit constant in the bound below, depending only on $L$ and $K_{\bm x}$. Set}
		\begin{align*}
			\eta^0&=\frac{1}{4L},\qquad
			T=\lceil C_1L^2\log(en)\rceil,\\
			{B_2}&={C_2K_{\bm x}\sqrt{d+\log(8n/\eta)}},\\
			{R}&={C_3(\sigma+K_{\bm x}C)\sqrt{\log(8nT/\eta)}},
			\qquad \bbeta^0=\bm0.
		\end{align*}
		If {$b=\lfloor n/T\rfloor\geq C_4\{d+\log(8T/\eta)\}$}, then with probability at least $1-\eta$,
		\begin{align}
			\|\bbeta^T - \bbeta\|_2^2
			\lesssim
			{\frac{\sigma^2\{d+\log(8T/\eta)\}}{b}
			+\frac{B_2^2R^2\{d+\log(8T/\eta)\}\log(1/\delta)}{b^2\varepsilon^2}}
			+\frac1n. \label{eq: non-sparse glm upper bound}
		\end{align}
\end{Theorem}
Theorem \ref{thm: low-dim regression upper bound} is proved in Section \ref{sec: proof of thm: low-dim regression upper bound} of the supplement. {Since $b\asymp n/\log n$, $B_2^2\lesssim K_{\bm x}^2\{d+\log(n/\eta)\}$, and $R^2\lesssim(\sigma+K_{\bm x}C)^2\log(nT/\eta)$, this result matches the lower bound up to polylogarithmic factors when $L$ and $K_{\bm x}$ are fixed. The additional logarithmic loss is the price of obtaining a global privacy guarantee without an almost-sure design bound.} The sample split is essential in the proof: it makes $\bbeta^t$ independent of the data block used at iteration $t$, so the sub-Gaussian residual bound is valid conditionally on the past.

\subsection{{Lower bounds for sparse linear regression}}\label{sec: high-dim regression lower bound}
{We next allow $d$ to exceed $n$ and assume that $\bbeta$ is sparse. Let $\mathcal P(\sigma,d,s^*,\Theta)$ denote the class in \eqref{eq: linear model} with $\Theta=\{\bbeta\in\R^d:\|\bbeta\|_0\leq s^*,\|\bbeta\|_2\leq1\}$, and write $m_{d,s^*}=s^*\log(ed/s^*)$. We use the effective entropy $q_{d,s^*}(\varepsilon,\delta)$ defined in \eqref{eq: effective private entropy}.}

For the lower bound it suffices to use a submodel with a fixed, parameter-independent design distribution under which the coordinates of $\bm x$ are independent Rademacher random variables. This submodel has $\Sigma_{\bm x}=\bm I_d$ and does not reveal, or depend on, the unknown support of $\bbeta$. A constant-weight packing and the private Fano inequality then yield the following result.

{\begin{Theorem}\label{thm: high-dim regression lower bound}
	Suppose $1\leq s^*\leq d/2$, $0<\varepsilon\leq1$, and $0\leq\delta<1/2$. Then
	\begin{align}
		&\inf_{M \in \mathcal M_{\varepsilon, \delta}}
		\sup_{P\in\mathcal P(\sigma, d, s^*,\Theta)}
		\E_P\|M(\bm y,\bm X)-\bbeta(P)\|_2^2 \notag\\
		&\qquad\gtrsim \min\left\{1,\ \sigma^2\left(
		\frac{m_{d,s^*}}{n}+
		\frac{q_{d,s^*}(\varepsilon,\delta)^2}{n^2\varepsilon^2}\right)\right\}.
		\label{eq: high-dim regression lower bound}
	\end{align}
\end{Theorem}
The proof is given in Section \ref{sec: proof of thm: high-dim regression lower bound} of the supplement. The cap by $1$ is forced by the diameter of $\Theta$. The same $(\varepsilon,\delta)$-dependent and polynomial-gap interpretation following Theorem \ref{thm: high-dim mean lower bound} applies.}

\subsection{{Private sparse linear regression}}\label{sec: high-dim regression upper bound}
{When $d>n$, the empirical Hessian $n^{-1}\bm X^\top\bm X$ is singular, so the low-dimensional contraction argument cannot hold on all of $\R^d$. Sparsity restores contraction on low-dimensional coordinate subspaces.}

{Accordingly, we combine the split-sample gradient step with the private Peeling projection from Section \ref{sec: high-dim mean upper bound}. To obtain a finite coordinatewise sensitivity without an almost-sure design bound, we clip each coordinate before forming the gradient. For $B_\infty>0$, write
\begin{align}\label{eq: linf design clipping}
	\mathsf C_{\infty,B_\infty}(\bm x)
	=\bigl(\Pi_{B_\infty}(x_1),\ldots,\Pi_{B_\infty}(x_d)\bigr)^\top .
\end{align}
Then $\|\mathsf C_{\infty,B_\infty}(\bm x)\|_\infty\leq B_\infty$ for every input $\bm x$.}

\vspace{2mm}

\begin{algorithm}[H]
	\SetAlgoLined
	\SetKwInOut{Input}{Input}
	\SetKwInOut{Output}{Output}
	\SetKwFunction{Peeling}{Peeling}
		\Input{{Data set $(\bm y, \bm X) = \{(y_i, \bm x_i)\}_{i \in [n]}$, step size $\eta^0$, privacy parameters $\varepsilon, \delta$, number of iterations $T$, residual-clipping level $R$, design-clipping level $B_\infty$, feasibility radius $C$, sparsity $s$, initial value $\bbeta^0$.}}
		{Set $\overline{\bm x}_i=\mathsf C_{\infty,B_\infty}(\bm x_i)$ for every $i\in[n]$\;}
		{Independently of the data, randomly partition the observations into disjoint blocks $I_0,\ldots,I_{T-1}$ of size $b=\lfloor n/T\rfloor$ (discarding at most $T-1$ observations)\;}
		\For{$t$ in $0$ \KwTo $T-1$}{
			{Compute $\bbeta^{t + 0.5} = \bbeta^t - (\eta^0/b)\sum_{i\in I_t}  \Pi_{R}(\overline{\bm x}_i^\top \bbeta^t-y_i)\overline{\bm x}_i$\;}
			{$\bbeta^{t+1} = \Pi_C\left(\Peeling(\bbeta^{t+0.5}, (\bm y_{I_t}, \overline{\bm X}_{I_t}), s, \varepsilon, \delta, 2\eta^0RB_\infty/b)\right)$.}
	}

	\Output{$\bbeta^T$.}
	\caption{Differentially Private Sparse Linear Regression}\label{algo: high-dim regression}
\end{algorithm}

\vspace{2mm}

If the ``Peeling" step is replaced by non-private, exact projection of the gradient step onto $\{\bm v \in \R^d: \|\bm v\|_0 \leq s\}$, we recover the well-known iterative hard thresholding algorithm \cite{blumensath2009iterative, jain2014iterative} for high-dimensional sparse regression.

{The accuracy analysis of Algorithm \ref{algo: high-dim regression} uses the following sparse analogues of (P) and (D).}
\begin{itemize}
	\item[{(P')}] {The true parameter vector satisfies $\|\bbeta\|_2\leq c_0$, $\|\bbeta\|_0\leq s^*$, and $C\geq c_0$.}
	\item[{(D')}] {Sparse sub-Gaussian design: $\E\bm x=\bm0$, $\sup_{\|u\|_2=1}\|u^\top\bm x\|_{\psi_2}\leq K_{\bm x}$, and, for a fixed universal constant $C_0\geq3$, every principal submatrix $\Sigma_I$ with $|I|\leq C_0s$ has eigenvalues in $[L^{-1},L]$.}
\end{itemize}
{Assumption (D') is needed only on coordinate sets large enough to contain the current support, the selected support, and the true support. The deterministic clipping map \eqref{eq: linf design clipping}, rather than (D'), controls the sensitivity of each block.}
\begin{Lemma}\label{lm: high-dim regression privacy}
	{For $0<\varepsilon\leq1$, $0<\delta<1/2$, every $B_\infty,R,C>0$, and every input distribution, Algorithm \ref{algo: high-dim regression} is $(\varepsilon, \delta)$-differentially private.}
\end{Lemma}

{The lemma is proved in Section \ref{sec: proof of lm: high-dim regression privacy} of the supplement. Assumption (D') gives the following convergence guarantee.}

\begin{Theorem}\label{thm: high-dim regression upper bound}
	Let $0<\varepsilon\leq1$, $0<\delta<1/2$, $0<\eta<1/4$, and {suppose (P') and (D') hold. There are constants $C_i$, and an implicit constant in the bound below, depending only on $L$ and $K_{\bm x}$, with $C_2>0$ chosen sufficiently small, such that, with}
		\begin{align*}
		{s}&={\lceil C_1s^*\rceil},\qquad \eta^0=C_2/L,
		\qquad T=\lceil C_3L^2\log(en)\rceil,\\
		{B_\infty}&={C_4K_{\bm x}\sqrt{\log(10nd/\eta)}},\\
		{R}&={C_5(\sigma+K_{\bm x}C)\sqrt{\log(10nT/\eta)}},
		\qquad \bbeta^0=\bm0,
	\end{align*}
	and assume {$s\leq d/2$ and $b=\lfloor n/T\rfloor\geq C_6\{s\log(ed/s)+\log(10T/\eta)\}$}. Then, with probability at least $1-\eta$,
	\begin{align}
		\|\bbeta^T - \bbeta\|_2^2
		\lesssim
		{\frac{\sigma^2s\log(10dT/\eta)}{b}
		+\frac{B_\infty^2R^2s^2\log^2(10dT/\eta)\log(1/\delta)}{b^2\varepsilon^2}}
		+\frac1n. \label{eq: high-dim regression upper bound}
	\end{align}
\end{Theorem}
The theorem is proved in Section \ref{sec: proof of thm: high-dim regression upper bound}. {Since $b\asymp n/\log n$, $B_\infty^2\lesssim K_{\bm x}^2\log(nd/\eta)$, and $R^2$ is logarithmic in $nT/\eta$, removing the almost-sure design bound costs only an additional logarithmic factor. Thus the theorem retains only polylogarithmic losses relative to the corresponding sparse lower bound whenever $m_{d,s^*}$ is polylogarithmic in $n$ and $L,K_{\bm x}$ are fixed.} The independent-block construction is used for the theorem; the full-data implementation in the numerical section should be viewed as a computational heuristic rather than as an instance covered by this proof.

\section{Simulation Studies}\label{sec: simulations}
{We use simulations to illustrate how privacy noise, clipping, and sparsity tuning affect finite-sample performance. These experiments are descriptive and are not used to establish the theoretical rates. Unless stated otherwise, the displayed privacy parameters refer to the estimation step; an end-to-end private workflow that tunes parameters on the same data must split the privacy budget between tuning and estimation, as dictated by composition. The data are generated as follows.}
\begin{description}
	{\item[\textbf{Mean estimation}] The observations $\bm x_1,\ldots,\bm x_n$ are drawn independently from $N_d(\bmu,\bm I_d)$. In each repetition of the low-dimensional experiment, the coordinates of $\bmu$ are drawn independently from $\operatorname{Uniform}(-10,10)$. In the sparse experiment, only the first $s^*$ coordinates are drawn from this distribution; the remaining coordinates are zero.}

	{\item[\textbf{Linear regression}] We generate $y_i=\bm x_i^\top\bbeta+\epsilon_i$, where the entries of $\bm x_i$ are independent $\operatorname{Uniform}(-d^{-1/2},d^{-1/2})$ variables and $\epsilon_i\stackrel{\mathrm{i.i.d.}}{\sim}N(0,1)$. This design is centered, sub-Gaussian, and has covariance $(3d)^{-1}\bm I_d$; it is also bounded, so the added design-clipping step is inactive for any $B_2\geq1$ or $B_\infty\geq d^{-1/2}$. In the low-dimensional experiment, $\bbeta$ is uniform on the unit sphere in $\R^d$. In the sparse experiment, its first $s^*$ coordinates are uniform on the unit sphere in $\R^{s^*}$ and the remaining coordinates are zero.}
\end{description}

{We consider three questions:}
\begin{itemize}
	\item Compare the performance of our algorithms under different choices of $R$, the truncation tuning parameter.
	\item Compare the performance of the high-dimensional algorithms under different choices of $s$, the sparsity tuning parameter.
	\item Compare our algorithms with their non-private counterparts, and with other differentially private algorithms in the literature.

\end{itemize}

\subsection{{Tuning the truncation level}}\label{sec: truncation tuning}
{For each problem, we compare three clipping rules.}
\begin{itemize}
	{\item No clipping. For Gaussian data this is a diagnostic benchmark, not an end-to-end differentially private method, because global sensitivity is unbounded.}
	{\item The nominal theoretical scaling $R=4\sigma\sqrt{\log n}$. This choice illustrates the $\sqrt{\log n}$ scaling; the formal theorems also require an offset that covers the mean or the current regression signal.}
	{\item A data-driven rule based on differentially private estimates of the $2.5\%$ and $97.5\%$ quantiles, computed by Algorithm $1'$ of \cite{lei2011differentially}, followed by clipping at those levels.}
\end{itemize}
{Figure \ref{fig: t tuning} shows that the data-driven rule has error comparable to the two benchmarks in these Gaussian experiments. This observation should not be extrapolated to asymmetric or heavy-tailed distributions, for which fixed central quantiles can introduce substantial bias.}

\begin{figure}[ht]
	\centering
	\subfloat[][]{\includegraphics[width=.40\textwidth]{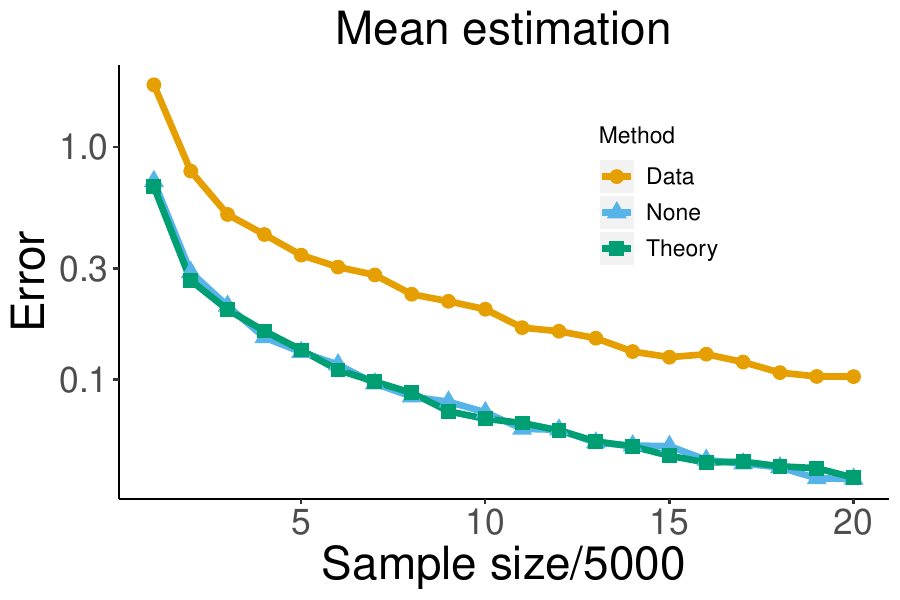}}\quad
	\subfloat[][]{\includegraphics[width=.40\textwidth]{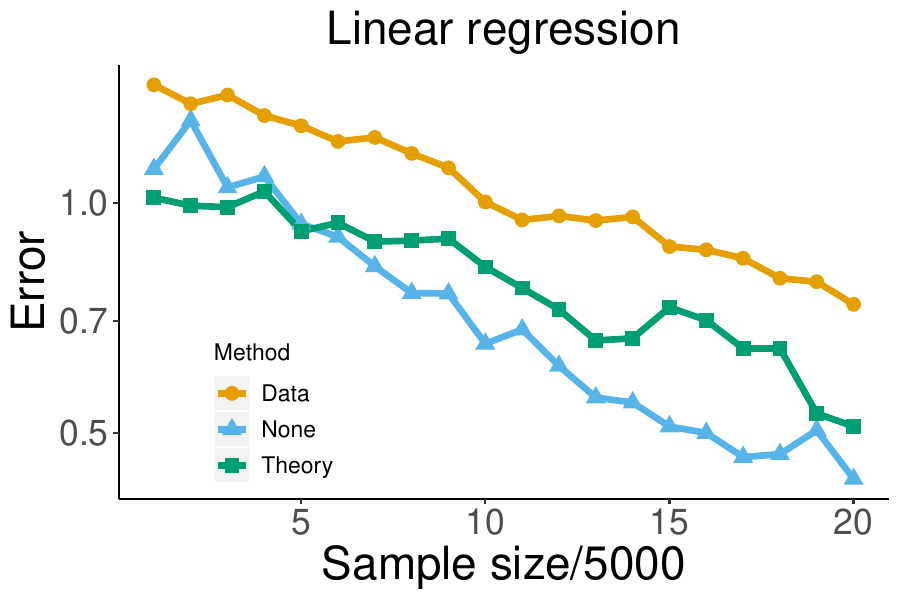}}\quad \\
	\subfloat[][]{\includegraphics[width=.40\textwidth]{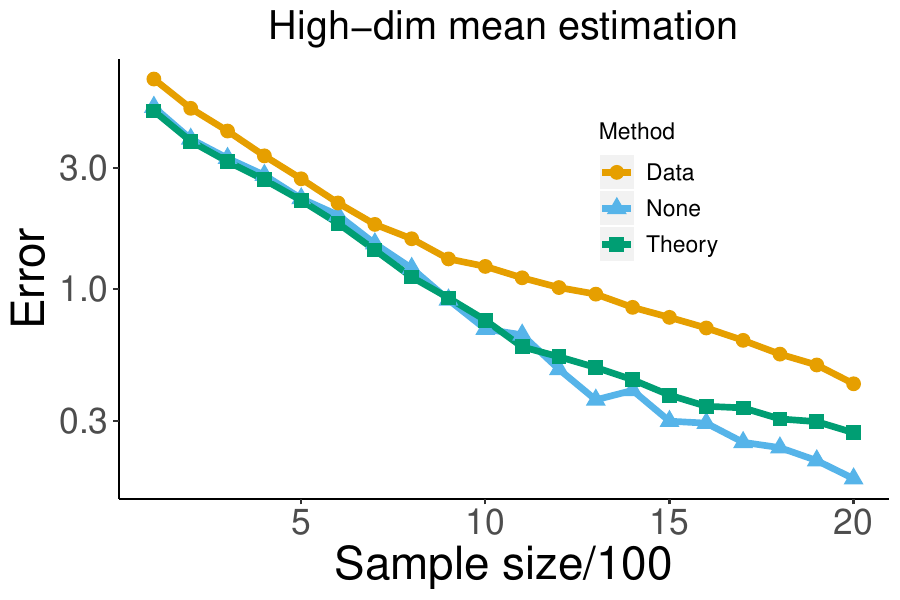}}\quad
	\subfloat[][]{\includegraphics[width=.40\textwidth]{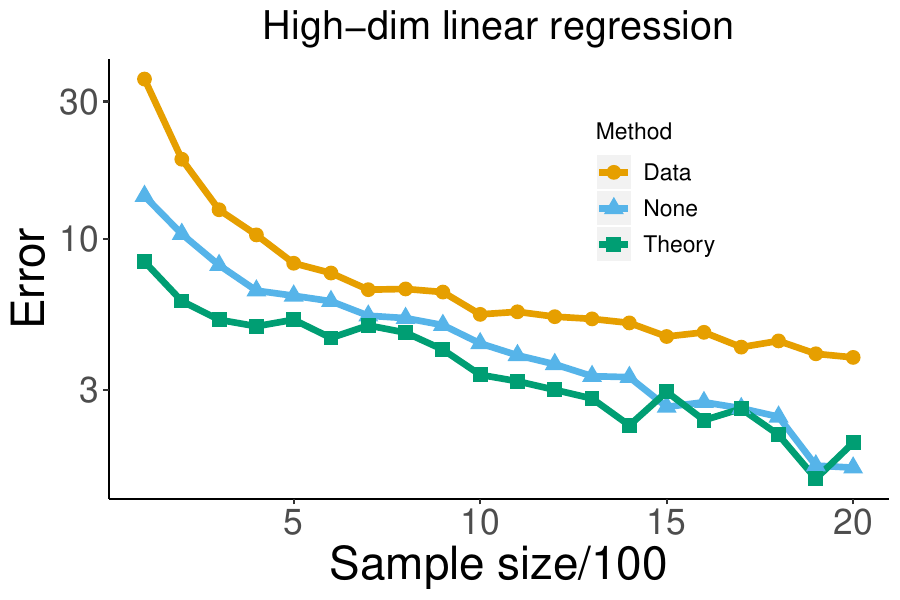}}\quad
	\caption{		\fontsize{10pt}{10}\selectfont
	{Average $\ell_2$ error over 100 repetitions versus sample size $n$, with the estimation step using $(\varepsilon,\delta)=(0.5,10n^{-1.1})$. Panels (a) and (b) show mean estimation and linear regression with $d=20$ and $n$ from 5000 to 100000. Panels (c) and (d) show sparse mean estimation and sparse regression with $n$ from 100 to 2000, $d=n$, and $s^*=s=20$.}}
\label{fig: t tuning}
\end{figure}

\subsection{{Tuning the sparsity level}}\label{sec: sparsity tuning}
{The sparse algorithms require a tuning value $s$. We compare the oracle choice $s=s^*$ with a five-fold cross-validation heuristic over a uniform grid from $s^*/2$ to $2s^*$. The reported implementation clips the validation errors using Algorithm $1'$ of \cite{lei2011differentially} and then uses an exponential-mechanism-style selection rule \cite{mcsherry2007mechanism}. Because changing a training record can also change the fitted candidate models, clipping only the final validation errors does not by itself establish an end-to-end sensitivity bound for this entire cross-validation pipeline. Accordingly, we treat the data-driven curve as a tuning experiment, not as a theorem-covered private release. A fully private implementation must account for the privacy cost of all candidate fits, use sensitivity-controlled validation (potentially on an independent holdout sample), and compose these costs with the final fit.}

\begin{figure}[H]
	\centering
	\subfloat[][]{\includegraphics[width=.40\textwidth]{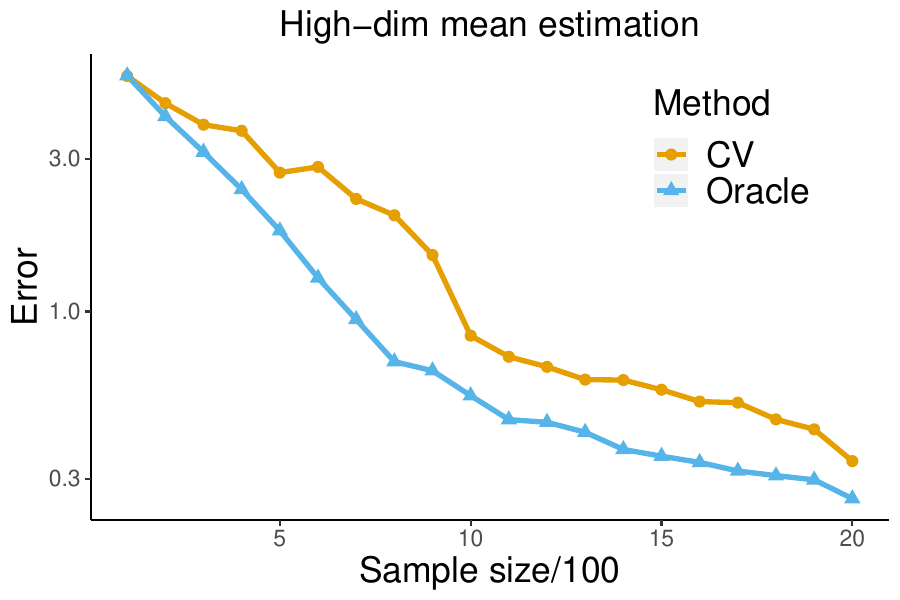}}\quad
	\subfloat[][]{\includegraphics[width=.40\textwidth]{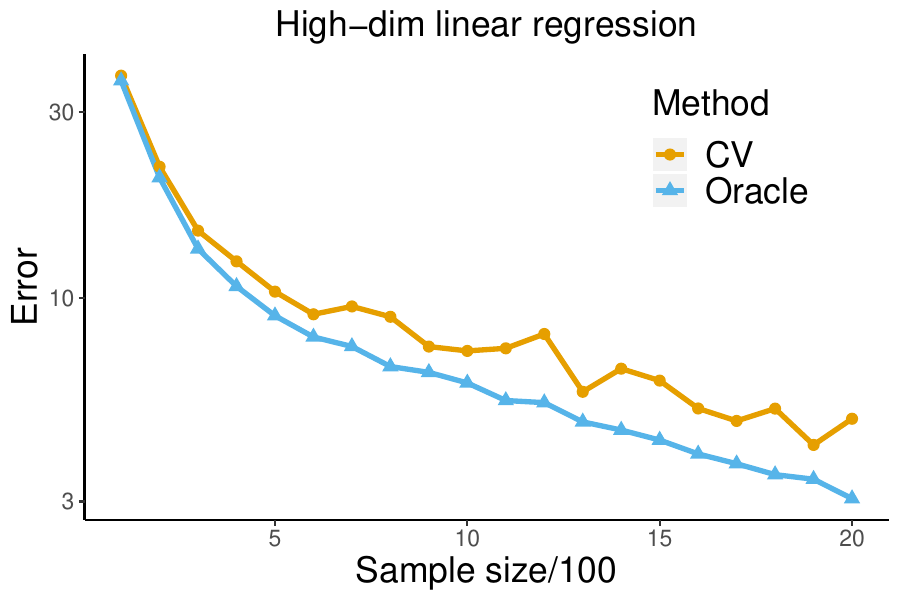}}\quad
	\caption{		\fontsize{10pt}{10}\selectfont
		{Average $\ell_2$ error over 50 repetitions versus sample size $n$, with the estimation step using $(\varepsilon,\delta)=(0.5,10n^{-1.1})$. Panels (a) and (b) show sparse mean estimation and sparse regression with $n$ from 100 to 2000, $d=n$, and true sparsity $s^*=20$.}}
	\label{fig: s cv}
\end{figure}
{Informed by the preceding experiment, we use the data-driven clipping rule here. Figure \ref{fig: s cv} shows that the cross-validation heuristic has error comparable to the oracle choice $s=s^*$ in these simulations; as noted above, this statement is numerical rather than an end-to-end privacy guarantee.}

\subsection{{Comparisons with other algorithms}}\label{sec: comparison with other algos}
{We next compare the proposed implementations with their non-private counterparts and with methods from the literature. For the low-dimensional problems, the comparison includes Algorithm 4 of \cite{karwa2017finite} for mean estimation and Algorithm 1 of \cite{sheffet2017differentially} for regression. For the sparse problems, it includes the method of \cite{talwar2015nearly}. The methods target different losses and inferential goals, so the plots are descriptive comparisons under the stated simulation design rather than claims of uniform dominance.}

{Figure \ref{fig: comparisons} reports these comparisons after the tuning choices have been fixed.}
\begin{figure}[H]
	\centering
	\subfloat[][]{\includegraphics[width=.40\textwidth]{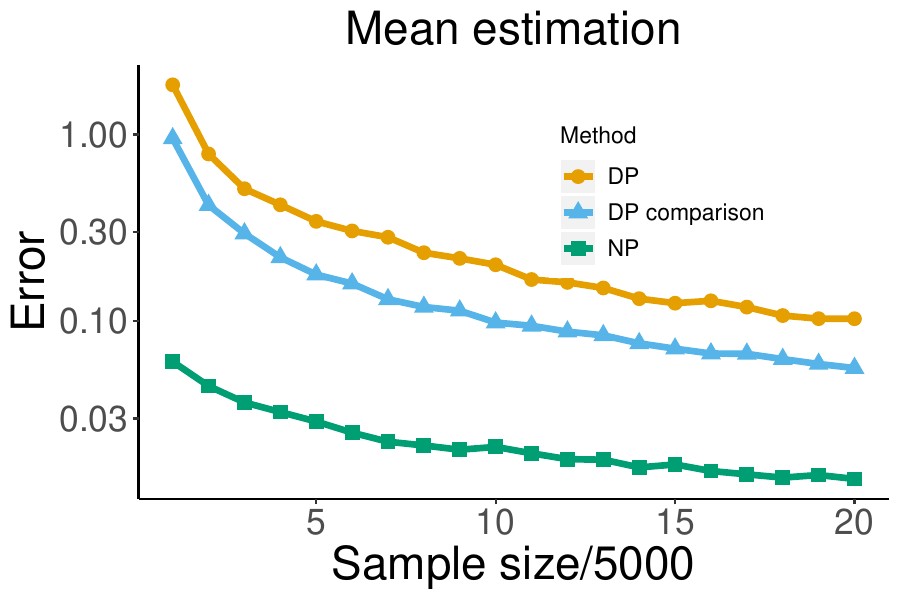}}\quad
	\subfloat[][]{\includegraphics[width=.40\textwidth]{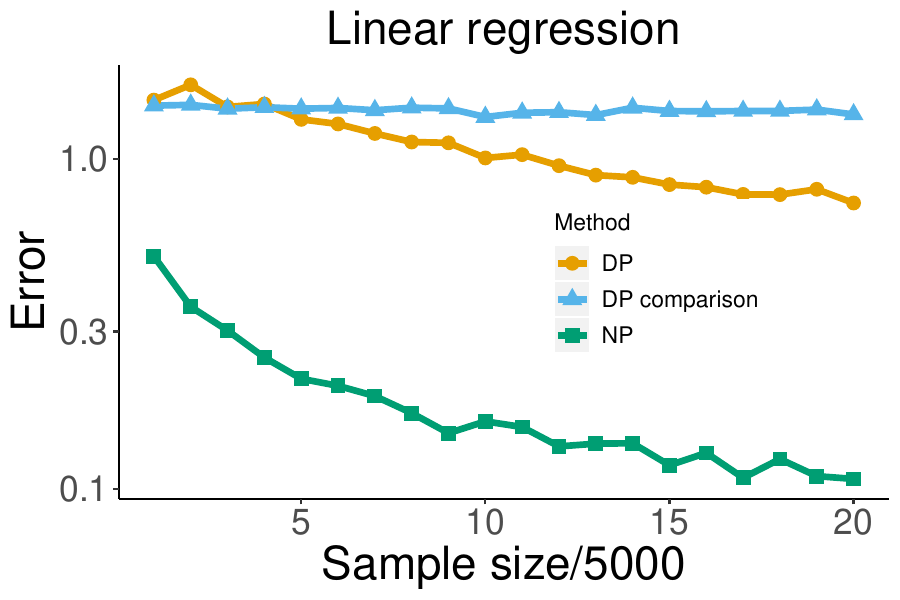}}\quad \\
	\subfloat[][]{\includegraphics[width=.40\textwidth]{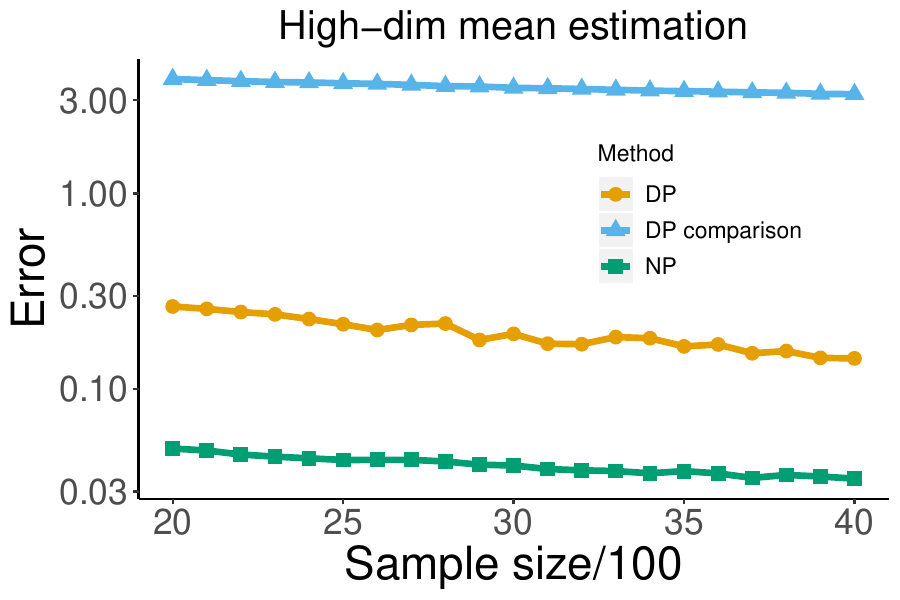}}\quad
	\subfloat[][]{\includegraphics[width=.40\textwidth]{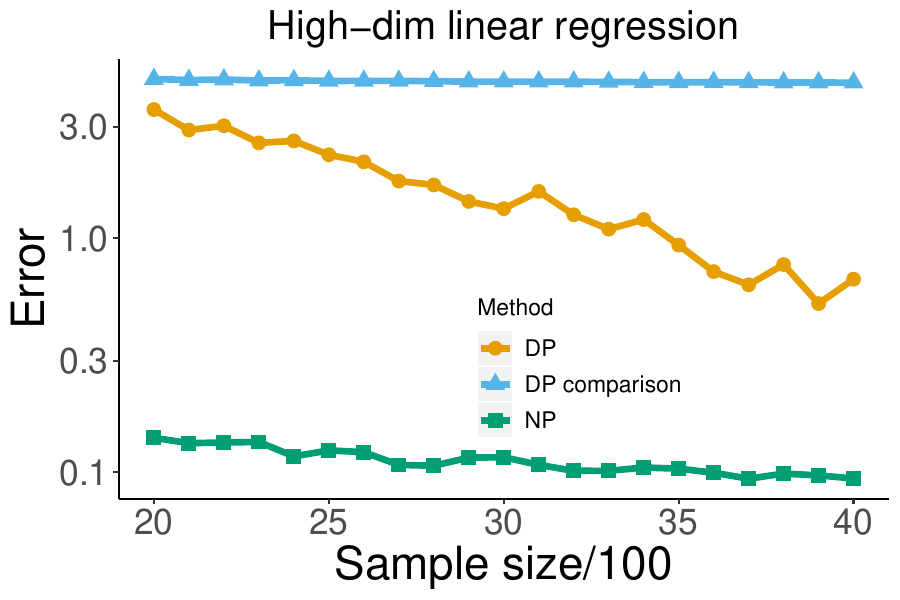}}\quad
	\caption{		\fontsize{10pt}{10}\selectfont
		{Average $\ell_2$ error over 100 repetitions versus sample size $n$, with the estimation step using $(\varepsilon,\delta)=(0.5,10n^{-1.1})$. Panels (a) and (b) show mean estimation and linear regression with $d=20$ and $n$ from 5000 to 100000. Panels (c) and (d) show sparse mean estimation and sparse regression with $n$ from 2000 to 4000, $d=2n$, and $s^*=s=20$.}}
	\label{fig: comparisons}
\end{figure}

{In these experiments, the parameter error of the proposed implementations is smaller than that of the displayed alternatives. This comparison must be read in light of their different objectives: the Johnson--Lindenstrauss procedure of \cite{sheffet2017differentially} was designed in part to support private inference, whereas \cite{talwar2015nearly} targets excess empirical risk. Inspection of the fitted solutions suggests two design-specific explanations for the observed gaps. The ridge subroutine in the former procedure is activated frequently and introduces visible shrinkage bias; the Frank--Wolfe solutions in the latter comparison are typically not sparse and retain appreciable mass outside the true support. These observations explain the plotted behavior but are not general performance theorems about either method.}

\section{Data Analysis}\label{sec: data analysis}
{We conclude with two data illustrations. The full-sample non-private quantities used below serve only as evaluation benchmarks and are not part of a private release.}

\subsection{{SNP array of adults with schizophrenia}}
{We use the SNP array data collected by \cite{lowther2017impact} to illustrate sparse mean estimation with an individual subject as the privacy unit. The data contain 387 adults with schizophrenia, of whom 241 are labeled ``average IQ" and 146 ``low IQ." Genotypes were measured using the Affymetrix Genome-Wide Human SNP 6.0 platform. We retain the 2000 SNPs with the highest minor-allele frequencies (MAFs); the full data set is available at \url{https://www.ncbi.nlm.nih.gov/geo/query/acc.cgi?acc=GSE106818}.}

{Privacy-preserving analysis is particularly relevant for genetic data: \cite{homer2008resolving} showed that an adversary can use released summary statistics to infer whether an individual's genetic data contributed to a study. Because minor allele frequencies are coordinatewise means of an SNP array, a differentially private mean estimator provides a natural protection mechanism.}

{The resulting $387\times2000$ matrix has entries in $\{0,1,2\}$, recording the minor-allele count at each SNP. Coordinatewise group means therefore estimate twice the corresponding MAFs. We apply the sparse estimator to differences between the two IQ groups, using sparsity as a modeling approximation that only a small subset of the 2000 coordinate differences is appreciable.}
\begin{figure}[H]
	\centering
	\subfloat[][]{\includegraphics[width=.37\textwidth]{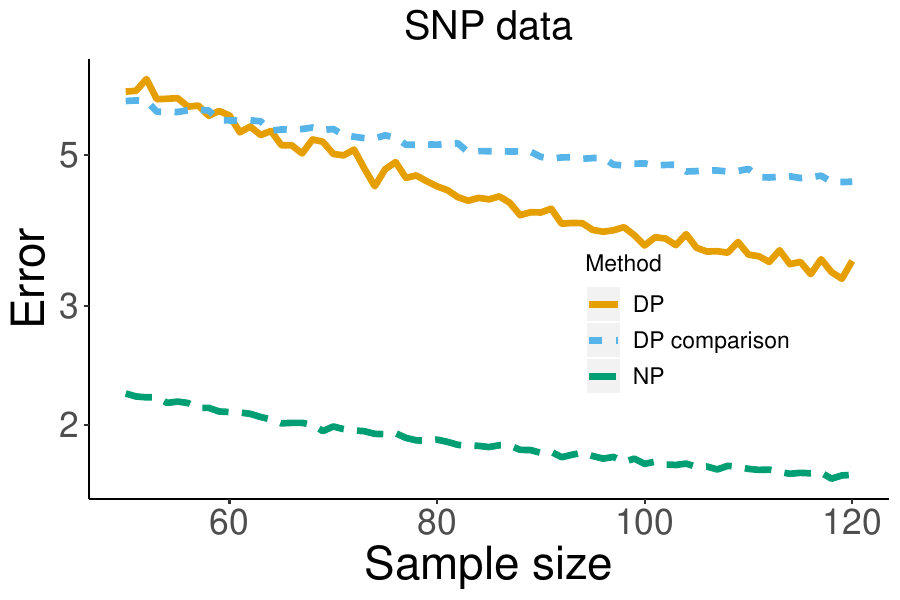}}\quad
	\subfloat[][]{\includegraphics[width=.37\textwidth]{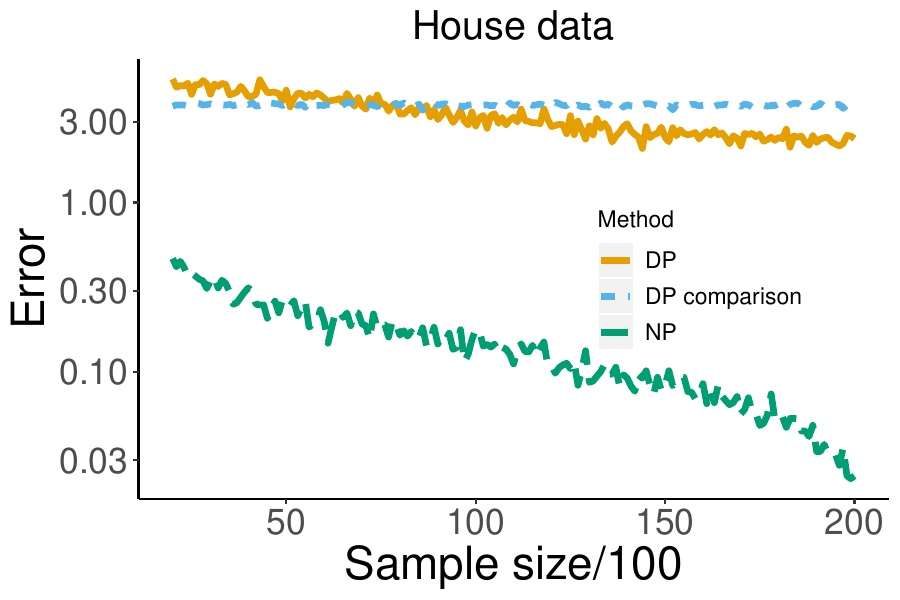}}\quad
	\caption{{(a) Estimated $\E\|\hat\bmu-\bmu\|_2$ for the private sparse-mean estimator and its benchmark as the per-group sample size increases from $50$ to $120$, with $s=20$. (b) Estimated $\E\|\hat\bbeta-\bbeta\|_2$ for the private regression estimators as sample size increases from $2000$ to $20000$; the full-sample non-private OLS fit is used only as the evaluation target.}}
	\label{fig: data}
\end{figure}
{For $m=50,\ldots,120$, we subsample $m$ subjects from each IQ group, randomly pair the two subsamples, and apply the sparse-mean estimator to the $m$ pairwise SNP differences with $s=20$ and $(\varepsilon,\delta)=(0.5,10m^{-1.1})$. We evaluate error against the difference of the two full-group sample means; this benchmark is used only for evaluation. Repeating the procedure 100 times gives Figure \ref{fig: data}(a). The corresponding curve for the method of \cite{talwar2015nearly} is included for comparison.}

\subsection{{California housing prices}}
{For regression, we use the California housing data of \cite{pace1997sparse}, available at \url{http://lib.stat.cmu.edu/datasets/houses.zip}. Each of the 20,640 rows is a 1990 Census block group. The response is median house value, and the covariates include median income, median age, total population, number of households, and total number of rooms. In this illustration, adjacency and the reported privacy parameters are interpreted at the block-group row level. Because each row already contains aggregated household statistics, the analysis should not be interpreted as providing household-level differential privacy without an additional sensitivity analysis at that finer unit.}

{For $m$ from $2000$ to $20000$, we subsample $m$ block groups and compute the private regression estimate with $(\varepsilon,\delta)=(0.5,10m^{-1.1})$. The design matrix is standardized before the algorithm is applied. We evaluate against the non-private OLS estimate computed from the full data set, again only as an offline benchmark. Repeating the procedure 100 times gives Figure \ref{fig: data}(b); the method of \cite{sheffet2017differentially} is included for comparison.}

\section{Discussion}\label{sec: discussions}

	{This paper characterizes the privacy--accuracy tradeoff for mean estimation and Gaussian linear regression. The low-dimensional mean lower bound is obtained by resampling a fingerprinting hard instance, the low-dimensional regression lower bound uses a tracing score and a smooth boundary-vanishing prior, and the sparse lower bounds use constant-weight packings with a private Fano lemma. The upper bounds are based on deterministic clipping, private support selection, and split-sample noisy gradients. In particular, Euclidean clipping in the low-dimensional regression algorithm and coordinatewise clipping in its sparse analogue remove the need for an almost-sure bounded-design assumption; under the stated sub-Gaussian conditions, the new clipping steps are inactive with high probability and cost only logarithmic factors in the rates.}

	{The sparse results also identify the limitation of the present approximate-DP argument. Up to universal constants and a fixed threshold, its effective privacy entropy is
	\begin{align*}
		\min\{s\log(ed/s),[\log((e^\varepsilon-1)/\delta)]_+\}.
	\end{align*}
	Thus, when $\delta$ is polynomially smaller than $\varepsilon$ and $s\log(ed/s)$ is polylogarithmic in $n$, the pure-DP dependence is retained up to a polylogarithmic factor, including when $\varepsilon=o(1)$. If $\delta$ is not sufficiently smaller than $\varepsilon$, the private-packing term becomes vacuous and the theorem correctly retains only the non-private lower bound. A uniformly matching approximate-DP lower bound beyond this effective-entropy regime remains open.}

	{The simulations and data examples illustrate finite-sample behavior but should be interpreted separately from the formal guarantees: data-dependent tuning consumes privacy budget, and the high-dimensional full-data implementation used numerically is a heuristic counterpart of the split-sample algorithm analyzed in the theorems. Within those limitations, the experiments show the expected decline of privacy error with sample size and the practical effect of clipping and sparsity tuning.}

	{Several extensions are natural. First, analogous minimax questions arise in covariance and precision-matrix estimation, graphical-model recovery, nonparametric regression, and principal-component analysis. Second, privacy-preserving uncertainty quantification remains much less developed than point estimation; \cite{karwa2017finite}, for example, established rate-optimal private confidence intervals for a one-dimensional Gaussian mean. Finally, it would be useful to compare the statistical cost of central differential privacy with concentrated differential privacy \cite{dwork2016concentrated}, R\'enyi differential privacy \cite{mironov2017renyi}, and Gaussian differential privacy \cite{dong2019gaussian}, especially in applications such as private stochastic-gradient methods and deep learning \cite{wang2015privacy,bu2019deep}.}

\section{Proofs}\label{sec: proofs}
{This section contains the two arguments most closely tied to the paper's main ideas.  Sections \ref{sec: proof of lm: low-dim mean group privacy}--\ref{sec: proof of thm: low-dim mean lower bound} prove the resampling--fingerprinting lower bound, including the ordinary Fano component.  The final two subsections state the one-step sparse-regression contraction and verify, event by event, the high-probability upper bound.  The remaining auxiliary proofs are collected in the supplement.}

\subsection{Auxiliary lemmas for the resampling reduction}
{
\begin{Lemma}[Random-resampling group privacy]\label{lm: random resampling privacy}
	Let $M$ be $(\varepsilon,\delta)$-differentially private on data sets of size $n$. Let $D=(\bm z_1,\ldots,\bm z_m)$ and let $\mathsf A(D)$ be obtained by drawing $J_1,\ldots,J_n$ independently from a fixed distribution on $[m]$ and applying $M$ to $(\bm z_{J_1},\ldots,\bm z_{J_n})$. If $D,D'$ differ only in row $j$, put $K=\sum_{i=1}^n\1\{J_i=j\}$. Then, for every integer $r\geq1$ and measurable set $S$,
	\begin{align}\label{eq: random resampling group privacy}
		\Pr\{\mathsf A(D)\in S\}
		\leq e^{r\varepsilon}\Pr\{\mathsf A(D')\in S\}
		+\delta\frac{e^{r\varepsilon}-1}{e^\varepsilon-1}
		+\Pr(K>r).
	\end{align}
\end{Lemma}
\begin{proof}
	Condition on the entire index vector $J=(J_1,\ldots,J_n)$. On $\{K\leq r\}$ the two induced input data sets differ in at most $r$ observations, so repeated application of differential privacy gives
	\begin{align*}
		\Pr\{M(D_J)\in S\mid J\}
		\leq e^{r\varepsilon}\Pr\{M(D'_J)\in S\mid J\}
		+\delta\sum_{q=0}^{r-1}e^{q\varepsilon}.
	\end{align*}
	Average this inequality over $J$ on $\{K\leq r\}$, bound the contribution of $\{K>r\}$ on the left by $\Pr(K>r)$, and enlarge the first term to the unconditional probability under $D'$. The geometric sum is $(e^{r\varepsilon}-1)/(e^\varepsilon-1)$, proving \eqref{eq: random resampling group privacy}.
\end{proof}

\begin{Lemma}[Binomial upper tail]\label{lm: binomial resampling tail}
	If $K\sim\operatorname{Binomial}(n,p)$ has mean $\mu=np$, then, for every $a\geq1$,
	\begin{align}\label{eq: binomial resampling tail}
		\Pr(K\geq a\mu)
		\leq\left(\frac{e^{a-1}}{a^a}\right)^\mu.
	\end{align}
\end{Lemma}
\begin{proof}
	For $\lambda>0$, the exponential Markov inequality and $1-p+pe^\lambda\leq\exp\{p(e^\lambda-1)\}$ give
	\begin{align*}
		\Pr(K\geq a\mu)
		\leq\exp\{\mu(e^\lambda-1)-\lambda a\mu\}.
	\end{align*}
	Taking $\lambda=\log a$ proves \eqref{eq: binomial resampling tail}; the case $a=1$ follows by continuity.
\end{proof}
}

\subsection{Proof of Lemma \ref{lm: low-dim mean group privacy}}\label{sec: proof of lm: low-dim mean group privacy}
\begin{proof}[Proof of Lemma \ref{lm: low-dim mean group privacy}]
	{We use the fingerprinting construction underlying the one-way-marginal lower bound of \cite{steinke2017between}. Let $D=(\bm z_1,\ldots,\bm z_m)\in\{-\sigma,\sigma\}^{m\times d}$ and let $Q_D$ be uniform on its rows. Given an estimator $M$ for $n_0=km$ observations, define a mechanism on $D$ by drawing $J_1,\ldots,J_{n_0}$ independently and uniformly from $[m]$ and returning
	\begin{align*}
		\mathsf A(D)=M(\bm z_{J_1},\ldots,\bm z_{J_{n_0}}).
	\end{align*}
	If $D,D'$ differ in row $j$, couple the executions using the same indices. The number $K=\sum_{i=1}^{n_0}\1\{J_i=j\}$ of affected observations is $\operatorname{Binomial}(n_0,1/m)$ and has mean $k$.} Lemma \ref{lm: random resampling privacy} therefore gives, for every integer $r\geq1$,
	\begin{align}\label{eq: resampling privacy}
		\Pr\{\mathsf A(D)\in S\}
		\leq e^{r\varepsilon}\Pr\{\mathsf A(D')\in S\}
		+\delta\frac{e^{r\varepsilon}-1}{e^\varepsilon-1}
		+\Pr(K>r).
	\end{align}
	In particular, the binomial-tail term is additive rather than multiplied by the number of rows.

	{Put $\ell=\log(1/\delta)$ and take
	\begin{align*}
		k=\left\lceil c\frac{\ell}{\varepsilon}\right\rceil,
		\qquad
		m=\left\lfloor c'\sqrt{\frac{d}{\ell}}\right\rfloor,
		\qquad n_0=km,
	\end{align*}
	where the constants may depend on $\omega$. Fix a sufficiently large integer $A$ and then a sufficiently small $c$ such that $Ac<\omega/\{2(1+\omega)\}$, and take $r=Ak$. Lemma \ref{lm: binomial resampling tail} gives
	\begin{align*}
		\Pr(K>Ak)\leq \left(\frac{e^{A-1}}{A^A}\right)^k.
	\end{align*}
	Moreover, $k\varepsilon\leq c\ell+1$ and
	$\varepsilon^{-1}\leq n_0/(cm\ell)$. Thus, using $e^\varepsilon-1\geq\varepsilon$,
	\begin{align*}
		\delta\frac{e^{Ak\varepsilon}-1}{e^\varepsilon-1}
		&\leq C\frac{\delta^{1-Ac}}{\varepsilon}
		\leq \frac{C}{m\ell}\,
		n_0^{1-(1+\omega)(1-Ac)}
		\leq \frac{1}{100m}.
	\end{align*}
	The second inequality uses $\delta\leq n_0^{-(1+\omega)}$; the last follows from the choice of $Ac$ after increasing, if necessary, the lower constant required of $d/\ell$ in the lemma. The same assumption gives $\log m\leq\log n_0\leq\ell/(1+\omega)$. Since $k\gtrsim\ell$ for $\varepsilon\leq1$, first choosing $A$ large enough and then $c$ small enough makes the binomial tail at most $(100m)^{-1}$. Thus $\mathsf A$ is $(\varepsilon_A,\delta_A)$-DP with
	\begin{align}\label{eq: resampled privacy parameters}
		\varepsilon_A=Ak\varepsilon\lesssim Ac\ell,
		\qquad \delta_A\leq(50m)^{-1}.
	\end{align}}

	{We now invoke the robust fingerprinting-code theorem stated precisely as Theorem \ref{thm: robust fingerprinting code} in the supplement. Let
	\begin{align*}
		\gamma=\frac{e^{-2\varepsilon_A}}{100m}.
	\end{align*}
	Because $\varepsilon_A\lesssim\ell$ and $\log m\lesssim\ell$,
	\begin{align*}
		\log(1/\gamma)=2\varepsilon_A+\log(100m)\leq C\ell.
	\end{align*}
	The required code length is therefore at most
	$Cm^2\log(1/\gamma)\leq C(c')^2d$; choosing $c'$ sufficiently small makes a code of length $d$ available.

	We verify the remaining black-box hypotheses explicitly. Draw a code database
	$C\in\{-1,1\}^{m\times d}$ and put $D=\sigma C$. Define
	\begin{align*}
		\mathsf B(C)=\sigma^{-1}
		\Pi_{[-\sigma,\sigma]^d}\{\mathsf A(\sigma C)\}.
	\end{align*}
	First, $\mathsf B(C)\in[-1,1]^d$, as required by the code theorem. Second, the true one-way-marginal vector is
	$\overline C=m^{-1}\sum_{i=1}^m C_i$, and coordinatewise projection cannot increase
	$\ell_1$ error from $\overline C$. Third, $\mathsf B$ is
	$(\varepsilon_A,\delta_A)$-DP under replace-one adjacency because scaling and projection are post-processing of $\mathsf A$.

	Suppose that, for every code database,
	$\E\|\mathsf A(\sigma C)-\sigma\overline C\|_1\leq c_0\sigma d$.
	Choose $c_0\leq1/800$. Markov's inequality gives
	\begin{align*}
		\Pr\{\|\mathsf B(C)-\overline C\|_1>d/8\}
		\leq8c_0\leq1/100.
	\end{align*}
	The $1/100$-completeness guarantee then yields
	\begin{align*}
		\Pr\{\mathsf{Trace}(\mathsf B(C))\neq\emptyset\}
		\geq1-\frac1{100}-\frac1{100}=\frac{49}{50}.
	\end{align*}
	Since a nonempty trace contains at least one user, averaging gives an index $i\in[m]$ such that
	\begin{align*}
		\Pr\{i\in\mathsf{Trace}(\mathsf B(C))\}\geq\frac{49}{50m}.
	\end{align*}
	Let $C_{-i}$ replace row $i$ by the fixed public row specified in the code theorem. The $\gamma$-soundness guarantee gives
	$\Pr\{i\in\mathsf{Trace}(\mathsf B(C_{-i}))\}\leq\gamma$.
	Finally, $C$ and $C_{-i}$ are adjacent, so differential privacy and \eqref{eq: resampled privacy parameters} imply
	\begin{align*}
		\frac{49}{50m}
		&\leq e^{\varepsilon_A}\gamma+\delta_A\leq\frac{1}{100m}+\frac{1}{50m}
		=\frac{3}{100m},
	\end{align*}
	a contradiction.}

	Projection cannot increase the error because the true marginal vector lies in $[-\sigma,\sigma]^d$. Since $\|u\|_1\leq\sqrt d\|u\|_2$, there is therefore a database $D$ for which
	\begin{align*}
		{\E_{Q_D^{n_0}}\|M-\E_{Q_D}\bm y\|_2\gtrsim\sigma\sqrt d.}
	\end{align*}
	The distribution $Q_D$ is coordinatewise sub-Gaussian because each coordinate lies in $[-\sigma,\sigma]$. This proves the lemma.
\end{proof}

\subsection{Proof of Theorem \ref{thm: low-dim mean lower bound}}\label{sec: proof of thm: low-dim mean lower bound}
\begin{proof}[Proof of Theorem \ref{thm: low-dim mean lower bound}]
	{We prove the sampling and privacy contributions separately.

	For the sampling contribution, let $\mathcal V\subseteq\{-1,1\}^d$ be a Varshamov--Gilbert code satisfying
	\begin{align*}
		\log|\mathcal V|\geq c_1d,
		\qquad d_{\mathrm H}(v,v')\geq c_2d
		\quad(v\neq v').
	\end{align*}
	(The case of bounded $d$ follows from a two-point argument.)  For $v\in\mathcal V$, put
	$P_v=N_d(a v,\sigma^2\bm I_d)$, where $a=c_3\sigma/\sqrt n$ and $c_3>0$ is sufficiently small.  These distributions belong to $\mathcal P(\sigma,d,\Theta)$ after decreasing $c_3$ if necessary.  Moreover,
	\begin{align*}
		\|av-av'\|_2^2
		&=4a^2d_{\mathrm H}(v,v')\geq c a^2d,\\
		D_{\mathrm{KL}}(P_v^{\otimes n}\|P_{v'}^{\otimes n})
		&=\frac{n}{2\sigma^2}\|av-av'\|_2^2
		\leq Cc_3^2d.
	\end{align*}
	Fano's inequality, with $c_3$ chosen so that the last display is at most a small multiple of $\log|\mathcal V|$, therefore gives
	\begin{align}\label{eq: nonprivate low mean lower}
		\inf_M\sup_{P\in\mathcal P(\sigma,d,\Theta)}
		\E_P\|M-\bmu(P)\|_2^2\geq c\frac{\sigma^2d}{n}.
	\end{align}

	We next prove the privacy contribution.  Let
	\begin{align*}
		n_0=km\asymp\frac{\sqrt{d\log(1/\delta)}}{\varepsilon}
	\end{align*}
	be the integer sample size in Lemma \ref{lm: low-dim mean group privacy}, and put $\alpha=n_0/n$.  By increasing $C_\omega$ in the theorem, we may assume $\alpha\leq c$.  Since $n_0\leq n$ and $\delta\leq n^{-(1+\omega)}$, we also have $\delta\leq n_0^{-(1+\omega)}$, so the lemma's privacy calculation is available.

	Fix a database $D\in\{-\sigma/4,\sigma/4\}^{m\times d}$.  Let $\bm y$ be uniform on its rows, draw $B\sim\operatorname{Bernoulli}(\alpha)$ independently of $\bm y$, and set $\bm x=B\bm y$.  Each centered coordinate of $\bm x$ is bounded in an interval of length at most $\sigma/2$, and hence is sub-Gaussian with parameter at most $\sigma$.  Also,
	$\E\bm x=\alpha\E\bm y\in\Theta$.

	Suppose, toward a contradiction, that an $(\varepsilon,\delta)$-DP estimator $M$ has squared $\ell_2$ risk less than $c_4\alpha^2\sigma^2d$ for every distribution in this family.  Define a randomized mechanism on $D$ by drawing $n$ independent thinned observations and returning
	\begin{align*}
		\mathsf A_\alpha(D)=\alpha^{-1}M(\bm x_1,\ldots,\bm x_n).
	\end{align*}
	Couple the mechanisms on neighboring databases by using the same Bernoulli variables and row indices.  If row $j$ is changed, the number of changed inputs to $M$ is
	\begin{align*}
		K\sim\operatorname{Binomial}\left(n,\frac{\alpha}{m}\right),
		\qquad \E K=\frac{n\alpha}{m}=\frac{n_0}{m}=k.
	\end{align*}
	Taking $r=Ak$ in Lemmas \ref{lm: random resampling privacy} and \ref{lm: binomial resampling tail} gives exactly the bounds established in \eqref{eq: resampled privacy parameters}:
	\begin{align*}
		\varepsilon_A=Ak\varepsilon,
		\qquad
		\delta_A\leq
		\delta\frac{e^{Ak\varepsilon}-1}{e^\varepsilon-1}
		+\Pr(K>Ak)\leq\frac1{50m}.
	\end{align*}
	The proof of Lemma \ref{lm: low-dim mean group privacy}, applied after rescaling $\{-\sigma/4,\sigma/4\}$ to $\{-\sigma,\sigma\}$, therefore rules out expected $\ell_1$ error smaller than $c_5\sigma d$ for $\mathsf A_\alpha$ uniformly over $D$.  But Cauchy--Schwarz and the assumed risk bound give, for every $D$,
	\begin{align*}
		\E\left\|\mathsf A_\alpha(D)-\E\bm y\right\|_1
		&\leq \frac{\sqrt d}{\alpha}
		\left(\E\|M-\E\bm x\|_2^2\right)^{1/2}
		<\sqrt{c_4}\,\sigma d.
	\end{align*}
	Choosing $c_4<c_5^2$ yields the desired contradiction.  Consequently, for some distribution in the family,
	\begin{align}\label{eq: private low mean lower}
		\E\|M-\E\bm x\|_2^2
		&\geq c\alpha^2\sigma^2d
		\asymp\frac{\sigma^2d^2\log(1/\delta)}{n^2\varepsilon^2}.
	\end{align}
	Finally, the maximum of the lower bounds in \eqref{eq: nonprivate low mean lower} and \eqref{eq: private low mean lower} is at least one half of their sum.  This proves \eqref{eq: low-dim mean lower bound}.}
\end{proof}

\subsection{A one-step bound for split-sample private IHT}
{
\begin{Lemma}\label{lm: high-dim regression contraction}
	Let $I_t$ be a block independent of $\bbeta^t$. Suppose that, on every coordinate subspace of dimension at most $2s+s^*$, the empirical covariance on $I_t$ has eigenvalues in $[\underline\lambda_L,\overline\lambda_L]$, where the endpoints depend only on $L$. Choose $\eta^0$ so that
	\begin{align*}
		q_L=\max\{|1-\eta^0\underline\lambda_L|,
		|1-\eta^0\overline\lambda_L|\}<1,
	\end{align*}
		and let $s\geq C_Ls^*$ for a sufficiently large constant $C_L$. {On the event that both design clipping and residual clipping are inactive on $I_t$,} the update in Algorithm \ref{algo: high-dim regression} satisfies
	\begin{align}\label{eq: high-dim regression contraction}
		\|\bbeta^{t+1}-\bbeta\|_2^2
		&\leq \rho_0\|\bbeta^t-\bbeta\|_2^2
		+Cs\left\|\frac1b\sum_{i\in I_t}\xi_i\bm x_i\right\|_\infty^2 \notag\\
		&\quad+C\left[
		\sum_{j=1}^s\|\bm w_{t,j}\|_\infty^2
		+\|\widetilde{\bm w}_{t,S^{t+1}}\|_2^2
		\right],
	\end{align}
	where $0<\rho_0<1$, $\xi_i=y_i-\bm x_i^\top\bbeta$, and the noise vectors are those used by Peeling at iteration $t$.
\end{Lemma}
The lemma is proved in Section \ref{sec: proof of lm: high-dim regression contraction} of the supplement. It is a noisy version of the standard iterative-hard-thresholding contraction; the proof uses the Peeling accuracy lemma rather than exact hard thresholding.
}

\subsection{Proof of Theorem \ref{thm: high-dim regression upper bound}}\label{sec: proof of thm: high-dim regression upper bound}
\begin{proof}[Proof of Theorem \ref{thm: high-dim regression upper bound}]
	{We verify five events used in the one-step lemma and then solve its recursion. Let $\mathcal F_t$ be the sigma-field generated by the blocks preceding $I_t$ and by all algorithmic randomness used before iteration $t$. Because the blocks are disjoint, $\bbeta^t$ is $\mathcal F_t$-measurable and $I_t$ is independent of $\mathcal F_t$. Projection onto the radius-$C$ ball gives $\|\bbeta^t\|_2\leq C$, while (P') gives $\|\bbeta\|_2\leq C$.

	First consider the design-clipping event
	\begin{align*}
		\mathcal E_{\bm x,\infty}
		=\left\{\max_{1\leq i\leq n}\max_{1\leq j\leq d}
		|x_{ij}|\leq B_\infty\right\}.
	\end{align*}
	For each coordinate, assumption (D') and the choice $u=e_j$ imply
	$\|x_{ij}\|_{\psi_2}\leq K_{\bm x}$. Therefore
	\begin{align*}
		\Pr(\mathcal E_{\bm x,\infty}^c)
		&\leq\sum_{i=1}^n\sum_{j=1}^d
		\Pr(|x_{ij}|>B_\infty)\\
		&\leq2nd\exp\left(-\frac{cB_\infty^2}{K_{\bm x}^2}\right)
		\leq\frac{\eta}{5},
	\end{align*}
	where the last inequality follows from
	$B_\infty=C_4K_{\bm x}\sqrt{\log(10nd/\eta)}$ after increasing $C_4$ if needed.
	On $\mathcal E_{\bm x,\infty}$,
	$\overline{\bm x}_i=\bm x_i$ for all $i$.

	Next consider the \emph{raw} residuals. Conditional on $\mathcal F_t$, each residual satisfies
	\begin{align*}
		\bm x_i^\top\bbeta^t-y_i
		=\bm x_i^\top(\bbeta^t-\bbeta)-\xi_i,
		\qquad i\in I_t,
	\end{align*}
	and is sub-Gaussian with parameter at most $C'(\sigma+K_{\bm x}C)$.  Hence
	\begin{align*}
		&\Pr\left(
		\max_{t<T}\max_{i\in I_t}
		|\bm x_i^\top\bbeta^t-y_i|>R
		\right)\\
		&\quad\leq
		\sum_{t=0}^{T-1}\E\left[
		\sum_{i\in I_t}
		\Pr\left(
		|\bm x_i^\top\bbeta^t-y_i|>R
		\,\middle|\,\mathcal F_t\right)\right]\\
		&\quad\leq 2bT\exp\left\{-\frac{cR^2}{(\sigma+K_{\bm x}C)^2}\right\}
		\leq\frac\eta5
	\end{align*}
	for the constant $C_5$ in the theorem. The conditional tail bound is legitimate even outside $\mathcal E_{\bm x,\infty}$, because $\bbeta^t$ is bounded by projection and is independent of the current raw block. Thus, on the intersection of the first two events, both design clipping and residual clipping are inactive on every used observation.}

	{We next establish the restricted covariance event rather than invoke it as a black box.  For a block $I_t$ and a coordinate set $U$ with $k=|U|\leq2s+s^*\leq3s$, write
	\begin{align*}
		\widehat\Sigma_{t,U}=\frac1b\sum_{i\in I_t}
		\bm x_{i,U}\bm x_{i,U}^\top.
	\end{align*}
	For every fixed unit vector $u\in\mathbb R^k$, assumption (D') implies that
	$(u^\top\bm x_{i,U})^2-\E(u^\top\bm x_{i,U})^2$ is sub-exponential with norm at most $CK_{\bm x}^2$.  Bernstein's inequality, followed by a union bound over a $1/4$-net of the unit sphere having at most $9^k$ points, yields
	\begin{align*}
		\Pr\left(
		\|\widehat\Sigma_{t,U}-\Sigma_U\|_{\mathrm{op}}>\frac1{2L}
		\right)
		\leq2\exp(-c_{L,K_{\bm x}}b+Ck).
	\end{align*}
	Here we used the deterministic net inequality
	$\|A\|_{\mathrm{op}}\leq2\max_{u\in\mathcal N}|u^\top A u|$
	for symmetric $A$.  The number of coordinate sets of size at most $3s$ is at most
	\begin{align*}
		\sum_{k=1}^{3s}\binom dk\leq
		\exp\{Cs\log(ed/s)\},
	\end{align*}
	with a harmless adjustment when $3s>d$.  A further union bound over $t<T$, together with the assumed lower bound on $b$, shows that, with failure probability at most $\eta/5$,
	\begin{align}\label{eq: split restricted covariance}
		\lambda_{\min}(\widehat\Sigma_{t,U})\geq\frac1{2L},
		\qquad
		\lambda_{\max}(\widehat\Sigma_{t,U})\leq2L
	\end{align}
	for every $t<T$ and every such $U$.  Thus Lemma \ref{lm: high-dim regression contraction} applies with
	$\underline\lambda_L=(2L)^{-1}$ and $\overline\lambda_L=2L$.  Taking $C_2$ sufficiently small and $C_1$ sufficiently large, as functions only of $L$, makes its contraction factor $\rho_0$ strictly smaller than one.

	For the stochastic score, independence of $\xi_i$ and $\bm x_i$ implies that each product $\xi_i x_{ij}$ is centered and sub-exponential with norm at most $CK_{\bm x}\sigma$.  Bernstein's inequality therefore gives, for every $u>0$,
	\begin{align*}
		\Pr\left(
		\left|\frac1b\sum_{i\in I_t}\xi_i x_{ij}\right|>u
		\right)
		\leq2\exp\left\{-cb\min\left(
		\frac{u^2}{K_{\bm x}^2\sigma^2},
		\frac{u}{K_{\bm x}\sigma}
		\right)\right\}.
	\end{align*}
	The lower bound on $b$ allows us to take
	$u=C_{K_{\bm x}}\sigma\sqrt{\log(10dT/\eta)/b}$.
	A union bound over $j\in[d]$ and $t<T$ then gives
	\begin{align}\label{eq: split score bound}
		\max_{t<T}\left\|\frac1b\sum_{i\in I_t}
		\xi_i\bm x_i\right\|_\infty^2
		\leq C_{K_{\bm x}}\frac{\sigma^2\log(10dT/\eta)}{b}
	\end{align}
	with failure probability at most $\eta/5$.

	It remains to control the algorithmic noise.  Every scalar Laplace variable used by Peeling at iteration $t$ has scale
	\begin{align*}
		{ \lambda_t
		=\frac{16\eta^0RB_\infty}{b}
		\frac{\sqrt{2s\log(1/\delta)}}{\varepsilon}.}
	\end{align*}
	There are at most $T(s+1)d$ such scalar variables.  Since
	$\Pr(|W|>\lambda_t z)=e^{-z}$, a union bound with
	$z=\log\{5T(s+1)d/\eta\}$ shows that, except on an event of probability $\eta/5$, all of them have absolute value at most $\lambda_tz$.  Therefore, simultaneously for all $t<T$,
	\begin{align}\label{eq: split peeling noise}
		\sum_{j=1}^s\|\bm w_{t,j}\|_\infty^2
		+\|\widetilde{\bm w}_{t,S^{t+1}}\|_2^2
		&\leq2s\lambda_t^2z^2\\
		&\leq
		C_L\frac{B_\infty^2R^2s^2
		\log^2(10dT/\eta)\log(1/\delta)}
		{b^2\varepsilon^2}.
	\end{align}

	Intersecting the five events above and applying Lemma \ref{lm: high-dim regression contraction}, we obtain, with probability at least $1-\eta$,
	\begin{align*}
		e_{t+1}^2\leq\rho_0e_t^2+B,
		\qquad e_t=\|\bbeta^t-\bbeta\|_2,
	\end{align*}
	where
	\begin{align*}
		B=C_{L,K_{\bm x}}\left\{
		\frac{\sigma^2s\log(10dT/\eta)}b
		+\frac{B_\infty^2R^2s^2\log^2(10dT/\eta)
		\log(1/\delta)}{b^2\varepsilon^2}
		\right\}.
	\end{align*}
	Iteration gives
	\begin{align*}
		e_T^2
		&\leq\rho_0^Te_0^2+B\sum_{r=0}^{T-1}\rho_0^r
		\leq\rho_0^Tc_0^2+\frac{B}{1-\rho_0}.
	\end{align*}
	Because $1-\rho_0$ depends only on $L$ and $T=C_3L^2\log(en)$, increasing $C_3$ makes the first term at most $n^{-1}$.  Absorbing the fixed factors and the constants inside the logarithms proves \eqref{eq: high-dim regression upper bound}.}
\end{proof}

\section{{Acknowledgments}}
{We thank the Associate Editor and referees for suggestions that substantially improved the paper, and Chi-Yun Wu for advice on Section \ref{sec: data analysis}. T. Cai was supported in part by NSF grants DMS-1712735 and DMS-2015259 and NIH grants R01-GM129781 and R01-GM123056. L. Zhang was supported in part by NSF grant DMS-2015378.}

\clearpage
\section*{Supplementary Material}

{This supplement contains the detailed arguments supporting the results in the main paper. Section \ref{sec: ub proofs} proves the upper bounds in the order in which they are used: clipping and Gaussian perturbation for low-dimensional means, Peeling accuracy and sparse-mean estimation, and then the privacy and contraction arguments for linear regression. Section \ref{sec: lb proofs} first establishes a private packing lemma that remains valid when $\varepsilon=o(1)$, then applies it to sparse mean and sparse regression, and finally gives the tracing argument for low-dimensional regression. The resampling--fingerprinting proof of the low-dimensional mean lower bound is kept in the main paper because it is central to the lower-bound strategy.} {For completeness, the supplement also states the exact Peeling-privacy and robust-fingerprinting-code results used as black boxes, including all sensitivity, noise-scale, completeness, soundness, and robustness hypotheses.}

\appendix
\section{ {Proofs of the upper bounds}}\label{sec: ub proofs}

{  This section follows the progression from estimation error to privacy error.  For each estimator, we first isolate the deterministic approximation or contraction step, then control the sampling fluctuation, and finally bound the added privacy noise.  All probability statements specify the events being intersected, so the displayed theorem-level confidence follows by a final union bound.}

\subsection{Proof of Theorem \ref{thm: low-dim mean upper bound}}\label{sec: proof of thm: low-dim mean upper bound}
{
\begin{proof}[Proof of Theorem \ref{thm: low-dim mean upper bound}]
	Write $Z_j=\Pi_R(x_j)$ and $b_j=\E Z_j-\mu_j$. We first control the bias introduced by clipping. If $R\geq |\mu_j|+t$, then
	\begin{align*}
		|Z_j-x_j|
		&\leq (|x_j|-R)_+
		\leq (|x_j-\mu_j|-t)_+.
	\end{align*}
	The assumed moment-generating-function bound implies
	\begin{align*}
		\Pr(|x_j-\mu_j|>u)\leq 2\exp(-u^2/(4\sigma^2)),
		\qquad u>0.
	\end{align*}
	Indeed, for the upper tail, exponential Markov inequality followed by the choice $\lambda=u/(2\sigma^2)$ gives
	$\Pr(x_j-\mu_j>u)\leq\exp(-u^2/(4\sigma^2))$.  Applying the same argument to $-(x_j-\mu_j)$, whose moment-generating function obeys the same bound, gives
	$\Pr(x_j-\mu_j<-u)\leq\exp(-u^2/(4\sigma^2))$.  The union bound proves the displayed two-sided tail. By the layer-cake formula,
	\begin{align*}
		|b_j|
		&=|\E(Z_j-x_j)|
		\leq\int_t^\infty\Pr(|x_j-\mu_j|>u)\,du\\
		&\leq C\sigma\exp(-ct^2/\sigma^2).
	\end{align*}
	Consequently, whenever $R\geq\|\bmu\|_\infty+t$,
	\begin{align}\label{eq: clipping bias}
		\max_j|b_j|\leq C\sigma\exp(-ct^2/\sigma^2).
	\end{align}

	We next bound the variance. Let $x_j'$ be an independent copy of $x_j$ and put $Z_j'=\Pi_R(x_j')$. Since $\Pi_R$ is one-Lipschitz,
	\begin{align*}
		\operatorname{Var}(Z_j)
		=\frac12\E(Z_j-Z_j')^2
		\leq\frac12\E(x_j-x_j')^2
		=\operatorname{Var}(x_j)\leq C\sigma^2.
	\end{align*}
	The last inequality follows, for example, by integrating the preceding sub-Gaussian tail. Independence across observations, but not across coordinates, now gives
	\begin{align*}
		\E\|\overline{\bm X}_R-\bmu\|_2^2
		&=\sum_{j=1}^d\frac{\operatorname{Var}(Z_j)}n+\|\bm b\|_2^2\\
		&\leq \frac{C\sigma^2d}{n}
		+C\sigma^2d\exp(-2ct^2/\sigma^2)
		\lesssim \frac{\sigma^2d}{n},
	\end{align*}
	where the last step uses $t=K\sigma\sqrt{\log(en)}$ and a sufficiently large universal $K$.

	It remains to account for privacy noise. If two data sets differ in one observation, then every coordinate of their clipped sample means differs by at most $2R/n$. Hence the $\ell_2$ sensitivity is at most
	\begin{align*}
		\Delta_2=\frac{2R\sqrt d}{n}.
	\end{align*}
	For $0<\varepsilon\leq1$ and $0<\delta<1/2$, the Gaussian mechanism is $(\varepsilon,\delta)$-DP when the noise variance is at least a universal constant times
	$\Delta_2^2\log(1/\delta)/\varepsilon^2$. The variance in Algorithm \ref{algo: low-dim mean estimation} has this size. Since its $d$ coordinates are independent and centered,
	\begin{align*}
		\E\|\bm w\|_2^2
		=d\,\operatorname{Var}(w_1)
		\lesssim\frac{R^2d^2\log(1/\delta)}{n^2\varepsilon^2}.
	\end{align*}
	The statistical error and $\bm w$ are independent and $\E\bm w=\bm0$, so their squared risks add. This proves the first assertion. If $\|\bm x\|_\infty\leq R_0$ almost surely, then $\Pi_{R_0}(\bm x)=\bm x$ almost surely. Thus $\bm b=\bm0$, and the same sensitivity and variance calculation applies with $R=R_0$.
\end{proof}
}

\subsection{A black-box privacy theorem for Peeling}\label{sec: peeling privacy theorem}
{
\begin{Lemma}[Peeling privacy; \citealp{dwork2018differentially}]\label{lm: peeling privacy}
	Use replace-one adjacency on data sets in $\mathcal X^n$. Let
	$\bm v:\mathcal X^n\to\R^d$ satisfy
	\begin{align*}
		\sup_{\bm Z\sim\bm Z'}
		\|\bm v(\bm Z)-\bm v(\bm Z')\|_\infty\leq\lambda.
	\end{align*}
	Fix $1\leq s\leq d$, $0<\varepsilon\leq1$, and
	$0<\delta<1/2$. At each of $s$ adaptive rounds, select an as-yet-unselected coordinate by adding independent
	$\operatorname{Laplace}(b_{\rm peel})$ noise to the scores
	$\{|v_j|\}$, and release the signed selected coordinate with an independent
	$\operatorname{Laplace}(b_{\rm peel})$ noise, where
	\begin{align}\label{eq: exact peeling scale}
		b_{\rm peel}
		=\frac{8\lambda\sqrt{2s\log(1/\delta)}}{\varepsilon}.
	\end{align}
	Equivalently, the $s$ noisy signed values may be released together after all indices are selected, as in Algorithm \ref{algo: peeling}. The resulting sparse vector is $(\varepsilon,\delta)$-differentially private.
\end{Lemma}
\begin{proof}
	This is the Peeling/Private-Min result of \cite{dwork2018differentially} with its constants written in a form valid over the parameter range used here. Put
	\begin{align*}
		\varepsilon_0
		=\frac{\varepsilon}{2\sqrt{2s\log(1/\delta)}}.
	\end{align*}
	In one round, the absolute-value scores and the signed coordinates each have sensitivity at most $\lambda$. With the common scale
	$4\lambda/\varepsilon_0=b_{\rm peel}$, Report Noisy Max is
	$(\varepsilon_0/2,0)$-DP for arbitrary two-sided score changes, while the fresh signed-value release is $(\varepsilon_0/4,0)$-DP. By sequential composition, the joint index--value release is therefore $(3\varepsilon_0/4,0)$-DP, and hence $(\varepsilon_0,0)$-DP. The advanced-composition theorem for the $s$ adaptive rounds gives
	\begin{align*}
		\varepsilon_{\rm comp}
		&=\sqrt{2s\log(1/\delta)}\,\varepsilon_0
		+s\varepsilon_0(e^{\varepsilon_0}-1),
		\qquad \delta_{\rm comp}=\delta.
	\end{align*}
	Here $\varepsilon_0<1$, and $e^u-1\leq2u$ for $0\leq u\leq1$. Since $\log(1/\delta)\geq\log2$ and $\varepsilon\leq1$,
	\begin{align*}
		\varepsilon_{\rm comp}
		&\leq\frac\varepsilon2
		+2s\varepsilon_0^2
		=\frac\varepsilon2+
		\frac{\varepsilon^2}{4\log(1/\delta)}
		<\varepsilon.
	\end{align*}
	Post-processing the released index--value pairs into a sparse vector preserves privacy.
\end{proof}
}

\subsection{Accuracy properties of Peeling}\label{sec: proof of lm: peeling accuracy}
{
\begin{proof}[Proof of Lemma \ref{lm: peeling accuracy}]
	Choose a bijection $\psi:R_2\to R_1$. For $j\in R_2$, let $i(j)$ be the iteration at which $\psi(j)$ is selected. Because $R_2\subseteq S^c$, the coordinate $j$ is never selected and is therefore still available at iteration $i(j)$. The selection rule gives
	\begin{align*}
		|v_j|+w_{i(j),j}
		\leq |v_{\psi(j)}|+w_{i(j),\psi(j)}.
	\end{align*}
	Consequently,
	\begin{align*}
		|v_j|
		&\leq |v_{\psi(j)}|
		+|w_{i(j),\psi(j)}|+|w_{i(j),j}|\\
		&\leq |v_{\psi(j)}|+2\|\bm w_{i(j)}\|_\infty.
	\end{align*}
	Using $(u+v)^2\leq(1+c)u^2+(1+1/c)v^2$, for every $c>0$ we obtain
	\begin{align*}
		v_j^2
		&\leq(1+c)v_{\psi(j)}^2
		+4(1+1/c)\|\bm w_{i(j)}\|_\infty^2.
	\end{align*}
	The map $j\mapsto i(j)$ is injective: different elements of $R_1$ are selected at different iterations. Hence
	\begin{align*}
		\sum_{j\in R_2}\|\bm w_{i(j)}\|_\infty^2
		\leq\sum_{i=1}^s\|\bm w_i\|_\infty^2.
	\end{align*}
	Finally, summing the preceding coordinatewise bound over $j\in R_2$ and using that $\psi$ is a bijection proves the claim.
\end{proof}
}

\begin{Lemma}\label{lm: laplace noise bound}
	Let $w_1,\ldots,w_k$ be independent $\operatorname{Laplace}(\lambda)$ random variables. For every $0<\eta<1/2$, with probability at least $1-\eta$,
	\begin{align*}
		\|\bm w\|_\infty\leq\lambda\log(k/\eta),
		\qquad
		\|\bm w\|_2^2\leq k\lambda^2\log^2(k/\eta).
	\end{align*}
\end{Lemma}
\begin{proof}
	{  For $u=\lambda\log(k/\eta)$, the exact Laplace tail and the union bound give
	\begin{align*}
		\Pr(\|\bm w\|_\infty>u)
		\leq\sum_{j=1}^k\Pr(|w_j|>u)
		=k e^{-u/\lambda}=\eta.
	\end{align*}
	On the complementary event, $\|\bm w\|_2^2\leq k\|\bm w\|_\infty^2\leq k\lambda^2\log^2(k/\eta)$, proving both assertions.}
\end{proof}

\subsection{Proof of Theorem \ref{thm: high-dim mean upper bound}}\label{sec: proof of thm: high-dim mean upper bound}
{
\begin{proof}[Proof of Theorem \ref{thm: high-dim mean upper bound}]
	{The $\ell_\infty$ sensitivity of $\overline{\bm X}_R$ is at most $2R/n$. Substitution into \eqref{eq: exact peeling scale} gives
	\begin{align*}
		b_{\rm mean}=\frac{16R\sqrt{2s\log(1/\delta)}}{n\varepsilon},
	\end{align*}
	which is exactly the scalar Laplace scale used by Algorithm \ref{algo: high-dim mean estimation}; Lemma \ref{lm: peeling privacy} therefore proves that the algorithm is $(\varepsilon,\delta)$-DP.} We now prove its accuracy.

	Let $\bm z=\overline{\bm X}_R$, $\bm e=\bm z-\bmu$, $S=\supp(\hat\bmu)$, and $S^*=\supp(\bmu)$. To control $\bm e$, put $Y_{ij}=\Pi_R(x_{ij})$. If $Y_{ij}'$ is an independent copy, symmetrization and the contraction property of $\Pi_R$ give
	\begin{align*}
		\|Y_{ij}-\E Y_{ij}\|_{\psi_2}
		&\leq C\|Y_{ij}-Y_{ij}'\|_{\psi_2}\\
		&\leq C\|x_{ij}-x_{ij}'\|_{\psi_2}
		\leq C\sigma.
	\end{align*}
	Therefore, for every $u>0$,
	\begin{align*}
		\Pr\left(\left|\frac1n\sum_{i=1}^n
		\{Y_{ij}-\E Y_{ij}\}\right|>u\right)
		\leq2\exp(-cnu^2/\sigma^2).
	\end{align*}
	Take $u=C\sigma\sqrt{\log(6d/\eta)/n}$ and apply a union bound over $j\in[d]$. With probability at least $1-\eta/3$,
	\begin{align*}
		\max_{j\leq d}\left|\frac1n\sum_{i=1}^n
		\{Y_{ij}-\E Y_{ij}\}\right|
		\leq C\sigma\sqrt{\frac{\log(d/\eta)}n}.
	\end{align*}
	The bias calculation \eqref{eq: clipping bias}, with
	$R-\|\bmu\|_\infty\geq K\sigma\sqrt{\log(en)}$, gives
	$\max_j|\E Y_{ij}-\mu_j|\leq C\sigma/\sqrt n$ after increasing $K$. Combining the stochastic and bias parts yields
	\begin{align}\label{eq: truncated mean infinity}
		\|\bm e\|_\infty^2\lesssim\frac{\sigma^2\log(d/\eta)}n
	\end{align}
	with probability at least $1-\eta/3$.

	We next compare the selected and missed coordinates. Let $R_2=S^*\setminus S$. Since $|S|=s\geq|S^*|$, we have
	\begin{align*}
		|S\setminus S^*|=s-|S\cap S^*|
		\geq |S^*|-|S\cap S^*|=|R_2|.
	\end{align*}
	Thus one can choose $R_1\subseteq S\setminus S^*$ with $|R_1|=|R_2|$. Lemma \ref{lm: peeling accuracy}, with $c=1$, gives
	\begin{align*}
		\|\bm z_{R_2}\|_2^2
		\leq2\|\bm z_{R_1}\|_2^2+8\sum_{i=1}^s\|\bm w_i\|_\infty^2.
	\end{align*}
	On $R_1$ the true mean is zero, so $\bm z_{R_1}=\bm e_{R_1}$. On the missed set $R_2$,
	$\bmu_{R_2}=\bm z_{R_2}-\bm e_{R_2}$, and hence
	\begin{align*}
		\|\bmu_{R_2}\|_2^2
		&\leq2\|\bm z_{R_2}\|_2^2+2\|\bm e_{R_2}\|_2^2\\
		&\leq4\|\bm e_{R_1}\|_2^2+2\|\bm e_{R_2}\|_2^2
		+16\sum_{i=1}^s\|\bm w_i\|_\infty^2.
	\end{align*}
	The released vector equals $\bm z_S+\widetilde{\bm w}_S$ on $S$ and is zero on $S^c$. Therefore
	\begin{align*}
		\|\hat\bmu-\bmu\|_2^2
		&=\|\bm e_S+\widetilde{\bm w}_S\|_2^2
		+\|\bmu_{R_2}\|_2^2\\
		&\leq2\|\bm e_S\|_2^2+2\|\widetilde{\bm w}_S\|_2^2
		+\|\bmu_{R_2}\|_2^2.
	\end{align*}
	Since all sets in the last two displays have cardinality at most $s+s^*\leq2s$, we obtain
	\begin{align}\label{eq: corrected sparse mean expansion}
		\|\hat\bmu-\bmu\|_2^2
		\leq C s\|\bm e\|_\infty^2
		+C\sum_{i=1}^s\|\bm w_i\|_\infty^2
		+C\|\widetilde{\bm w}_S\|_2^2.
	\end{align}
	{The Laplace scale in Algorithm \ref{algo: peeling} equals
	\begin{align*}
		\lambda=\frac{16R\sqrt{2s\log(1/\delta)}}{n\varepsilon}.
	\end{align*}
}
	There are $sd$ scalar selection noises and $d$ scalar release noises. Applying the exact Laplace tail and a union bound shows that, with probability at least $1-2\eta/3$,
	\begin{align*}
		\max_{i\leq s}\|\bm w_i\|_\infty
		&\leq\lambda\log(6sd/\eta),\\
		\|\widetilde{\bm w}_S\|_2^2
		&\leq s\lambda^2\log^2(6d/\eta).
	\end{align*}
	Consequently,
	\begin{align*}
		\sum_{i=1}^s\|\bm w_i\|_\infty^2
		+\|\widetilde{\bm w}_S\|_2^2
		&\leq Cs\lambda^2\log^2(6sd/\eta)\\
		&\leq C\frac{R^2s^2\log^2(d/\eta)\log(1/\delta)}
		{n^2\varepsilon^2},
	\end{align*}
	where the last inequality uses $s\leq d$ and
	$\log(6sd/\eta)\leq C\log(d/\eta)$. Thus the last two terms of \eqref{eq: corrected sparse mean expansion} are at most
	\begin{align*}
		C\frac{R^2s^2\log^2(d/\eta)\log(1/\delta)}{n^2\varepsilon^2}.
	\end{align*}
	Intersecting this noise event with \eqref{eq: truncated mean infinity} gives probability at least $1-\eta$. Substitution into \eqref{eq: corrected sparse mean expansion} and the assumption $s^*\leq s\leq Cs^*$ prove the theorem.
\end{proof}
}

\subsection{Proof of Lemma \ref{lm: low-dim regression privacy}}\label{sec: proof of lm: low-dim regression privacy}
{
\begin{proof}[Proof of Lemma \ref{lm: low-dim regression privacy}]
	First condition on the data-independent random partition. It is enough to prove privacy for every fixed realization of this partition, because averaging the resulting output laws over partition randomness preserves the same $(\varepsilon,\delta)$ guarantee. Fix an iteration $t$ and condition on all outputs from blocks $I_0,\ldots,I_{t-1}$. Under this conditioning, $\bbeta^t$ is fixed. Before adding $\bm w_t$ and projecting, the update is
	\begin{align*}
		G_t(I_t)=\bbeta^t-\frac{\eta^0}{b}
		\sum_{i\in I_t}\Pi_R(\overline{\bm x}_i^\top\bbeta^t-y_i)
		\overline{\bm x}_i,
	\end{align*}
	where $\overline{\bm x}_i=\mathsf C_{2,B_2}(\bm x_i)$ is a deterministic preprocessing of the record. Therefore
	$\|\overline{\bm x}_i\|_2\leq B_2$ for every possible input, without any distributional assumption.
	If $I_t$ and $I_t'$ differ by replacing $(y_i,\bm x_i)$ with $(y_i',\bm x_i')$, all summands except one cancel. The deterministic design bound above and $|\Pi_R(u)|\leq R$ give
	\begin{align*}
		\|G_t(I_t)-G_t(I_t')\|_2
		&\leq\frac{\eta^0}{b}\left\{
		R\|\overline{\bm x}_i\|_2+
		R\|\overline{\bm x}_i'\|_2\right\}\\
		&\leq\frac{2\eta^0RB_2}{b}
		=: \Delta_t.
	\end{align*}
	The Gaussian mechanism with sensitivity $\Delta_t$ requires coordinate variance at least a universal constant times
	\begin{align*}
		\frac{\Delta_t^2\log(1.25/\delta)}{\varepsilon^2}
		=\frac{4(\eta^0)^2R^2B_2^2\log(1.25/\delta)}
		{b^2\varepsilon^2}.
	\end{align*}
	The variance $32(\eta^0)^2R^2B_2^2\log(1.25/\delta)/(b^2\varepsilon^2)$ used in Algorithm \ref{algo: low-dim regression} is a conservative constant multiple of this quantity. Projection onto the Euclidean ball is post-processing, so the release at iteration $t$ is $(\varepsilon,\delta)$-DP conditional on the past.

	For completeness, consider two complete data sets differing in one observation. If the observation is discarded, the output laws coincide. Otherwise, because the blocks are disjoint, that observation lies in a unique block, say $I_{t_0}$. The iterates before $t_0$ have the same distribution on both data sets. The release at $t_0$ is $(\varepsilon,\delta)$-DP by the preceding conditional argument. Every later iterate is obtained from this private release, fresh algorithmic randomness, and raw observations in blocks disjoint from $I_{t_0}$. Thus all later computation is post-processing with respect to the changed observation. This proves that the final iterate is $(\varepsilon,\delta)$-DP, with no factor $T$ in the privacy parameters.
\end{proof}
}

\subsection{Proof of Theorem \ref{thm: low-dim regression upper bound}}\label{sec: proof of thm: low-dim regression upper bound}
{
\begin{proof}[Proof of Theorem \ref{thm: low-dim regression upper bound}]
	Put
	\begin{align*}
		\bm e_t=\bbeta^t-\bbeta,
		\qquad
		\widehat\Sigma_t=\frac1b\sum_{i\in I_t}\bm x_i\bm x_i^\top,
		\qquad
		\bm z_t=\frac1b\sum_{i\in I_t}\xi_i\bm x_i,
	\end{align*}
	where $\xi_i=y_i-\bm x_i^\top\bbeta\sim N(0,\sigma^2)$ is independent of $\bm x_i$. Since $\bbeta^t$ depends only on preceding blocks, it is independent of the observations in $I_t$. Projection guarantees $\|\bbeta^t\|_2\leq C$, while (P) gives $\|\bbeta\|_2\leq c_0\leq C$. Hence $\|\bm e_t\|_2\leq2C$.

	We first show that neither deterministic clipping operation changes the update on a high-probability event. Define
	\begin{align*}
		\mathcal E_{\bm x}
		=\left\{\max_{1\leq i\leq n}\|\bm x_i\|_2\leq B_2\right\}.
	\end{align*}
	For a fixed $i$, let $\mathcal N_{1/2}$ be a $1/2$-net of the Euclidean unit sphere with $|\mathcal N_{1/2}|\leq5^d$. The deterministic net inequality gives
	\begin{align*}
		\|\bm x_i\|_2
		\leq2\max_{u\in\mathcal N_{1/2}}|u^\top\bm x_i|.
	\end{align*}
	Assumption (D), the sub-Gaussian tail bound, and a union bound over the net therefore imply
	\begin{align*}
		\Pr(\|\bm x_i\|_2>B_2)
		&\leq\sum_{u\in\mathcal N_{1/2}}
		\Pr(|u^\top\bm x_i|>B_2/2)\\
		&\leq2\cdot5^d\exp\left(-\frac{cB_2^2}{K_{\bm x}^2}\right).
	\end{align*}
	After a union bound over $i\in[n]$, the stated choice
	\[
		B_2=C_2K_{\bm x}\sqrt{d+\log(8n/\eta)},
	\]
	with $C_2$ sufficiently large, yields
	\begin{align}\label{eq: low design no clipping}
		\Pr(\mathcal E_{\bm x}^c)\leq\frac{\eta}{4}.
	\end{align}
	On $\mathcal E_{\bm x}$, $\overline{\bm x}_i=\bm x_i$ for every observation.

	Next define the event that all \emph{raw} residuals used by the algorithm are below $R$:
	\begin{align*}
		\mathcal E_r
		=\left\{\max_{0\leq t<T}\max_{i\in I_t}
		|\bm x_i^\top\bbeta^t-y_i|\leq R\right\}.
	\end{align*}
	The iterate $\bbeta^t$ is a function only of preceding blocks and preceding algorithmic randomness, whether or not design clipping occurred there. Conditional on this past,
	\begin{align*}
		\bm x_i^\top\bbeta^t-y_i
		=\bm x_i^\top\bm e_t-\xi_i
	\end{align*}
	is sub-Gaussian with norm at most
	$C'(\sigma+K_{\bm x}\|\bm e_t\|_2)
	\leq C'(\sigma+K_{\bm x}C)$. Hence
	\begin{align*}
		\Pr\left(|\bm x_i^\top\bbeta^t-y_i|>R
		\mid I_0,\ldots,I_{t-1}\right)
		\leq2\exp\left\{-\frac{cR^2}{(\sigma+K_{\bm x}C)^2}\right\}.
	\end{align*}
	There are at most $bT\leq n$ used observations. Taking expectations in the last conditional inequality and summing over those observations shows, for the stated choice of $R$ with $C_3$ sufficiently large, that
	\begin{align}\label{eq: low residual no clipping}
		\Pr(\mathcal E_r^c)\leq\frac{\eta}{4}.
	\end{align}
	On $\mathcal E_{\bm x}\cap\mathcal E_r$, both design clipping and residual clipping are inactive at every iteration.

	We now establish the two concentration events used in the contraction. For a fixed unit vector $u$, the random variable $u^\top\bm x_i$ is sub-Gaussian with norm at most $K_{\bm x}$, so
	$(u^\top\bm x_i)^2-\E(u^\top\bm x_i)^2$ is sub-exponential with norm at most $CK_{\bm x}^2$. Bernstein's inequality gives, for $0<v\leq1$,
	\begin{align*}
		\Pr\left(\left|u^\top(\widehat\Sigma_t-\Sigma_{\bm x})u\right|>v\right)
		\leq2\exp(-cbv^2/K_{\bm x}^4).
	\end{align*}
	Let $\mathcal N$ be a $1/4$-net of the unit sphere with $|\mathcal N|\leq9^d$. The deterministic net inequality
	\begin{align*}
		\|A\|_{\mathrm{op}}
		\leq2\max_{u\in\mathcal N}|u^\top A u|
	\end{align*}
	holds for every symmetric matrix $A$. Taking $v=1/(4L)$ and using a union bound over $\mathcal N$ and the $T$ blocks shows, under the assumed lower bound on $b$, that
	\begin{align*}
		\max_{t<T}\|\widehat\Sigma_t-\Sigma_{\bm x}\|_{\mathrm{op}}
		\leq\frac1{2L}
	\end{align*}
	with probability at least $1-\eta/8$. Since the eigenvalues of $\Sigma_{\bm x}$ lie in $[L^{-1},L]$, this implies
	\begin{align*}
		(2L)^{-1}\bm I_d\preceq\widehat\Sigma_t\preceq2L\bm I_d
		\qquad\text{for all }t<T.
	\end{align*}

	Conditional on the design in block $I_t$, the score vector is Gaussian:
	\begin{align*}
		\bm z_t\mid\bm X_{I_t}
		\sim N_d\left(\bm0,\frac{\sigma^2}{b}\widehat\Sigma_t\right).
	\end{align*}
	On the covariance event, its conditional covariance is bounded above by
	$2L\sigma^2\bm I_d/b$. If $g\sim N_d(\bm0,\bm I_d)$, the chi-square tail inequality
	\begin{align*}
		\Pr\{\|g\|_2^2>d+2\sqrt{du}+2u\}\leq e^{-u}
	\end{align*}
	therefore implies, after taking $u=\log(8T/\eta)$ and a union bound over $t$, that
	\begin{align*}
		\max_{t<T}\|\bm z_t\|_2^2
		\leq C\frac{\sigma^2\{d+\log(8T/\eta)\}}b
	\end{align*}
	with conditional failure probability at most $\eta/8$. Combining the covariance and score events, we have simultaneously for all $t<T$, with probability at least $1-\eta/4$,
	\begin{align}\label{eq: low split events}
		(2L)^{-1}\bm I_d\preceq\widehat\Sigma_t\preceq2L\bm I_d,
		\qquad
		\|\bm z_t\|_2^2\lesssim\frac{\sigma^2\{d+\log(8T/\eta)\}}b.
	\end{align}
	On $\mathcal E_{\bm x}\cap\mathcal E_r$, the update can be written exactly in terms of the raw design as
	\begin{align*}
		\bbeta^{t+1}
		=\Pi_C\{\bbeta^t-\eta^0\widehat\Sigma_t\bm e_t
		+\eta^0\bm z_t+\bm w_t\}.
	\end{align*}
	Because $\bbeta$ lies in the projection ball, nonexpansiveness gives
	\begin{align*}
		\|\bm e_{t+1}\|_2
		\leq\|(\bm I_d-\eta^0\widehat\Sigma_t)\bm e_t
		+\eta^0\bm z_t+\bm w_t\|_2.
	\end{align*}
	For $\eta^0=1/(4L)$ and eigenvalues of $\widehat\Sigma_t$ in $[(2L)^{-1},2L]$,
	\begin{align*}
		q_L:=\|\bm I_d-\eta^0\widehat\Sigma_t\|_{\mathrm{op}}
		\leq\max\left\{1-\frac1{8L^2},\frac12\right\}<1.
	\end{align*}
	Choose $\tau>0$ depending only on $L$ so that $\rho_0=(1+\tau)q_L^2<1$. Applying
	$\|u+v\|_2^2\leq(1+\tau)\|u\|_2^2+(1+1/\tau)\|v\|_2^2$ and then
	$\|\eta^0\bm z_t+\bm w_t\|_2^2\leq C_L(\|\bm z_t\|_2^2+\|\bm w_t\|_2^2)$ yields
	\begin{align*}
		\|\bbeta^{t+1}-\bbeta\|_2^2
		\leq\rho_0\|\bbeta^t-\bbeta\|_2^2+C\bigl(\|\bm z_t\|_2^2+\|\bm w_t\|_2^2\bigr)
	\end{align*}
	for a constant $0<\rho_0<1$ depending only on $L$.

	It remains to bound the privacy noise. Write $\bm w_t=s_wg_t$, where $g_t\sim N_d(\bm0,\bm I_d)$ and
	\begin{align*}
		s_w^2=32(\eta^0)^2R^2B_2^2
		\frac{\log(1.25/\delta)}{b^2\varepsilon^2}.
	\end{align*}
	The same chi-square inequality, now with $u=\log(4T/\eta)$, and a union bound over $t<T$ show, with probability at least $1-\eta/4$, that
	\begin{align*}
		\max_{t<T}\|\bm w_t\|_2^2
		&\leq Cs_w^2\{d+\log(4T/\eta)\}\\
		&\lesssim \frac{B_2^2R^2\{d+\log(8T/\eta)\}
		\log(1/\delta)}{b^2\varepsilon^2}.
	\end{align*}

	Intersect the design-no-clipping event \eqref{eq: low design no clipping}, the residual-no-clipping event \eqref{eq: low residual no clipping}, the two concentration events in \eqref{eq: low split events}, and the privacy-noise event. Their total failure probability is at most
	$\eta/4+\eta/4+\eta/4+\eta/4=\eta$. Let $B$ denote the sum of the two non-recursive terms in the claimed bound. The preceding inequalities give
	\begin{align*}
		\|\bm e_{t+1}\|_2^2\leq\rho_0\|\bm e_t\|_2^2+B,
		\qquad t=0,\ldots,T-1.
	\end{align*}
	Iterating this scalar recursion yields
	\begin{align*}
		\|\bm e_T\|_2^2
		&\leq\rho_0^T\|\bm e_0\|_2^2
		+B\sum_{j=0}^{T-1}\rho_0^j\\
		&\leq\rho_0^Tc_0^2+\frac{B}{1-\rho_0}.
	\end{align*}
	Since $1-\rho_0$ depends only on $L$, it is absorbed in the implicit constant. Choosing $C_1$ in $T=\lceil C_1L^2\log(en)\rceil$ sufficiently large makes $\rho_0^Tc_0^2\leq n^{-1}$. Substituting the displayed bounds for $\bm z_t$ and $\bm w_t$ proves \eqref{eq: non-sparse glm upper bound}.
\end{proof}
}

\subsection{Proof of Lemma \ref{lm: high-dim regression privacy}}\label{sec: proof of lm: high-dim regression privacy}
{
\begin{proof}[Proof of Lemma \ref{lm: high-dim regression privacy}]
	Condition on the data-independent random partition as in the preceding proof. Fix $t$ and condition on the preceding releases, so $\bbeta^t$ is fixed. Before Peeling, the update is
	\begin{align*}
		G_t(I_t)=\bbeta^t-\frac{\eta^0}{b}
		\sum_{i\in I_t}\Pi_R(\overline{\bm x}_i^\top\bbeta^t-y_i)
		\overline{\bm x}_i,
	\end{align*}
	where $\overline{\bm x}_i=\mathsf C_{\infty,B_\infty}(\bm x_i)$ and therefore
	$\|\overline{\bm x}_i\|_\infty\leq B_\infty$ for every possible input.
	For adjacent versions $I_t,I_t'$ of the current block, coordinate by coordinate,
	\begin{align*}
		\|G_t(I_t)-G_t(I_t')\|_\infty
		&\leq\frac{\eta^0}{b}
		\left\{R\|\overline{\bm x}_i\|_\infty+
		R\|\overline{\bm x}_i'\|_\infty\right\}\\
		&\leq\frac{2\eta^0RB_\infty}b.
	\end{align*}
	{Thus the exact $\ell_\infty$ sensitivity passed to Peeling in Algorithm \ref{algo: high-dim regression} is
	\begin{align*}
		\lambda_t^{\rm sens}=\frac{2\eta^0RB_\infty}{b}.
	\end{align*}
	Substitution into \eqref{eq: exact peeling scale} gives the scalar Laplace scale
	\begin{align*}
		b_t^{\rm peel}
		=\frac{16\eta^0RB_\infty\sqrt{2s\log(1/\delta)}}
		{b\varepsilon}.
	\end{align*}
	Lemma \ref{lm: peeling privacy} therefore makes $\bbeta^{t+1}$ $(\varepsilon,\delta)$-DP conditional on the past;} projection is post-processing. Finally, if two full data sets differ in one observation, that observation belongs to exactly one block. The same adaptive parallel-composition argument given in the proof of Lemma \ref{lm: low-dim regression privacy} shows that later iterates are post-processing of the unique private release involving that block. Hence the final output is $(\varepsilon,\delta)$-DP.
\end{proof}
}

\subsection{Proof of Lemma \ref{lm: high-dim regression contraction}}\label{sec: proof of lm: high-dim regression contraction}
{
We first prove the precise noisy hard-thresholding inequality needed below.
\begin{Lemma}\label{lm: peeling restricted accuracy}
	Let $\bm v\in\R^d$, let $\bm x$ be supported on a set $S^*$ with $|S^*|\leq s^*<s$, and let $S$ be the size-$s$ support selected by Peeling before the final release noise is added. Put $\overline{\bm v}=\bm v_S$ and
	\begin{align*}
		E_w=\sum_{j=1}^s\|\bm w_j\|_\infty^2.
	\end{align*}
	For every $c,\zeta>0$,
	\begin{align}\label{eq: peeling restricted accuracy}
		\|\overline{\bm v}-\bm x\|_2^2
		\leq \nu_{a,\zeta}^2
		\|(\bm v-\bm x)_{S\cup S^*}\|_2^2
		+C_{c,\zeta}E_w,
	\end{align}
	where
	\begin{align*}
		a^2=(1+c)\frac{s^*}{s-s^*},
		\qquad
		\nu_{a,\zeta}^2=(1+\zeta)
		\left(1+\frac{a^2}{2}+\frac a2\sqrt{a^2+4}\right).
	\end{align*}
\end{Lemma}
\begin{proof}
	Write $A=S\setminus S^*$ and $B=S^*\setminus S$. If $B=\emptyset$, then $\bm x$ is supported on $S$ and
	\begin{align*}
		\|\overline{\bm v}-\bm x\|_2^2
		=\|(\bm v-\bm x)_S\|_2^2
		\leq\|(\bm v-\bm x)_{S\cup S^*}\|_2^2.
	\end{align*}
	Since $\nu_{a,\zeta}^2\geq1$, the lemma follows in this case. We therefore assume $B\neq\emptyset$.

	Because $|S|=s$ and $|S^*|\leq s^*$,
	\begin{align*}
		|A|=s-|S\cap S^*|\geq s-s^*,
		\qquad |B|=|S^*|-|S\cap S^*|\leq s^*.
	\end{align*}
	Let $A_0\subseteq A$ index the $|B|$ smallest coordinates of $\bm v_A$ in absolute value. Lemma \ref{lm: peeling accuracy}, applied with $R_1=A_0$ and $R_2=B$, yields
	\begin{align}
		\|\bm v_B\|_2^2
		&\leq(1+c)\|\bm v_{A_0}\|_2^2
		+4(1+1/c)E_w \notag\\
		&\leq(1+c)\frac{|B|}{|A|}\|\bm v_A\|_2^2
		+4(1+1/c)E_w \notag\\
		&\leq a^2\|(\bm v-\bm x)_A\|_2^2
		+4(1+1/c)E_w. \label{eq: missed selected comparison}
	\end{align}
	To justify the middle inequality explicitly, order the values $v_j^2$, $j\in A$, increasingly. The average of the first $|B|$ values cannot exceed the average of all $|A|$ values, so
	\begin{align*}
		\|\bm v_{A_0}\|_2^2
		\leq\frac{|B|}{|A|}\|\bm v_A\|_2^2.
	\end{align*}
	The last inequality in \eqref{eq: missed selected comparison} uses $\bm x_A=\bm0$ and
	$(1+c)|B|/|A|\leq(1+c)s^*/(s-s^*)=a^2$.

	Let $\bm h=\bm v-\bm x$, $p=\|\bm h_A\|_2$, and $q=\|\bm h_B\|_2$. By \eqref{eq: missed selected comparison} and the triangle inequality,
	\begin{align*}
		\|\bm x_B\|_2
		=\|\bm v_B-\bm h_B\|_2
		\leq q+ap+2\sqrt{(1+1/c)E_w}.
	\end{align*}
	Put $r_w=2\sqrt{(1+1/c)E_w}$. Since
	$\|\bm x_B\|_2\leq q+ap+r_w$, the inequality
	$(u+v)^2\leq(1+\zeta)u^2+(1+1/\zeta)v^2$, applied with $u=q+ap$ and $v=r_w$, gives
	\begin{align*}
		\|\overline{\bm v}-\bm x\|_2^2
		&=\|\bm h_{S\cap S^*}\|_2^2+p^2+\|\bm x_B\|_2^2\\
		&\leq(1+\zeta)
		\left\{\|\bm h_{S\cap S^*}\|_2^2+p^2+(q+ap)^2\right\}
		+C_{c,\zeta}E_w.
	\end{align*}
	The quadratic form in $(p,q)$ satisfies
	\begin{align*}
		p^2+(q+ap)^2
		=\begin{pmatrix}p&q\end{pmatrix}
		\begin{pmatrix}1+a^2&a\\a&1\end{pmatrix}
		\begin{pmatrix}p\\q\end{pmatrix}.
	\end{align*}
	The two eigenvalues of
	\begin{align*}
		\begin{pmatrix}1+a^2&a\\a&1\end{pmatrix}
	\end{align*}
	are the roots of
	$\lambda^2-(2+a^2)\lambda+1=0$. The larger root is
	$1+a^2/2+(a/2)\sqrt{a^2+4}$. Therefore
	\begin{align*}
		p^2+(q+ap)^2
		\leq\left(1+\frac{a^2}{2}+\frac a2\sqrt{a^2+4}\right)(p^2+q^2),
	\end{align*}
	which proves \eqref{eq: peeling restricted accuracy}.
\end{proof}
}

{
\begin{proof}[Proof of Lemma \ref{lm: high-dim regression contraction}]
	Let $S_t=\supp(\bbeta^t)$, let $S=S^{t+1}$ denote the support selected by Peeling, and put $S^*=\supp(\bbeta)$ and $\bm e_t=\bbeta^t-\bbeta$. On the event that both design clipping and residual clipping are inactive,
	\begin{align}\label{eq: split iht unthresholded step}
		\bm u_t:=\bbeta^{t+0.5}
		=\bbeta^t-\eta^0\widehat\Sigma_t\bm e_t+\eta^0\bm z_t,
		\qquad
		\bm z_t=\frac1b\sum_{i\in I_t}\xi_i\bm x_i.
	\end{align}
	Set $U=S_t\cup S\cup S^*$ and $V=S\cup S^*$. Because $|S_t|\leq s$, $|S|=s$, and $|S^*|\leq s^*$,
	\begin{align*}
		|U|\leq2s+s^*,\qquad |V|\leq s+s^*.
	\end{align*}
	Moreover, $\bm e_t$ is supported on $S_t\cup S^*\subseteq U$. The restricted covariance hypothesis implies
	\begin{align*}
		\|\bm I_U-\eta^0\widehat\Sigma_{t,U,U}\|_{\mathrm{op}}
		\leq q_L<1.
	\end{align*}
	Since $\bm e_t$ vanishes outside $U$, the restriction to $U$ of the deterministic part in \eqref{eq: split iht unthresholded step} is
	\begin{align*}
		\{\bm e_t-\eta^0\widehat\Sigma_t\bm e_t\}_U
		=(\bm I_U-\eta^0\widehat\Sigma_{t,U,U})\bm e_{t,U}.
	\end{align*}
	Restricting further from $U$ to $V$ cannot increase the Euclidean norm. Therefore
	\begin{align}
		\|(\bm u_t-\bbeta)_V\|_2
		&\leq q_L\|\bm e_t\|_2
		+\eta^0\|\bm z_{t,V}\|_2 \notag\\
		&\leq q_L\|\bm e_t\|_2
		+\eta^0\sqrt{s+s^*}\|\bm z_t\|_\infty. \label{eq: restricted gradient contraction}
	\end{align}
	The last inequality uses
	$\|\bm z_{t,V}\|_2\leq\sqrt{|V|}\|\bm z_t\|_\infty$ and $|V|\leq s+s^*$. Squaring with Young's inequality gives, for every $\tau>0$,
	\begin{align}\label{eq: restricted gradient contraction squared}
		\|(\bm u_t-\bbeta)_V\|_2^2
		\leq(1+\tau)q_L^2\|\bm e_t\|_2^2
		+C_{L,\tau}s\|\bm z_t\|_\infty^2,
	\end{align}
	where we used $s^*\leq s$ and absorbed the fixed step size into the constant.

	Let $\overline\bbeta^{t+1}=(\bm u_t)_S$ be the vector after support selection but before the final Peeling noise. Lemma \ref{lm: peeling restricted accuracy}, with $\bm v=\bm u_t$ and $\bm x=\bbeta$, first gives
	\begin{align*}
		\|\overline\bbeta^{t+1}-\bbeta\|_2^2
		\leq\nu_{a,\zeta}^2\|(\bm u_t-\bbeta)_V\|_2^2
		+C_{c,\zeta}\sum_{j=1}^s\|\bm w_{t,j}\|_\infty^2.
	\end{align*}
	Substituting \eqref{eq: restricted gradient contraction squared} and absorbing fixed constants depending only on $L,c,\zeta,$ and $\tau$ yields
	\begin{align}\label{eq: pre-release iht contraction}
		\|\overline\bbeta^{t+1}-\bbeta\|_2^2
		&\leq \nu_{a,\zeta}^2(1+\tau)q_L^2
		\|\bm e_t\|_2^2
		+C_Ls\|\bm z_t\|_\infty^2
		+C_L\sum_{j=1}^s\|\bm w_{t,j}\|_\infty^2.
	\end{align}
	We now verify that the leading coefficient can be made smaller than one. Fix $L$, hence fix $q_L<1$. First choose $\zeta>0$ and $\tau>0$ sufficiently small that $(1+\zeta)(1+\tau)q_L^2<1$. Next increase the ratio $s/s^*$: since
	\begin{align*}
		a^2=(1+c)\frac{s^*}{s-s^*}\longrightarrow0,
	\end{align*}
	the factor $\nu_{a,\zeta}^2$ decreases to $1+\zeta$. Thus the constant $C_L$ in $s\geq C_Ls^*$ can be chosen so that
	\begin{align*}
		\rho_1:=\nu_{a,\zeta}^2(1+\tau)q_L^2<1.
	\end{align*}

	Finally, Algorithm \ref{algo: high-dim regression} forms
	\begin{align*}
		\bbeta^{t+1}=\Pi_C\{\overline\bbeta^{t+1}
		+\widetilde{\bm w}_{t,S}\}.
	\end{align*}
	Since $\bbeta$ belongs to the projection ball, nonexpansiveness of $\Pi_C$ and another application of Young's inequality give, for every $\tau'>0$,
	\begin{align*}
		\|\bbeta^{t+1}-\bbeta\|_2^2
		&\leq(1+\tau')\|\overline\bbeta^{t+1}-\bbeta\|_2^2
		+(1+1/\tau')\|\widetilde{\bm w}_{t,S}\|_2^2.
	\end{align*}
	Because $\rho_1<1$, one can choose $\tau'>0$ small enough that $\rho_0=(1+\tau')\rho_1<1$. Substituting \eqref{eq: pre-release iht contraction}, and absorbing the fixed factors $1+\tau'$ and $1+1/\tau'$ into the constant $C$, gives exactly \eqref{eq: high-dim regression contraction}.
\end{proof}
}

\clearpage
\section{ {Proofs of the lower bounds}}\label{sec: lb proofs}
{  The low-dimensional mean lower bound is proved in Section \ref{sec: proof of thm: low-dim mean lower bound} of the main paper by a resampling reduction to the fingerprinting construction of \cite{steinke2017between}.  Here we begin with a private packing lemma and use it for both sparse lower bounds.  We then prove the low-dimensional regression bound separately: first the attack inequalities, then the score identity, and finally the minimax reduction.  This ordering makes explicit which parts use ordinary information theory, which use differential privacy, and where the approximate-DP remainder enters.}

\subsection{Robust fingerprinting codes: black-box theorem}\label{sec: robust fingerprinting theorem}
{
For a code database $C=(C_1,\ldots,C_m)\in\{-1,1\}^{m\times d}$, write
$\overline C=m^{-1}\sum_{i=1}^m C_i\in[-1,1]^d$. The following is the exact external result used in the resampling proof in the main paper.

\begin{Theorem}[Robust $\ell_1$ fingerprinting code; Theorem 3.3 of \citealp{steinke2017between}]\label{thm: robust fingerprinting code}
	There is a universal constant $C_0>0$ such that, for every $m\geq1$, every $0<\gamma<1$, and every integer
	\begin{align*}
		d\geq C_0m^2\log(1/\gamma),
	\end{align*}
	there is a joint distribution over a code database
	$C\in\{-1,1\}^{m\times d}$ and a (possibly randomized) tracing map
	$\mathsf{Trace}:[-1,1]^d\to2^{[m]}$ with the following guarantees for every fixed, possibly randomized mechanism
	$\mathsf B:\{-1,1\}^{m\times d}\to[-1,1]^d$:
	\begin{enumerate}
		\item \emph{Completeness and $1/8$ robustness:}
		\begin{align*}
			\Pr\left(
			\|\mathsf B(C)-\overline C\|_1\leq\frac d8,
			\ \mathsf{Trace}(\mathsf B(C))=\emptyset
			\right)\leq\frac1{100}.
		\end{align*}
		\item \emph{$\gamma$-soundness:} for every $i\in[m]$,
		\begin{align*}
			\Pr\left\{i\in\mathsf{Trace}(\mathsf B(C_{-i}))\right\}
			\leq\gamma,
		\end{align*}
		where $C_{-i}$ is obtained from $C$ by replacing row $i$ with a fixed public vector in $\{-1,1\}^d$.
	\end{enumerate}
	All probabilities include the randomness of the code, tracer, and mechanism. Thus the code is $1/100$-complete, $\gamma$-sound, and $1/8$-robust in the terminology of \cite{steinke2017between}.
\end{Theorem}
}

{  \subsection{A private packing lemma}\label{sec: private packing lemma}
We first record and prove the precise approximate-DP version of the private Fano argument used in both sparse lower bounds. Its dependence on $e^\varepsilon-1$ is essential when $\varepsilon=o(1)$.

\begin{Lemma}[Private packing]\label{lm: private packing}
	Let $\{P_v:v\in\mathcal V\}$ be a family of one-observation distributions with parameters $\{\theta_v:v\in\mathcal V\}$, where $N=|\mathcal V|\geq2$. Suppose that, for all distinct $v,w\in\mathcal V$,
	\begin{align*}
		\|\theta_v-\theta_w\|_2\geq 4a,\qquad
		D_{\mathrm{KL}}(P_v\|P_w)\leq \kappa^2,
		\qquad d_{\mathrm{TV}}(P_v,P_w)\leq \tau.
	\end{align*}
	There are universal constants $c_i,C_i>0$ such that every $(\varepsilon,\delta)$-differentially private estimator $M$ based on $n$ i.i.d. observations, with $0<\varepsilon\leq1$ and $0\leq\delta<1$, satisfies
	\begin{align}\label{eq: private packing conclusion}
		\sup_{v\in\mathcal V}\E_v\|M-\theta_v\|_2^2\geq c_1a^2
	\end{align}
	if either
	\begin{align}
		n\kappa^2\leq c_2\log N, \label{eq: private packing ordinary condition}
	\end{align}
	or both
	\begin{align}
		n\varepsilon\tau&\leq c_2\log N,\label{eq: private packing privacy condition}\\
		\delta&\leq c_3(e^\varepsilon-1)N^{-C_3}.
		\label{eq: private packing delta}
	\end{align}
\end{Lemma}

\begin{proof}
	Given the output of $M$, let $\widehat V$ be a nearest-neighbor decoder in the finite set $\{\theta_v:v\in\mathcal V\}$, with ties broken deterministically, and put
	\begin{align*}
		\beta_v=\Pr_v(\widehat V\neq v),
		\qquad B=\sum_{v\in\mathcal V}\beta_v.
	\end{align*}
	We first prove the ordinary, non-private alternative. Let $V$ be uniform on $\mathcal V$, and, conditional on $V=v$, draw the data from $P_v^{\otimes n}$. For $N\geq3$, Fano's inequality gives
	\begin{align*}
		\frac BN=\Pr(\widehat V\neq V)
		\geq1-\frac{I(V;\bm X)+\log2}{\log N}.
	\end{align*}
	If $\overline P=N^{-1}\sum_wP_w^{\otimes n}$ denotes the mixture distribution, then
	\begin{align*}
		I(V;\bm X)
		&=\frac1N\sum_vD_{\mathrm{KL}}(P_v^{\otimes n}\|\overline P)\\
		&\leq\frac1{N^2}\sum_{v,w}
		D_{\mathrm{KL}}(P_v^{\otimes n}\|P_w^{\otimes n})\\
		&=\frac n{N^2}\sum_{v,w}D_{\mathrm{KL}}(P_v\|P_w)
		\leq n\kappa^2.
	\end{align*}
	The first inequality is the log-sum inequality, and the equality uses additivity of KL divergence for product measures. Thus \eqref{eq: private packing ordinary condition}, with $c_2$ sufficiently small, makes $B/N$ larger than a universal positive constant when $N\geq3$.

	When $N=2$, let the two hypotheses be $v$ and $w$. Pinsker's inequality and the same KL calculation give
	\begin{align*}
		d_{\mathrm{TV}}(P_v^{\otimes n},P_w^{\otimes n})
		\leq\sqrt{\frac12D_{\mathrm{KL}}(P_v^{\otimes n}\|P_w^{\otimes n})}
		\leq\sqrt{n\kappa^2/2}.
	\end{align*}
	The minimum average error of any binary test is
	$\{1-d_{\mathrm{TV}}(P_v^{\otimes n},P_w^{\otimes n})\}/2$. Again taking $c_2$ sufficiently small makes this error at least a positive universal constant. This completes the proof under \eqref{eq: private packing ordinary condition} for every $N\geq2$.

	We next prove the private alternative directly. Fix distinct $v,w$. A maximal coupling $(X,Y)$ of $P_v$ and $P_w$ satisfies
	$\Pr(X\neq Y)=d_{\mathrm{TV}}(P_v,P_w)\leq\tau$. Take $n$ independent copies of this coupling. The resulting data sets $\bm X$ and $\bm Y$ have marginals $P_v^{\otimes n}$ and $P_w^{\otimes n}$, respectively. If
	\begin{align*}
		H=\sum_{i=1}^n\1\{X_i\neq Y_i\}
	\end{align*}
	is their Hamming distance, then $\E H\leq n\tau$. Set
	\begin{align*}
		r=\lceil10n\tau\rceil,
		\qquad
		\Delta_r=\delta\frac{e^{r\varepsilon}-1}{e^\varepsilon-1},
	\end{align*}
	with $\Delta_0=0$. If $n\tau>0$, then $r\geq10n\tau$, so Markov's inequality gives
	\begin{align*}
		\Pr(H>r)\leq\frac{\E H}{r}\leq0.1.
	\end{align*}
	If $n\tau=0$, then $H=0$ almost surely. Repeated application of $(\varepsilon,\delta)$-DP gives, for deterministic data sets $x,y$ with Hamming distance at most $r$ and every output event $E$,
	\begin{align*}
		\Pr\{M(x)\in E\}
		&\leq e^{r\varepsilon}\Pr\{M(y)\in E\}+\Delta_r,\\
		\Pr\{M(x)\in E\}
		&\geq e^{-r\varepsilon}
		\left[\Pr\{M(y)\in E\}-\Delta_r\right].
	\end{align*}
	The second line is obtained by interchanging $x$ and $y$ in the first line and rearranging.

	Apply this second inequality pointwise to the coupled pair $(\bm X,\bm Y)$ on $\{H\leq r\}$, with $E=\{\widehat V=w\}$. Integrating over the coupling gives
	\begin{align*}
		\Pr_v(\widehat V=w)
		&\geq\Pr(\widehat V(\bm X)=w,H\leq r)\\
		&\geq e^{-r\varepsilon}
		\left\{\Pr(\widehat V(\bm Y)=w,H\leq r)
		-\Delta_r\Pr(H\leq r)\right\}\\
		&\geq e^{-r\varepsilon}
		\left\{\Pr_w(\widehat V=w)-\Pr(H>r)-\Delta_r\right\}.
	\end{align*}
	Since $\Pr_w(\widehat V=w)=1-\beta_w$ and $\Pr(H>r)\leq0.1$, this proves
	\begin{align}\label{eq: approximate private pairwise testing}
		\Pr_v(\widehat V=w)
		&\geq e^{-r\varepsilon}
		\{0.9-\beta_w-\Delta_r\}.
	\end{align}

	We now average the pairwise inequalities. For each fixed $v$, an incorrect decoder output must be one of the values $w\neq v$, and hence
	\begin{align*}
		\beta_v=\sum_{w\neq v}\Pr_v(\widehat V=w).
	\end{align*}
	Summing \eqref{eq: approximate private pairwise testing} over all ordered pairs $(v,w)$ with $v\neq w$ therefore yields
	\begin{align*}
		B
		&\geq e^{-r\varepsilon}
		\sum_v\sum_{w\neq v}(0.9-\Delta_r-\beta_w)\\
		&=e^{-r\varepsilon}
		\left\{N(N-1)(0.9-\Delta_r)-(N-1)B\right\}.
	\end{align*}
	Moving the last term to the left and dividing by $N\{1+(N-1)e^{-r\varepsilon}\}$ gives
	\begin{align}\label{eq: approximate private average testing error}
		\frac BN
		\geq
		\frac{(N-1)e^{-r\varepsilon}(0.9-\Delta_r)}
		{1+(N-1)e^{-r\varepsilon}}.
	\end{align}
	Under \eqref{eq: private packing privacy condition},
	\begin{align*}
		r\varepsilon\leq10n\varepsilon\tau+\varepsilon
		\leq10c_2\log N+1.
	\end{align*}
	For example, if $10c_2\leq1/4$, then
	\begin{align*}
		(N-1)e^{-r\varepsilon}
		\geq e^{-1}(N-1)N^{-1/4},
	\end{align*}
	which is bounded below by a positive universal constant for every $N\geq2$. Moreover, using $e^{r\varepsilon}-1\leq e^{r\varepsilon}$, \eqref{eq: private packing delta} gives
	\begin{align*}
		\Delta_r
		\leq c_3N^{-C_3}e^{r\varepsilon}
		\leq ec_3N^{10c_2-C_3}\leq0.1
	\end{align*}
	when, for instance, $C_3>10c_2+1$ and $c_3$ is sufficiently small. Thus $0.9-\Delta_r\geq0.8$, and both the numerator factor and the denominator ratio in \eqref{eq: approximate private average testing error} are bounded away from zero. Consequently $B/N\geq c$ for a universal constant $c>0$.

	Finally, suppose $\|M-\theta_v\|_2<2a$. For every $w\neq v$, the triangle inequality and the $4a$ separation give
	\begin{align*}
		\|M-\theta_w\|_2
		\geq\|\theta_w-\theta_v\|_2-\|M-\theta_v\|_2
		>2a>\|M-\theta_v\|_2.
	\end{align*}
	Thus $\theta_v$ is the unique nearest neighbor and $\widehat V=v$. Taking the contrapositive, $\widehat V\neq v$ implies $\|M-\theta_v\|_2\geq2a$. Hence
	\begin{align*}
		\sup_v\E_v\|M-\theta_v\|_2^2
		\geq \frac1N\sum_v\E_v\|M-\theta_v\|_2^2
		\geq4a^2\frac BN,
	\end{align*}
	which proves \eqref{eq: private packing conclusion}.
\end{proof}}

\subsection{Proof of Theorem \ref{thm: high-dim mean lower bound}}\label{sec: proof of thm: high-dim mean lower bound}
\begin{proof}
	{  Put $s=s^*$, $m=m_{d,s}$, and
	$q=q_{d,s}(\varepsilon,\delta)$.  We first construct the packing used
	throughout the proof.  The constant-weight Varshamov--Gilbert lemma gives
	\begin{align}\label{eq: constant weight packing}
		\mathcal V\subseteq\{v\in\{0,1\}^d:\|v\|_0=s\},\qquad
		\log|\mathcal V|\geq c_0s\log(ed/s),
	\end{align}
	and, for distinct $v,w\in\mathcal V$,
	\begin{align}\label{eq: constant weight distance}
		c_0s\leq\|v-w\|_0\leq2s.
	\end{align}
	For completeness, this follows by taking a maximal code under Hamming
	distance $c_0s$.  A Hamming ball inside the constant-weight slice contains
	at most
	\begin{align*}
		\sum_{j\leq c_0s/2}\binom{s}{j}\binom{d-s}{j}
		\leq\exp\{C c_0s\log(ed/s)\}
	\end{align*}
	points, whereas the whole slice has size $\binom{d}{s}\geq
	\exp\{c s\log(ed/s)\}$.  Taking $c_0$ small gives
	\eqref{eq: constant weight packing}; the upper bound in
	\eqref{eq: constant weight distance} holds because both vectors have
	weight $s$.  This argument also covers $s=1$; when the resulting code has
	only two elements, the binary part of Lemma \ref{lm: private packing}
	applies.

	Let $0<\alpha\leq\sqrt{s}/4$ and define
	\begin{align*}
		\bmu_v=\frac{\alpha}{\sqrt{s}}v,
		\qquad P_v=N_d(\bmu_v,\sigma^2\bm I_d).
	\end{align*}
	Every $\bmu_v$ is $s$-sparse and has coordinates smaller than one, so
	$P_v\in\mathcal P(\sigma,d,s,\Theta)$.  Equations
	\eqref{eq: constant weight distance} and the identity
	\begin{align*}
		\|\bmu_v-\bmu_w\|_2^2
		=\frac{\alpha^2}{s}\|v-w\|_0
	\end{align*}
	show that $c\alpha\leq\|\bmu_v-\bmu_w\|_2\leq
	\sqrt2\alpha$.  Thus the separation parameter $a$ in Lemma
	\ref{lm: private packing} may be taken to be $c_1\alpha$.  The Gaussian
	KL formula and Pinsker's inequality give, respectively,
	\begin{align}\label{eq: sparse mean divergences}
		D_{\mathrm{KL}}(P_v\|P_w)
		&=\frac{\|\bmu_v-\bmu_w\|_2^2}{2\sigma^2}
		\leq C\frac{\alpha^2}{\sigma^2},\\
		d_{\mathrm{TV}}(P_v,P_w)
		&\leq\min\left\{1,
		\sqrt{D_{\mathrm{KL}}(P_v\|P_w)/2}\right\}
		\leq\min\{1,C\alpha/\sigma\}.
	\end{align}

	We first obtain the ordinary statistical term.  Set
	\begin{align*}
		\alpha_{\rm stat}=c_2\min\left\{\sqrt{s},
		\sigma\sqrt{m/n}\right\}.
	\end{align*}
	The first entry in the minimum guarantees
	$\alpha_{\rm stat}\leq\sqrt{s}/4$ after decreasing $c_2$.  By
	\eqref{eq: sparse mean divergences}, the one-observation KL bound in
	Lemma \ref{lm: private packing} satisfies
	\begin{align*}
		n\kappa^2\leq Cn\alpha_{\rm stat}^2/\sigma^2
		\leq Cc_2^2m\leq c\log|\mathcal V|.
	\end{align*}
	The ordinary alternative of that lemma therefore yields risk at least
	$c\alpha_{\rm stat}^2$.

	It remains to prove the privacy term.  If $q=0$, define
	$\alpha_{\rm priv}=0$ and there is no privacy term to establish.  Suppose
	$q>0$.  When $q$ is larger than a fixed universal constant, choose a
	subset $\mathcal V_q\subseteq\mathcal V$ of cardinality between
	$\exp(cq)$ and $\exp(Cq)$; this is possible because $q\leq m$ and
	\eqref{eq: constant weight packing} holds.  When $q$ is bounded, take any
	two distinct elements.  In both cases, with $N_q=|\mathcal V_q|$,
	\begin{align}\label{eq: sparse mean subpacking size}
		c\max\{1,q\}\leq\log N_q\leq C\max\{1,q\}.
	\end{align}
	Set
	\begin{align}\label{eq: sparse mean private amplitude}
		\alpha_{\rm priv}=c_2\min\left\{\sqrt{s},
		\frac{\sigma q}{n\varepsilon}\right\}.
	\end{align}
	We now verify the private-testing condition without suppressing the
	total-variation truncation.  Write
	$u=\min\{\sqrt{s},\sigma q/(n\varepsilon)\}$, so that
	$\alpha_{\rm priv}=c_2u$.  If
	$C\alpha_{\rm priv}/\sigma<1$, then either
	$u=\sigma q/(n\varepsilon)$, which directly gives
	\begin{align*}
		n\varepsilon C\alpha_{\rm priv}/\sigma\leq Cc_2q,
	\end{align*}
	or $u=\sqrt{s}$.  In the latter case, the fact that this entry is the
	minimum says $n\varepsilon\sqrt{s}/\sigma\leq q$, and the same bound
	follows.  If $C\alpha_{\rm priv}/\sigma\geq1$, then the total-variation
	bound is one.  When $u=\sigma q/(n\varepsilon)$, the displayed inequality
	implies $n\varepsilon\leq Cc_2q$; when $u=\sqrt{s}$, it follows from
	$n\varepsilon\sqrt{s}/\sigma\leq q$ and
	$\sigma\leq Cc_2\sqrt{s}$.  Consequently, in all four cases,
	\begin{align}\label{eq: sparse mean private testing check}
		n\varepsilon d_{\mathrm{TV}}(P_v,P_w)
		\leq Cc_2q\leq c\log N_q,
	\end{align}
	after choosing $c_2$ sufficiently small.

	We next check the $\delta$ condition.  The definition
	\eqref{eq: effective private entropy} implies, whenever $q>0$,
	\begin{align}\label{eq: sparse mean effective entropy delta implication}
		\delta\leq c_\star(e^\varepsilon-1)e^{-q/c_\star}.
	\end{align}
	Indeed, if the logarithmic entry defines $q$, this is equality up to the
	outer factor $c_\star$; if the minimum is attained at $m$, the logarithmic
	entry is at least $m=q$, which gives the same inequality.  For unbounded
	$q$, combine \eqref{eq: sparse mean subpacking size} and
	\eqref{eq: sparse mean effective entropy delta implication}; taking
	$c_\star$ small enough gives
	\begin{align*}
		\delta\leq c_3(e^\varepsilon-1)N_q^{-C_3}.
	\end{align*}
	For bounded positive $q$, positivity of the bracket in
	\eqref{eq: effective private entropy} gives
	$\delta<c_\star(e^\varepsilon-1)$, and choosing
	$c_\star\leq c_3 2^{-C_3}$ proves the same assertion for $N_q=2$.
	Thus both hypotheses of the private alternative in Lemma
	\ref{lm: private packing} hold, and that lemma gives risk at least
	$c\alpha_{\rm priv}^2$.

	Finally, for nonnegative $x,y$,
	\begin{align*}
		\min\{s,x\}+\min\{s,y\}\geq\min\{s,x+y\},
		\qquad \max\{x,y\}\geq\tfrac12(x+y).
	\end{align*}
	Using these inequalities with the two lower bounds just proved yields
	\begin{align*}
		\max\{\alpha_{\rm stat}^2,\alpha_{\rm priv}^2\}
		&\geq\frac12(\alpha_{\rm stat}^2+\alpha_{\rm priv}^2)\\
		&\geq c\min\left\{s,\sigma^2\left(
		\frac mn+\frac{q^2}{n^2\varepsilon^2}\right)\right\}.
	\end{align*}
	This is the claimed sparse-mean lower bound.}
\end{proof}

{  We retain the following tracing-attack inequality for the low-dimensional regression proof below. Here $A_i$ and $A'_i$ are the attack values on the original and one-record-replaced samples defined in Lemma \ref{lm: low-dim regression attack}; conditional on all auxiliary randomness and data other than the replaced record, their output arguments are computed from adjacent data sets.}
\begin{Lemma}\label{lm: attack upper bound}
	{  If $M$ is an $(\varepsilon, \delta)$-differentially private algorithm with $0 < \varepsilon \leq 1$ and $\delta \geq 0$, then for every $T > 0$,}
	\begin{align}\label{eq: attack upper bound}
		{  \E A_i \leq \E A'_i + 2\varepsilon \E|A'_i| + 2\delta T
		+ \int_T^\infty \left\{\Pro\left(|A_i| > t \right)
		+\Pro\left(|A'_i|>t\right)\right\}\,dt.}
	\end{align}
\end{Lemma}
This inequality has previously appeared in \cite{steinke2017tight} and \cite{kamath2018privately} in their respective analysis of tracing attacks. We include a proof in Section \ref{sec: proof of lm: attack upper bound}.
\subsubsection{Proof of Lemma \ref{lm: attack upper bound}}\label{sec: proof of lm: attack upper bound}
\begin{proof}[Proof of Lemma \ref{lm: attack upper bound}]
	{  Let $Z^+=\max\{Z,0\}$ and $Z^-=\max\{-Z,0\}$.  We first
	explain precisely why differential privacy applies to the attack values.
	Condition on the original record, its independent replacement, every
	other record, and all auxiliary randomness except the internal randomness
	of $M$.  The two input data sets to $M$ are then deterministic and
	adjacent.  The map from the output of $M$ to the attack value uses the same
	fixed original record in both cases, so it is a postprocessing map.  Hence,
	for every $t\geq0$,
	\begin{align}\label{eq: attack dp tail comparison}
		\Pr(A_i^+>t)&\leq e^\varepsilon\Pr({A'_i}^+>t)+\delta,\notag\\
		\Pr(A_i^->t)&\geq e^{-\varepsilon}
		\{\Pr({A'_i}^->t)-\delta\}.
	\end{align}
	The second inequality follows by applying DP in the reverse direction,
	$\Pr({A'_i}^->t)\leq e^\varepsilon\Pr(A_i^->t)+\delta$, and
	rearranging.  Averaging over the variables on which we conditioned leaves
	\eqref{eq: attack dp tail comparison} unchanged.

	For any integrable real random variable $Z$, the layer-cake identities are
	\begin{align*}
		\E Z=\int_0^\infty\Pr(Z^+>t)\,dt
		-\int_0^\infty\Pr(Z^->t)\,dt,
		\qquad
		\E|Z|=\int_0^\infty\Pr(|Z|>t)\,dt.
	\end{align*}
	Split the positive-part integral at $T$.  Since $0<\varepsilon\leq1$
	implies $e^\varepsilon-1\leq2\varepsilon$, the first line of
	\eqref{eq: attack dp tail comparison} gives
	\begin{align}\label{eq: attack positive integral}
		\int_0^\infty\Pr(A_i^+>t)\,dt
		&\leq e^\varepsilon\int_0^T\Pr({A'_i}^+>t)\,dt
		+\delta T+\int_T^\infty\Pr(A_i^+>t)\,dt\notag\\
		&\leq\int_0^\infty\Pr({A'_i}^+>t)\,dt
		+2\varepsilon\int_0^\infty\Pr({A'_i}^+>t)\,dt\notag\\
		&\qquad+\delta T+\int_T^\infty\Pr(|A_i|>t)\,dt.
	\end{align}
	For the negative part, $e^{-\varepsilon}\geq1-\varepsilon\geq
	1-2\varepsilon$.  The second line of
	\eqref{eq: attack dp tail comparison} therefore yields
	\begin{align}\label{eq: attack negative integral}
		\int_0^\infty\Pr(A_i^->t)\,dt
		&\geq e^{-\varepsilon}\int_0^T\Pr({A'_i}^->t)\,dt
		-e^{-\varepsilon}\delta T\notag\\
		&\geq\int_0^\infty\Pr({A'_i}^->t)\,dt
		-2\varepsilon\int_0^\infty\Pr({A'_i}^->t)\,dt\notag\\
		&\qquad-\delta T-\int_T^\infty\Pr(|A'_i|>t)\,dt.
	\end{align}
	Subtracting \eqref{eq: attack negative integral} from
	\eqref{eq: attack positive integral}, and using
	$\Pr({A'_i}^+>t)+\Pr({A'_i}^->t)=\Pr(|A'_i|>t)$, gives
	\begin{align*}
		\E A_i
		&\leq\E A'_i+2\varepsilon\E|A'_i|+2\delta T\\
		&\qquad+\int_T^\infty\left\{\Pr(|A_i|>t)
		+\Pr(|A'_i|>t)\right\}\,dt,
	\end{align*}
	which is \eqref{eq: attack upper bound}.}
\end{proof}

\subsection{Proof of Lemma \ref{lm: low-dim regression attack}}\label{sec: proof of lm: low-dim regression attack}
\begin{proof}[Proof of Lemma \ref{lm: low-dim regression attack}]
{  Write $D=(\bm y,\bm X)$, let $D_i'=(\bm y_i',\bm X_i')$
be the sample with record $i$ replaced, and abbreviate
\begin{align*}
	A_i&=\A_{\bbeta}((y_i,\bm x_i),M(D)),\\
	A_i'&=\A_{\bbeta}((y_i,\bm x_i),M(D_i')).
\end{align*}
Also put $\xi_i=y_i-\bm x_i^\top\bbeta$.  Under the model,
$\xi_i\sim N(0,\sigma^2)$ and $\xi_i$ is independent of $\bm x_i$.

We begin with the soundness assertions in part 1.  The estimator
$M(D_i')-\bbeta$ is a function of $D_i'$ and is therefore independent of
the original pair $(\xi_i,\bm x_i)$.  Conditional on $D_i'$,
\begin{align*}
	\E(A_i'\mid D_i')
	&=\{M(D_i')-\bbeta\}^\top\E(\xi_i\bm x_i)=0.
\end{align*}
This proves $\E A_i'=0$.  If $h=M(D_i')-\bbeta$, the same independence
also gives the exact second-moment identity
\begin{align*}
	\E\{(A_i')^2\mid D_i'\}
	&=\E\{\xi_i^2(\bm x_i^\top h)^2\mid D_i'\}\\
	&=\E\xi_i^2\,h^\top\E(\bm x_i\bm x_i^\top)h
	=\sigma^2\|h\|_{\Sigma_{\bm x}}^2.
\end{align*}
The replacement sample $D_i'$ has exactly the same distribution as $D$.
Thus Cauchy--Schwarz and the preceding identity yield
\begin{align*}
	\E|A_i'|
	&\leq\sqrt{\E(A_i')^2}
	=\sigma\sqrt{\E\|M(D)-\bbeta\|_{\Sigma_{\bm x}}^2},
\end{align*}
as required.

We next prove the completeness assertion.  The log-likelihood of the full
sample, up to a term independent of $\bbeta$, is
\begin{align*}
	\log f_{\bbeta}(D)
	=-\frac1{2\sigma^2}\sum_{i=1}^n
	(y_i-\bm x_i^\top\bbeta)^2+\text{constant}.
\end{align*}
Consequently its $j$th score coordinate is
\begin{align}\label{eq: regression likelihood score}
	\frac{\partial}{\partial\beta_j}\log f_{\bbeta}(D)
	=\frac1{\sigma^2}\sum_{i=1}^n\xi_i x_{ij}.
\end{align}
Because the score has mean zero,
\begin{align*}
	\sum_{i=1}^n\E_{D\mid\bbeta}A_i
	&=\sum_{j=1}^d\E_{D\mid\bbeta}
	\left[(M_j(D)-\beta_j)\sum_{i=1}^n\xi_i x_{ij}\right]\\
	&=\sigma^2\sum_{j=1}^d\E_{D\mid\bbeta}
	\left[M_j(D)\frac{\partial}{\partial\beta_j}
	\log f_{\bbeta}(D)\right].
\end{align*}
Define $g_j(\bbeta)=\E_{D\mid\bbeta}M_j(D)$.  Since $M$ takes values in
the unit ball, it is bounded.  The Gaussian likelihood score has finite
moments locally uniformly in $\bbeta$, so differentiation under the integral
is justified by domination.  The usual score identity, here obtained simply
by differentiating $\int M_j(D)f_{\bbeta}(D)\,dD$, gives
\begin{align}\label{eq: attack divergence identity}
	\sum_{i=1}^n\E_{D\mid\bbeta}A_i
	=\sigma^2\sum_{j=1}^d\partial_jg_j(\bbeta).
\end{align}

We now choose the prior and verify every identity used below.  Let its
coordinates be independent with common density
\begin{align}\label{eq: smooth prior density}
	\pi_j(t)=\sqrt d\cos^2\left(\frac{\pi\sqrt d\,t}{2}\right)
	\1\{|t|\leq d^{-1/2}\}.
\end{align}
The substitution $u=\sqrt d\,t$ shows that the density integrates to
\[
	\int_{-1}^1\cos^2(\pi u/2)\,du=1.
\]
Its product support is contained in $[-d^{-1/2},d^{-1/2}]^d\subseteq\Theta$.
In the interior of its support,
its score equals
\begin{align*}
	s_j(t)=\frac{\pi_j'(t)}{\pi_j(t)}
	=-\pi\sqrt d\tan\left(\frac{\pi\sqrt d\,t}{2}\right).
\end{align*}
Since $\pi_j$ vanishes at both endpoints, one-dimensional integration by
parts gives
\begin{align*}
	\E\{-\beta_js_j(\beta_j)\}
	&=-\int_{-d^{-1/2}}^{d^{-1/2}}t\pi_j'(t)\,dt
	=\int_{-d^{-1/2}}^{d^{-1/2}}\pi_j(t)\,dt=1.
\end{align*}
The same substitution as above gives
\begin{align*}
	\E s_j(\beta_j)^2
	&=\pi^2d\int_{-1}^1
	\tan^2(\pi u/2)\cos^2(\pi u/2)\,du\\
	&=\pi^2d\int_{-1}^1\sin^2(\pi u/2)\,du
	=\pi^2d.
\end{align*}
These computations establish
\begin{align}\label{eq: smooth prior identities}
	\E[-\beta_js_j(\beta_j)]=1,
	\qquad \E s_j(\beta_j)^2=\pi^2d.
\end{align}

Average \eqref{eq: attack divergence identity} over this product prior.
Holding all coordinates except $\beta_j$ fixed and integrating by parts in
$\beta_j$ gives, with no boundary term,
\begin{align*}
	\E_{\bm\pi}\partial_jg_j(\bbeta)
	&=-\E_{\bm\pi}\{g_j(\bbeta)s_j(\beta_j)\}\\
	&=\E_{\bm\pi}\{-\beta_js_j(\beta_j)\}
	-\E_{\bm\pi}\{(g_j(\bbeta)-\beta_j)s_j(\beta_j)\}\\
	&=1-\E_{\bm\pi}\{(g_j(\bbeta)-\beta_j)s_j(\beta_j)\}.
\end{align*}
After summing over $j$, Cauchy--Schwarz in the joint prior probability space
and \eqref{eq: smooth prior identities} imply
\begin{align*}
	\sum_{j=1}^d\E_{\bm\pi}\partial_jg_j(\bbeta)
	&\geq d-
	\sqrt{\E_{\bm\pi}\|g(\bbeta)-\bbeta\|_2^2}
	\sqrt{\E_{\bm\pi}\sum_{j=1}^ds_j(\beta_j)^2}\\
	&=d-\pi d\sqrt{\E_{\bm\pi}\|g(\bbeta)-\bbeta\|_2^2}.
\end{align*}
Finally, conditional Jensen gives
\begin{align*}
	\|g(\bbeta)-\bbeta\|_2^2
	=\|\E_{D\mid\bbeta}(M(D)-\bbeta)\|_2^2
	\leq\E_{D\mid\bbeta}\|M(D)-\bbeta\|_2^2.
\end{align*}
Combining the last three displays proves
\begin{align}\label{eq: low-dim regression lower bound completeness}
	\E_{\bm\pi}\sum_{i=1}^n\E_{D\mid\bbeta}A_i
	\geq\sigma^2d\left[1-\pi
	\sqrt{\E_{\bm\pi}\E_{D\mid\bbeta}\|M(D)-\bbeta\|_2^2}
	\right],
\end{align}
which is exactly \eqref{eq: regression attack quantitative completeness}.}
\end{proof}

\subsection{Proof of Theorem \ref{thm: low-dim regression lower bound}}\label{sec: proof of thm: low-dim regression lower bound}

\begin{proof}[Proof of Theorem \ref{thm: low-dim regression lower bound}]
{  We work throughout with the submodel in which the coordinates of
$\bm x$ are independent Rademacher variables.  This submodel belongs to the
class in the theorem and satisfies $\Sigma_{\bm x}=\bm I_d$; hence its
$\Sigma_{\bm x}$-loss is exactly Euclidean loss.

We first prove the ordinary statistical term rather than invoke it
implicitly.  For $d\geq2$, the binary Varshamov--Gilbert lemma gives vectors
$v^1,\ldots,v^N\in\{-1,1\}^d$ such that
\begin{align*}
	\log N\geq c d,
	\qquad \|v^k-v^\ell\|_0\geq c d\quad(k\neq\ell).
\end{align*}
For $0<\alpha\leq1$, set $\bbeta_k=\alpha v^k/\sqrt d$.  Then
$\bbeta_k\in\Theta$, and distinct parameters have Euclidean distance at
least $c\alpha$.  For one observation, the design distribution is common to
all parameters and
\begin{align*}
	D_{\mathrm{KL}}(P_k\|P_\ell)
	&=\E_{\bm x}\frac{\{\bm x^\top(\bbeta_k-\bbeta_\ell)\}^2}
	{2\sigma^2}
	=\frac{\|\bbeta_k-\bbeta_\ell\|_2^2}{2\sigma^2}
	\leq C\frac{\alpha^2}{\sigma^2}.
\end{align*}
Choose $\alpha=c\min\{1,\sigma\sqrt{d/n}\}$.  Then the ordinary
alternative of Lemma \ref{lm: private packing} applies and gives
\begin{align}\label{eq: regression ordinary lower component}
	\sup_{\bbeta\in\Theta}\E_{D\mid\bbeta}\|M(D)-\bbeta\|_2^2
	\geq c\min\{1,\sigma^2d/n\}.
\end{align}
For $d=1$, the same conclusion follows from the two points
$\bbeta=\pm\alpha$ and the binary part of that lemma.

We next prove the privacy term.  Euclidean projection of the output onto the
closed unit ball is postprocessing and, for every $\bbeta\in\Theta$, cannot
increase $\|M-\bbeta\|_2$.  We may therefore assume
$\|M(D)\|_2\leq1$.  Let $\bm\pi$ be the prior constructed in Lemma
\ref{lm: low-dim regression attack}, and write
\begin{align*}
	R_\pi=\E_{\bm\pi}\E_{D\mid\bbeta}\|M(D)-\bbeta\|_2^2.
\end{align*}
If $R_\pi>1/(4\pi^2)$, then, since
$\min\{1,\sigma^2d^2/(n^2\varepsilon^2)\}\leq1$, we already have
\begin{align*}
	R_\pi\geq c\min\left\{1,
	\frac{\sigma^2d^2}{n^2\varepsilon^2}\right\}.
\end{align*}
We may thus assume $R_\pi\leq1/(4\pi^2)$.  The completeness part of Lemma
\ref{lm: low-dim regression attack} then gives
\begin{align}\label{eq: regression completeness constant}
	\E_{\bm\pi}\sum_{i=1}^n\E_{D\mid\bbeta}A_i
	\geq\frac12\sigma^2d.
\end{align}

We now bound the same quantity from above.  Part 1 of the attack lemma gives
$\E A_i'=0$ and, because $\Sigma_{\bm x}=\bm I_d$,
\begin{align*}
	\E_{D\mid\bbeta}|A_i'|
	\leq\sigma\sqrt{\E_{D\mid\bbeta}\|M(D)-\bbeta\|_2^2}.
\end{align*}
Moreover, $\|M(D)-\bbeta\|_2\leq2$ and
$\|\bm x_i\|_2=\sqrt d$.  Therefore both attacks obey the deterministic
bound
\begin{align*}
	|A_i|\leq2\sqrt d\,|\xi_i|,
	\qquad |A_i'|\leq2\sqrt d\,|\xi_i|,
\end{align*}
where $\xi_i\sim N(0,\sigma^2)$.  The Gaussian tail inequality gives, for
every $t>0$,
\begin{align}\label{eq: regression attack gaussian tail}
	\max\{\Pr(|A_i|>t),\Pr(|A_i'|>t)\}
	\leq2\exp\left(-\frac{t^2}{8d\sigma^2}\right).
\end{align}
Applying Lemma \ref{lm: attack upper bound} separately for each $i$ and
summing, we obtain, for every finite $T>0$,
\begin{align}\label{eq: summed regression attack upper}
	\sum_{i=1}^n\E_{D\mid\bbeta}A_i
	&\leq2n\varepsilon\sigma
	\sqrt{\E_{D\mid\bbeta}\|M(D)-\bbeta\|_2^2}
	+2n\delta T\notag\\
	&\qquad+4n\int_T^\infty
	\exp\left(-\frac{t^2}{8d\sigma^2}\right)dt.
\end{align}

Here is the tail calculation used to choose $T$.  For every $T>0$,
\begin{align}\label{eq: integrated regression gaussian tail}
	\int_T^\infty e^{-t^2/(8d\sigma^2)}dt
	&\leq\frac1T\int_T^\infty t e^{-t^2/(8d\sigma^2)}dt\\
	&=\frac{4d\sigma^2}{T}
	e^{-T^2/(8d\sigma^2)}.\notag
\end{align}
When $0<\delta<\delta_0$, take
$T=C\sigma\sqrt{d\log(1/\delta)}$, where $C$ is large enough that
$C^2/8\geq1$.  Since $\delta_0$ is a sufficiently small constant,
$\log(1/\delta)\geq1$.  Thus \eqref{eq: integrated regression gaussian tail}
and $e^{-T^2/(8d\sigma^2)}=\delta^{C^2/8}\leq\delta$ show that the last two
terms in \eqref{eq: summed regression attack upper} are at most
\begin{align*}
	C_1\sigma n\delta\sqrt{d\log(1/\delta)}.
\end{align*}
When $\delta=0$, let $T\to\infty$ in
\eqref{eq: summed regression attack upper}; the last two terms tend to zero.
Hence, with the usual zero interpretation at $\delta=0$, averaging
\eqref{eq: summed regression attack upper} over the prior and applying
Jensen's inequality to the square root gives
\begin{align}\label{eq: averaged regression attack upper}
	\E_{\bm\pi}\sum_{i=1}^n\E_{D\mid\bbeta}A_i
	&\leq2n\varepsilon\sigma\sqrt{R_\pi}
	+C_1\sigma n\delta\sqrt{d\log(1/\delta)}.
\end{align}
	The hypothesis
$n\delta\sqrt{\log(1/\delta)}\leq c_0\sigma\sqrt d$, with $c_0$ small
enough, makes the second term in \eqref{eq: averaged regression attack upper}
at most $\sigma^2d/4$.  Comparing \eqref{eq: regression completeness constant}
and \eqref{eq: averaged regression attack upper} yields
\begin{align*}
	\frac14\sigma^2d\leq2n\varepsilon\sigma\sqrt{R_\pi},
	\qquad
	R_\pi\geq c\frac{\sigma^2d^2}{n^2\varepsilon^2}.
\end{align*}
Together with the preceding large-$R_\pi$ case, this proves
\begin{align}\label{eq: regression private lower component}
	\sup_{\bbeta\in\Theta}\E_{D\mid\bbeta}\|M(D)-\bbeta\|_2^2
	\geq R_\pi
	\geq c\min\left\{1,\frac{\sigma^2d^2}{n^2\varepsilon^2}\right\}.
\end{align}

Finally, for $a,b\geq0$,
$\min\{1,a\}+\min\{1,b\}\geq\min\{1,a+b\}$.  Taking the larger of
\eqref{eq: regression ordinary lower component} and
\eqref{eq: regression private lower component}, and using
$\max\{u,v\}\geq(u+v)/2$, gives
\begin{align*}
	\sup_{\bbeta\in\Theta}\E_{D\mid\bbeta}\|M(D)-\bbeta\|_2^2
	\geq c\min\left\{1,\sigma^2\left(
	\frac dn+\frac{d^2}{n^2\varepsilon^2}\right)\right\}.
\end{align*}
Since the Rademacher-design submodel is contained in the full class, this is
the asserted minimax lower bound.}
\end{proof}

\subsection{Proof of Theorem \ref{thm: high-dim regression lower bound}}\label{sec: proof of thm: high-dim regression lower bound}
\begin{proof}
{  Put $s=s^*$, $m=m_{d,s}$, and
$q=q_{d,s}(\varepsilon,\delta)$.  It is enough to consider a fixed design
distribution with independent Rademacher coordinates.  This distribution is
independent of the parameter and has $\E(\bm x\bm x^\top)=\bm I_d$.

Use the constant-weight code $\mathcal V$ constructed in
\eqref{eq: constant weight packing}--\eqref{eq: constant weight distance};
thus $\log|\mathcal V|\geq cm$ and
$cs\leq\|v-w\|_0\leq2s$ for distinct codewords.  For
$0<\alpha\leq1/4$, define
\begin{align*}
	\bbeta_v=\frac{\alpha}{\sqrt{s}}v,
	\qquad
	P_v(d\bm x,dy)=P_{\bm x}(d\bm x)
	N(\bm x^\top\bbeta_v,\sigma^2)(dy).
\end{align*}
The parameter is $s$-sparse and $\|\bbeta_v\|_2=\alpha\leq1$, so this is a
submodel of the class in the theorem.  Moreover,
\begin{align*}
	\|\bbeta_v-\bbeta_w\|_2^2
	=\frac{\alpha^2}{s}\|v-w\|_0,
\end{align*}
and hence the parameters are separated by $c\alpha$.  Conditional on
$\bm x$, the two response laws are one-dimensional Gaussians.  Averaging
their KL divergence over the common design distribution gives
\begin{align}\label{eq: sparse regression divergences}
	D_{\mathrm{KL}}(P_v\|P_w)
	&=\E_{\bm x}\frac{\{\bm x^\top(\bbeta_v-\bbeta_w)\}^2}{2\sigma^2}
	=\frac{\|\bbeta_v-\bbeta_w\|_2^2}{2\sigma^2}
	\leq C\frac{\alpha^2}{\sigma^2},\notag\\
	d_{\mathrm{TV}}(P_v,P_w)
	&\leq\min\left\{1,
	\sqrt{D_{\mathrm{KL}}(P_v\|P_w)/2}\right\}
	\leq\min\{1,C\alpha/\sigma\}.
\end{align}

For the ordinary statistical term, choose
\begin{align*}
	\alpha_{\rm stat}=c_1\min\{1,\sigma\sqrt{m/n}\}.
\end{align*}
After reducing $c_1$, this amplitude is at most $1/4$, and
\eqref{eq: sparse regression divergences} implies
\begin{align*}
	n\kappa^2\leq Cn\alpha_{\rm stat}^2/\sigma^2
	\leq Cc_1^2m\leq c\log|\mathcal V|.
\end{align*}
The ordinary alternative of Lemma \ref{lm: private packing} therefore gives
risk at least $c\alpha_{\rm stat}^2$.

For the private term, set $\alpha_{\rm priv}=0$ if $q=0$.  Suppose now that
$q>0$.  If $q$ is larger than a fixed constant, choose
$\mathcal V_q\subseteq\mathcal V$ with
$c q\leq\log N_q\leq Cq$, where $N_q=|\mathcal V_q|$; if $q$ is bounded,
take two distinct codewords.  Thus in all cases
\begin{align}\label{eq: sparse regression subpacking size}
	c\max\{1,q\}\leq\log N_q\leq C\max\{1,q\}.
\end{align}
Take
\begin{align*}
	\alpha_{\rm priv}=c_1\min\left\{1,
	\frac{\sigma q}{n\varepsilon}\right\}.
\end{align*}
To check the private-testing condition, write
$u=\min\{1,\sigma q/(n\varepsilon)\}$.  If
$C\alpha_{\rm priv}/\sigma<1$ and $u=\sigma q/(n\varepsilon)$, then
\begin{align*}
	n\varepsilon C\alpha_{\rm priv}/\sigma\leq Cc_1q.
\end{align*}
If the same total-variation bound is uncapped and $u=1$, the fact that
$u$ is the smaller entry gives $n\varepsilon/\sigma\leq q$, and the same
conclusion follows.  If $C\alpha_{\rm priv}/\sigma\geq1$ and
$u=\sigma q/(n\varepsilon)$, the cap condition itself gives
$n\varepsilon\leq Cc_1q$.  Finally, if the bound is capped and $u=1$, then
$\sigma\leq Cc_1$; together with $n\varepsilon/\sigma\leq q$, this again
gives $n\varepsilon\leq Cc_1q$.  Therefore
\begin{align*}
	n\varepsilon d_{\mathrm{TV}}(P_v,P_w)
	\leq Cc_1q\leq c\log N_q
\end{align*}
for $c_1$ sufficiently small.

It remains to check approximate DP.  Directly from
\eqref{eq: effective private entropy}, for $q>0$,
\begin{align*}
	\delta\leq c_\star(e^\varepsilon-1)e^{-q/c_\star}.
\end{align*}
When $q$ is unbounded, the upper bound on $\log N_q$ in
\eqref{eq: sparse regression subpacking size} and a sufficiently small
$c_\star$ imply
$\delta\leq c_3(e^\varepsilon-1)N_q^{-C_3}$.  When $q$ is bounded and
positive, its positivity implies $\delta<c_\star(e^\varepsilon-1)$; choosing
$c_\star\leq c_3 2^{-C_3}$ proves the required inequality for $N_q=2$.
Thus the private alternative of Lemma \ref{lm: private packing} applies and
gives risk at least $c\alpha_{\rm priv}^2$.

For $a,b\geq0$, use
$\max\{a,b\}\geq(a+b)/2$ and
$\min\{1,a\}+\min\{1,b\}\geq\min\{1,a+b\}$.  Combining the statistical
and privacy bounds, including the convention $\alpha_{\rm priv}=0$ when
$q=0$, gives
\begin{align*}
	\inf_{M\in\mathcal M_{\varepsilon,\delta}}
	\sup_{P\in\mathcal P(\sigma,d,s,\Theta)}
	\E_P\|M-\bbeta(P)\|_2^2
	\geq c\min\left\{1,\sigma^2\left(
	\frac mn+\frac{q^2}{n^2\varepsilon^2}\right)\right\},
\end{align*}
which is the claimed result.}
\end{proof}

\bibliographystyle{plain}
\bibliography{reference}

\end{document}